\documentclass[conference]{IEEEtran}
\usepackage[numbers, compress]{natbib}

\usepackage{times}
\usepackage{epsfig}
\usepackage{graphicx}
\usepackage{amsmath}
\usepackage{amsthm}
\usepackage{amssymb}
\usepackage{graphicx}
\usepackage{subcaption}
\usepackage{booktabs}
\usepackage{algorithm, setspace}
\usepackage{algpseudocode, array, multirow}
\usepackage{caption}
\usepackage{xcolor}
\usepackage{extarrows}
\usepackage{booktabs}
\usepackage{rotating}
\usepackage{fixltx2e}
\usepackage{textcomp}
\usepackage{tikz}
\usepackage{pgfplots}
\usetikzlibrary{arrows.meta, shapes.geometric, calc, shadows, matrix, positioning, angles, quotes}
\usepackage{etoolbox} 
\usepackage{listofitems}
\usepackage[shortlabels]{enumitem}
\usepackage[export]{adjustbox}

\usepackage[pagebackref=false,breaklinks=true,letterpaper=true,colorlinks,bookmarks=false]{hyperref}
\hypersetup{colorlinks,breaklinks,urlcolor=[rgb]{0,0,0},linkcolor=[rgb]{1,0,0},citecolor=[rgb]{0,0,1}}
\usepackage[capitalize]{cleveref}
\Crefname{section}{Section}{Sections}
\Crefname{table}{Table}{Tables}
\Crefname{equation}{Equation}{Equations}
\Crefname{figure}{Figure}{Figures}
\Crefname{algorithm}{Algorithm}{Algorithms}

\newtheorem{theorem}{Theorem}
\newtheorem{corollary}{Corollary}[theorem]
\newtheorem{lemma}{Lemma}
\theoremstyle{definition}

\theoremstyle{remark}

\usepackage{orcidlink}
\theoremstyle{proposition}

\newtheorem{innercustomthm}{Theorem}
\newenvironment{customthm}[1]
{\renewcommand\theinnercustomthm{#1}\innercustomthm}
{\endinnercustomthm}

\graphicspath{{./figures/}}

\newcommand{\norm}[1]{\left\lVert#1\right\rVert}

\DeclareMathOperator*{\argmax}{arg\,max}
\DeclareMathOperator*{\argmin}{arg\,min}

\makeatletter
\newcommand\footnoteref[1]{\protected@xdef\@thefnmark{\ref{#1}}\@footnotemark}
\makeatother

\definecolor{darkgreen}{rgb}{0.035, 0.412, 0.098}
\definecolor{mygreen}{rgb}{0.05, 0.71, 0.47}
\colorlet{myred}{red!80!black}
\colorlet{myblue}{blue!80!black}
\colorlet{mygreen}{green!60!black}
\colorlet{mydarkred}{myred!40!black}
\colorlet{mydarkblue}{myblue!40!black}
\colorlet{mydarkgreen}{mygreen!40!black}

\definecolor{lightblue}{rgb}{0.145,0.6666,1}
\definecolor{darkblue}{rgb}{0.169,0.549,0.745}
\definecolor{darkgreen}{rgb}{0.015,0.6215,0.4413}

\definecolor{myorange}{RGB}{250, 184, 2}
\definecolor{mypink}{RGB}{247, 10, 125}
\definecolor{myblue}{RGB}{2, 106, 250}
\definecolor{mypurple}{RGB}{147, 2, 250}

\begin{document}

\title{The Devil's Advocate:\\Shattering the Illusion of Unexploitable Data using Diffusion Models}

\author{\IEEEauthorblockN{Hadi M.~Dolatabadi\textsuperscript{\orcidlink{0000-0001-9418-1487}}, Sarah Erfani\textsuperscript{\orcidlink{0000-0003-0885-0643}}, and Christopher Leckie\textsuperscript{\orcidlink{0000-0002-4388-0517}}}
\IEEEauthorblockA{\textit{School of Computing and Information Systems} \\
\textit{The University of Melbourne}\\
Victoria, Australia \\
\{h.dolatabadi, sarah.erfani, caleckie\}@unimelb.edu.au}
}

\maketitle

\begin{abstract}
   Protecting personal data against exploitation of machine learning models is crucial.
   Recently, availability attacks have shown great promise to provide an extra layer of protection against the unauthorized use of data to train neural networks.
   These methods aim to add imperceptible noise to clean data so that the neural networks cannot extract meaningful patterns from the protected data, claiming that they can make personal data ``unexploitable.''
   This paper provides a strong countermeasure against such approaches, showing that unexploitable data might only be an illusion.
   In particular, we leverage the power of diffusion models and show that a carefully designed denoising process can counteract the effectiveness of the data-protecting perturbations.
   We rigorously analyze our algorithm, and theoretically prove that the amount of required denoising is directly related to the magnitude of the data-protecting perturbations.
   Our approach, called \textsc{Avatar}, delivers state-of-the-art performance against a suite of recent availability attacks in various scenarios, outperforming adversarial training even under distribution mismatch between the diffusion model and the protected data.
   Our findings call for more research into making personal data unexploitable, showing that this goal is far from over.
   Our implementation is available at this repository: \url{https://github.com/hmdolatabadi/AVATAR}.
\end{abstract}

\begin{IEEEkeywords}
neural networks, availability attacks, diffusion models, facial recognition
\end{IEEEkeywords}

\section{Introduction}\label{sec:intro}
Neural networks have achieved great success in various areas of computer vision including object detection~\citep{he2016deep, dosovitskiy2021vit}, semantic segmentation~\citep{zhou2021gmnet, liu2022cmx}, and photo-realistic image/video generation~\citep{karras2020stylegan2, dhariwal2021diffusion, singer2023makeavideo}.
While the efforts of the community in the development of such models cannot be undermined, this unparalleled success would have been impossible without the abundance of data resources available today~\citep{deng2009imagenet, krizhevsky2009learning, russakovsky2015imagenet, lin2014coco}.
In this regard, social media, and the internet in general, provides a platform that can be crawled easily to create massive datasets.
This capability can act both as a blessing and a curse: while the collected data can facilitate learning larger, more accurate neural networks, the users lose control over protecting their personal data from being exploited.
This issue has raised increasing concerns about misuse of personal data~\citep{hill2019photos, hill2020secretive, birhane2021large}.

Recently, there has been an increasing number of studies on hindering the unauthorized use of personal data for neural network image classifiers~\citep{feng2019con, huang2021emn, yuan2021ntga, fowl2021tap, fu2022remn, yu2022shr, tian2022confoundergan, sandoval2022ar}.
These methods tend to add an imperceptible amount of noise to the clean images so that while the data has the same appearance as the ground-truth, it cannot provide any meaningful patterns for the neural networks to learn.
As a result, such approaches, collectively known as \textit{availability attacks}~\citep{biggio2018wild}, claim that personal image data can be made \textit{unexploitable} for the neural networks~\citep{huang2021emn, yu2022shr}.
While there has been an abundance of research on designing better availability attacks, far too little attention has been paid to counter-attacks that might be employed by adversaries to break such precautionary measures.

Unfortunately, the assumptions of existing availability attacks are far too weak to make the data unexploitable.
For example, consider a user who shares their protected photos over their social media.
We can clearly see that once the photos are shared, they cannot be protected against \textit{all} future countermeasures~\citep{dixit2022data}.
For instance, consider a corporate entity that aims to train face recognition models by crawling over social media without the consent of the users.
While this unauthorized entity might not have unprotected versions of a particular person's image from his/her social media, they can have a large pre-trained model representing a facial image distribution.
Given this threat model, shown in \Cref{fig:threat_model}, we aim to show that counteracting the protecting perturbations is indeed plausible.

\begin{figure*}[tb!]
    \centering
        \tikzstyle{box} = [rectangle, draw, fill=red!30, rounded corners=0.2cm, text width=2.75cm, text badly centered, node distance=1cm, minimum height=1.20cm, inner sep=0pt, rotate=90]
        \resizebox{0.70\textwidth}{!}{
        \begin{tikzpicture}[node distance = 1cm, auto]
        \node [inner sep=0pt] at (0.5,3.0) (Inp2) {\includegraphics[width=1.75cm,height=1.75cm]{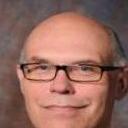}};
        \node [inner sep=0pt] at (0.65,1.95) (Inp3) {\includegraphics[width=1.75cm,height=1.75cm]{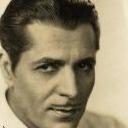}};
        \node [inner sep=0pt] at (0,2.6) (Inp) {\includegraphics[width=1.75cm,height=1.75cm]{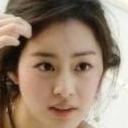}};
        \node [inner sep=0pt] at (-0.4,2.9) (Inp4) {\includegraphics[width=1.75cm,height=1.75cm]{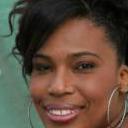}};
        \node [inner sep=0pt] at (-0.7,2.4) (Inp5) {\includegraphics[width=1.75cm,height=1.75cm]{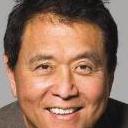}};
        \draw [dashed, rounded corners=0.25cm] (-1.85,0.95) rectangle (1.75,4.00);
        \node [inner sep=0pt] at (0,4.2) (mul) {\small Large Dataset};
        \node [inner sep=0pt] at (12.65,2.6) (Input) {\includegraphics[width=.1\textwidth]{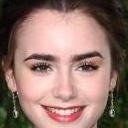}};
        \node [inner sep=0pt] at (12.45,2.4) (Input) {\includegraphics[width=.1\textwidth]{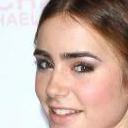}};
        \node [inner sep=0pt] at (12.25,2.2) (Input) {\includegraphics[width=.1\textwidth]{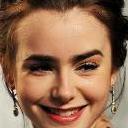}};
        \node [inner sep=0pt] at (12.05,2.) (Input) {\includegraphics[width=.1\textwidth]{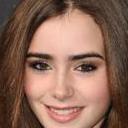}};
        \draw [dashed, rounded corners=0.25cm] (11.05,1.) rectangle (13.65,3.6);
        \node [inner sep=0pt] at (12.35,3.8) (mul) {\small User};
        \node [box, fill=green!30] at (9.25, 2.25) {Data Protection};
        \node [inner sep=0pt] at (5.75, 2.3) (www) {\includegraphics[width=2.75cm,height=2.75cm]{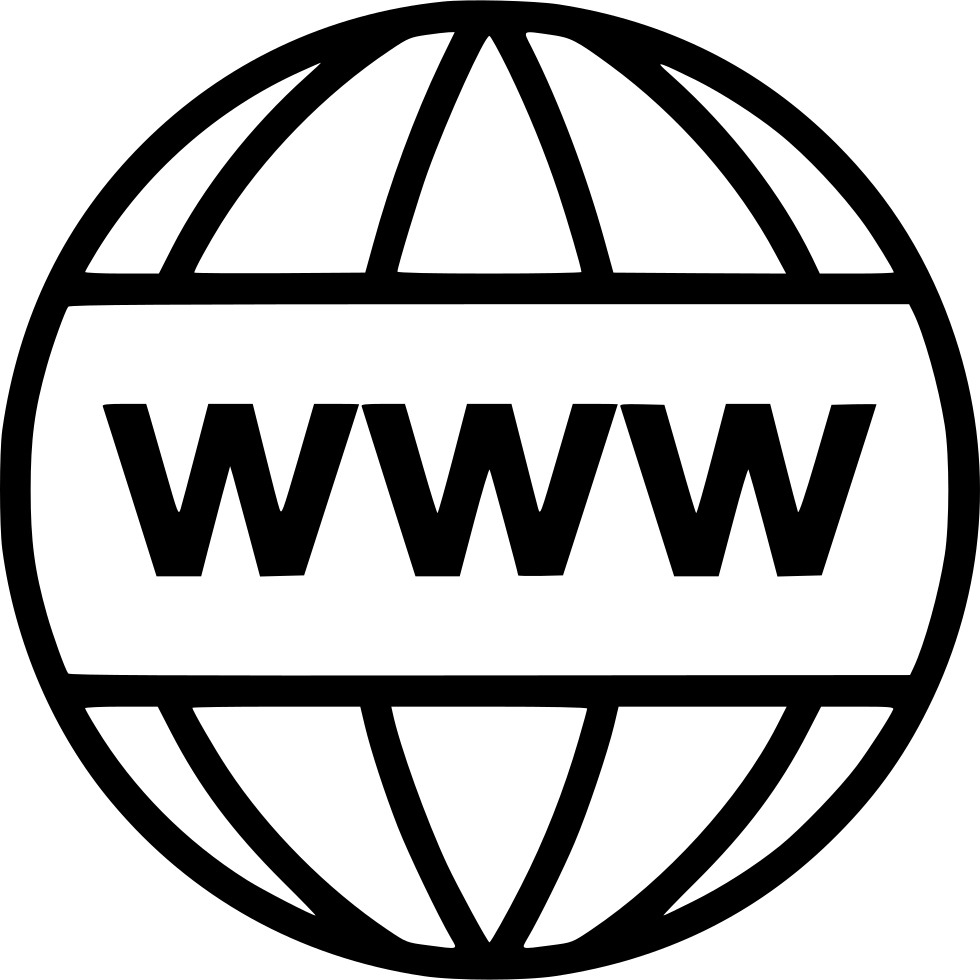}};
        \draw [fill=black!20, rounded corners=0.25cm] (-1.55,-1.65) rectangle (13.75,0.55);
        \node [box, rotate=-90] at (-0.05, -0.5) {Diffusion Model};
        \node [box, rotate=-90] at (5.75, -0.5) {Protection Defuser};
        \node [inner sep=0pt] at (9.65,-0.2) (Input) {\includegraphics[width=.05\textwidth]{scenario/person_0.jpg}};
        \node [inner sep=0pt] at (9.45,-0.3) (Input) {\includegraphics[width=.05\textwidth]{scenario/person_3.jpg}};
        \node [inner sep=0pt] at (9.25,-0.4) (Input) {\includegraphics[width=.05\textwidth]{scenario/person_2.jpg}};
        \node [inner sep=0pt] at (9.05,-0.5) (Input) {\includegraphics[width=.05\textwidth]{scenario/person_1.jpg}};
        \node [box, rotate=-90, fill=yellow] at (12.25, -0.5) {Unauthorized Facial Recognition};
        \draw[-{Latex[length=3mm, width=1.5mm]}, line width=0.4mm] (11.05, 2.30) -- (9.85, 2.30);
        \draw[-{Latex[length=3mm, width=1.5mm]}, line width=0.4mm] (8.65, 2.30) -- (www.east);
        \draw[-{Latex[length=3mm, width=1.5mm]}, line width=0.4mm] (-0.05, 0.95) -- (-0.05, +0.10);
        \draw[-{Latex[length=3mm, width=1.5mm]}, line width=0.4mm] (1.30, -0.5) -- (4.35, -0.5);
        \draw[-{Latex[length=3mm, width=1.5mm]}, line width=0.4mm] (7.10, -0.5) -- (8.50, -0.5);
        \draw[-{Latex[length=3mm, width=1.5mm]}, line width=0.4mm] (10.10, -0.5) -- (10.85, -0.5);
        \draw[-{Latex[length=3mm, width=1.5mm]}, line width=0.4mm] (www.south) -- (5.75, +0.10);
        \node [inner sep=0pt, rotate=90] at (-1.80,-1.25) (dev) {\includegraphics[width=0.35cm,height=0.35cm]{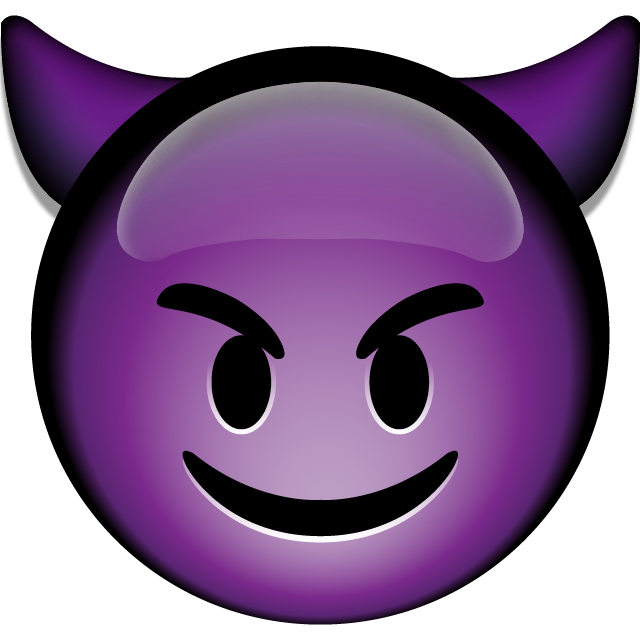}};
        \node [inner sep=0pt, rotate=90] at (-1.80,-0.30) (mul) {\small Evil Entity};
	\end{tikzpicture}}
    \caption{The threat model considered in this paper. Availability attacks cannot guarantee to protect all the data that exists over the web.
             A data exploiter might use large density estimators to defuse the data-protecting perturbations and exploit the data.}
    \label{fig:threat_model}
\end{figure*}

To this end, we show that pre-trained density estimators are powerful tools that can be used to counteract the effects of the data-protecting perturbations, eventually enabling us to exploit protected data.
We utilize the power of diffusion models in representing the image data distributions to show that reverse-engineering unexploitable data is easier than what is thought.
In particular, given a training dataset, we first diffuse the images by adding a controlled amount of Gaussian noise following the forward process of a pre-trained diffusion model.
Then, we denoise the noisy images using the reverse process of the aforementioned model, resulting in a dataset purified from data-protecting perturbations.
Theoretically, using contraction properties of stochastic difference equations we prove that the number of diffusion steps required to cancel the data-protecting perturbations is directly influenced by the magnitude of its norm.
Thus, protecting personal data using imperceptible perturbations is not possible.
We also empirically show that our approach is surprisingly powerful, being able to deliver the state-of-the-art~(SOTA) performance against a wide variety of recent availability attacks.
Our findings indicate the fragility of \textit{unexploitable data}, calling for more research to protect personal data.

Diffusion models have been extensively used in various areas.
Closely related to our work, \citet{yoon2021adp} and \citet{nie2022diffpure} have employed diffusion models to increase robustness against adversarial attacks.
In contrast to these methods, in this paper, we investigate the capabilities of diffusion models as a threat against personal data protected by availability attacks.
In particular, we leverage the SOTA diffusion models as a proxy for the true data distribution and argue why unlearnable examples provide a false sense of data privacy.

Our contributions can be summarized as follows:
\begin{itemize}
    \item We introduce \textsc{Avatar} as a countermeasure against data availability attacks. To the best of our knowledge, this is the first work that explores the use of diffusion models to circumvent such attacks.
	\item We show the power of \textsc{Avatar} in breaking availability attacks over five datasets, four architectures, and seven of the most recent availability attacks. \textsc{Avatar} achieves the SOTA performance against availability attacks, outperforming adversarial training.
    \item Our results indicate that even in the absence of the true data distribution, one can use a similar distribution to counteract availability attacks.
    \item Theoretically, we show that the amount of noise needed to diffuse the data-protecting perturbation is directly related to the magnitude of its norm. This result indicates that achieving both goals of availability attacks (data utility and protection) at the same time is impossible.
\end{itemize}

\section{Related Work}\label{sec:related}
In this section, we review the related work to our approach.

\paragraph{Poisoning and Backdoor Attacks}
A considerable number of studies have been published on various types of \textit{data poisoning attacks}~\citep{biggio2018wild, schwarzschild2021toxic, goldblum2023dataset}.
These attacks aim to pollute the training data so that they can hinder the performance of the machine learning model at test-time~\citep{biggio2012svmbackdoor, koh2017understanding, munozgonzalez2017towards}.
While these methods are quite successful in achieving this goal, they often tend to perform weakly against neural networks~\citep{munozgonzalez2017towards} and appear to be distinguishable from the clean samples, damaging the utility of the underlying data~\citep{yang2017generative}.
\textit{Backdoor attacks} are a popular family of data poisonings against deep neural networks~\citep{gu2017badnets,barni2019sig,tran2018ss,dolatabadi2022collider}.
Unlike general poisoning attacks, these methods attach triggers to a small fraction of the clean training data so that the model creates an association between the existence of the trigger and a particular class.
During inference, the neural network would behave normally on benign samples.
However, if the trigger is activated, the model would output the attacker's desired value due to the existence of a backdoor in the model.

\paragraph{Availability Attacks}
Motivated to address the lack of personal data privacy, an emerging type of poisoning attacks known as \textit{availability attacks} have drawn considerable attention.
Unlike previous types of poisoning attacks, availability attacks seek to add imperceptible perturbations to the clean training data with two goals in mind.
First, the added perturbation should be able to protect the underlying data from being exploited by a neural network during training.
Second, the perturbed data should still preserve its normal utility.
To understand these constraints, consider a user sharing their photo over their social media.
While the user wants to protect their photo from unauthorized use of web-crawlers to train a face recognition model~\citep{hill2019photos}~(first constraint), they still wish their photo to appear normal to their audience~(second constraint)~\citep{huang2021emn}.

\citet{feng2019con} propose to produce the poisoning perturbations by training an auto-encoder, whose aim is to get the lowest performance from an auxiliary classifier.
In a similar spirit, \citet{tian2022confoundergan} train a conditional generative adversarial network~(GAN)~\citep{goodfellow2014gan} to generate the availability attacks' perturbation.
The training objective is designed to create a spurious correlation between the noisy image and the ground-truth labels.
Concurrently, \citet{yu2022shr} empirically investigate various types of availability attacks and show that almost all of them leverage these spurious features to create a shortcut within neural networks~\citep{geirhos2020shortcut}.
\citet{yu2022shr} then propose a fast and scalable approach for perturbation generation by generating randomly-initialized linearly-separable perturbations which can generate availability attacks for an entire dataset in a few seconds.
Concurrently, \citet{sandoval2022ar} proposed another approach that generates the random noise independent from the data.
In this approach, first the beginning rows and columns of each channel are populated with Gaussian noise.
Then, an autoregressive process is used to find the value of the remaining pixel values.

Another popular approach to generate availability attacks is via direct optimization.
\citet{huang2021emn} define a bi-level optimization objective to generate error-minimizing noise for data samples and an auxiliary classifier.
It is argued that since the perturbed images minimize the auxiliary classifier's loss, they contain no useful information for any other target classifier to learn, and as such, the model would not exploit them during training.
In contrast, \citet{fowl2021tap} show that using adversarial examples~\citep{szegedy2014intriguing, goodfellow2014explaining} as the poisoned data would make it hard for the classifier to learn any meaningful pattern, and thus, they can serve as a powerful family of availability attacks.
While optimization-based availability attacks are potent, they are often computationally demanding and several attempts have been made to ease their computational burden~\citep{fowl2021preventing, zhang2021data}.

Compared to various types of availability attacks, preventative measures have received little attention.
It has been shown that various data augmentation techniques~(such as CutOut~\citep{devries2017cutout}, Mixup~\citep{zhang2018mixup}, CutMix~\citep{yun2019cutmix}, and Fast Auto-augment~\citep{lim2019autoaugment}) are not able to prevent availability attacks~\citep{huang2021emn, fowl2021tap, tian2022confoundergan, yu2022shr}.
\citet{tao2021preventing} show that adversarial training~\citep{madry2018towards,zhang2019trades,dolatabadi2022ellinf}, originally proposed to enhance robustness against adversarial attacks~\citep{szegedy2014intriguing, goodfellow2014explaining}, can be used to train successful classifiers against availability attacks.
Later, \citet{fu2022remn} extended the error-minimizing noise of \citet{huang2021emn} resulting in perturbations that can even prevent adversarial training from learning over the poisoned data.
Despite this, adversarial training has remained one of the strongest defense baselines against availability attacks.
In this work, we show that one can outperform adversarial training in an attempt to counteract availability attacks.

\paragraph{Diffusion Models}
Denoising diffusion probabilistic modeling~(DDPM)~\citep{sohl2015deep, ho2020ddpm}~(also known as score-matching networks~\citep{song2019generative, song2020improved, song2021scoresde}) are a family of deep generative models that have achieved the SOTA performance in image~\citep{dhariwal2021diffusion, xu2023stable}, text-to-image~\citep{rombach2022stable}, video~\citep{singer2023makeavideo}, and 3D-object~\citep{poole2023dreamfusion} generation.
Diffusion models generally comprise of a forward and a backward process~\citep{croitoru2022diffusion}.
In the forward process, the model gradually adds noise to the data until it is transformed into Gaussian noise.
The backward process is the reverse of the forward process, where the model tries to gradually transform/denoise a Gaussian vector into a data point.

\section{Proposed Method}\label{sec:method}
This section formally introduces our proposed method, called \textsc{Avatar}~(dAta aVailAbiliTy Attacks defuseR).
First, we define our notation and problem settings.
Next, we introduce our proposed approach that materializes our threat model and provide a theoretical analysis of our framework.
Finally, we discuss the potential advantages of \textsc{Avatar} compared to existing methods such as adversarial training.

\subsection{Problem Statement}\label{sec:sec:statement}

Let $\mathcal{D}=\{(\boldsymbol{x}^{(i)}, y^{(i)})\}_{i=1}^{n}$ be a labeled dataset consisting of $n$ i.i.d.~samples~$\boldsymbol{x}^{(i)}$ each with a label $y^{(i)}$.
Without loss of generality, in this paper, we consider image data $\boldsymbol{x}^{(i)} \in \mathbb{R}^{d}$ where $d$ shows the data dimension.
Also, we assume that $y^{(i)}$ takes one of the $K$ possible class values~$\{1, 2, \dots, K\}$.
Furthermore, let ${f_{\boldsymbol{\theta}}: \mathbb{R}^{d} \rightarrow \mathbb{R}^{K}}$ denote a neural network classifier parameterized by ${\boldsymbol{\theta}}$ that takes an image $\boldsymbol{x}$ and outputs a real-valued vector ${\boldsymbol{z} = f_{\boldsymbol{\theta}}(\boldsymbol{x})}$ known as the logit.
The final decision of the classifier is determined via $\hat{y} = \argmax_{j} {z}_{j}$.
To train the classifier, one usually aims to minimize the empirical error between the ground-truth labels and the classifier predictions:
\begin{equation}\label{eq:empirical_error}
    \small{\argmin_{\boldsymbol{\theta}}\mathbb{E}_{(\boldsymbol{x}, y) \in \mathcal{D}}[\ell(f_{\boldsymbol{\theta}}(\boldsymbol{x}), y)]},
\end{equation}
where $\ell(\cdot)$ denotes the cross-entropy loss.

Following the convention in availability attacks, we assume that there exists a data curator that manipulates the dataset $\mathcal{D}$ into ${\mathcal{D}_{\rm pr}=\{(\tilde{\boldsymbol{x}}^{(i)}, y^{(i)})\}_{i=1}^{n}}$ such that once a neural network is trained over $\mathcal{D}_{\rm pr}$, it performs poorly over the clean data~$\mathcal{D}$:
\begin{align}\label{eq:threat_model}\nonumber
    & \small{\argmax_{\mathcal{D}_{\rm pr}} \mathbb{E}_{(\boldsymbol{x}, y) \in \mathcal{D}}[\ell(f_{\boldsymbol{\theta}^{*}}(\boldsymbol{x}), y)]}\\
    \small{\text{s.t.}~\boldsymbol{\theta}^{*}} &\small{= \argmin_{\boldsymbol{\theta}}\mathbb{E}_{(\boldsymbol{x}, y) \in \mathcal{D}_{\rm pr}}[\ell(f_{\boldsymbol{\theta}}(\boldsymbol{x}), y)]}.
\end{align}
Since each image $\tilde{\boldsymbol{x}}^{(i)}$ needs to maintain its normal utility, it is assumed that $\tilde{\boldsymbol{x}}^{(i)} = \boldsymbol{x}^{(i)} + \boldsymbol{\delta}^{(i)}$.
Here, $\boldsymbol{\delta}^{(i)}$'s are the data-protecting perturbations such that $\norm{\boldsymbol{\delta}^{(i)}}_{p} \leq \varepsilon$, where $\norm{\cdot}_{p}$ denotes the $L_p$ norm.

\subsection{dAta aVailAbiliTy Attacks defuseR~(\textsc{Avatar})}\label{sec:avatar}
As discussed, large pre-trained generative models can pose a threat to availability attacks and personal data protection.
In this section, we show how diffusion models, which are the SOTA in image generation, can be leveraged to cancel out the effects of availability attacks.

Recall that availability attacks provide a manipulated version of the original data~$\boldsymbol{x}$ that is seemingly unexploitable.
At the same time, the protected image~$\tilde{\boldsymbol{x}} = \boldsymbol{x} + \boldsymbol{\delta}$ should have its normal utility as it is going to be used by the users, e.g., to post over their social media.
This condition reflects itself through the constraint that $\norm{\boldsymbol{\delta}}_{p} \leq \varepsilon$.

A trivial idea would be to add \textit{random} noise to the protected perturbation that might counteract the perturbation, but this is detrimental/ineffective in removing the unlearnable effect~\citep{huang2021emn}.
As such, we propose to use a diffusion model for denoising as outlined next.\footnote{Note that while here we use DDPMs~\citep{ho2020ddpm} to demonstrate our method, it can be easily extended to other types of diffusion models as they are all different ways of representing the same process~\citep{song2021scoresde}.}

Specifically, let us assume that we have a pre-trained DDPM~\citep{ho2020ddpm} model that represents the data distribution~$\boldsymbol{x}_0 \sim p_{\rm data}(\boldsymbol{x})$.
The forward process of this model is represented using a Markov chain of length $T$, such that:
\begin{equation}\label{eq:forward_process}
    \small
    \boldsymbol{x}_{t} = \sqrt{1 - \beta_{t}} \boldsymbol{x}_{t-1} + \sqrt{\beta_{t}} \boldsymbol{\epsilon}_{t},
\end{equation}
where $\boldsymbol{\epsilon}_{t} \sim \mathcal{N}(\boldsymbol{0}, \mathbf{I})$ is the normal distribution, and ${t=1, 2, \dots, T}$. 
The constants $\beta_{t}$, known as variance schedules, are selected such that $\boldsymbol{x}_{T} \sim \mathcal{N}(\boldsymbol{0}, \mathbf{I})$.
If we set ${\alpha_{t}:=\prod_{s=1}^t\left(1-\beta_{s}\right)}$, then this Markov process can also be performed via a single step~\citep{ho2020ddpm}:
\begin{equation}\label{eq:forward_process_2}
    \small
    \boldsymbol{x}_{t} =  \sqrt{\alpha_{t}} \boldsymbol{x}_{0} + \sqrt{1 - \alpha_{t}} \boldsymbol{\epsilon}.
\end{equation}
The reverse of this process is also a variational Markov chain which is represented by:
\begin{equation}\label{eq:reverse_process}
    \small
    \boldsymbol{x}_{t-1}=\frac{1}{\sqrt{1-\beta_{t}}}\left(\boldsymbol{x}_{t}+\beta_{t} \mathbf{s}_{\boldsymbol{\phi}}(\boldsymbol{x}_{t}, t)\right)+\sqrt{\beta_{t}} \boldsymbol{\epsilon}_{t}.
\end{equation}
Here, $\mathbf{s}_{\boldsymbol{\phi}}(\cdot, t)$ is a network parameterized by $\boldsymbol{\phi}$ representing the score of the noisy data distribution at scale~$t$.

\begin{figure*}[tb!]
    \centering
    \pgfmathdeclarefunction{gauss}{2}{
        \pgfmathparse{1/(#2*sqrt(2*pi))*exp(-((x-#1)^2)/(2*#2^2))}%
    }
    \newsavebox\mybox
    \begin{lrbox}{\mybox}
 	\tikzstyle{node}=[very thick,circle,draw=myblue,minimum size=22,inner sep=0.5,outer sep=0.6]
    \tikzstyle{connect}=[-,thick]
    \tikzset{ 
      node 1/.style={node,black,draw=darkgreen,fill=black!15},
      node 2/.style={node,black,draw=blue,fill=black!15},
      node 3/.style={node,black,draw=red,fill=black!15},
    }
    \def\nstyle{int(\lay<\Nnodlen?min(2,\lay):3)}
    \resizebox{0.4\textwidth}{!}{
    \begin{tikzpicture}[x=2.4cm,y=1.2cm]
      \node [inner sep=0pt] at (2.40,2.20) {\Large DNN};
      \draw [dashed, fill=blue!20] (0.65,-3.45) rectangle (4.35,1.90);
    
      \readlist\Nnod{4,5,4,3} 
      \readlist\Nstr{n,m_1,m_2,k} 
      \readlist\Cstr{x,h^{(\prev)}, y} 
      \def\yshift{0.55} 
      
      \foreachitem \N \in \Nnod{
        \def\lay{\Ncnt} 
        \pgfmathsetmacro\prev{int(\Ncnt-1)} 
        \foreach \i [evaluate={\c=int(\i==\N); \y=\N/2-\i-\c*\yshift;
                     \x=\lay; \n=\nstyle;
                     \index=(\i<\N?int(\i):"\Nstr[\Ncnt]");}] in {1,...,\N}{ 
          \node[node \n] (N\lay-\i) at (\x,\y) {};
          
          \ifnumcomp{\lay}{>}{1}{ 
            \foreach \j in {1,...,\Nnod[\prev]}{ 
              \draw[white,line width=1.2,shorten >=1] (N\prev-\j) -- (N\lay-\i);
              \draw[connect] (N\prev-\j) -- (N\lay-\i);
            }
            \ifnum \lay=\Nnodlen
              \draw[connect] (N\lay-\i) --++ (0.5,0); 
            \fi
          }{
            \draw[connect] (0.5,\y) -- (N\lay-\i); 
          }
          
        }
        \path (N\lay-\N) --++ (0,1+\yshift) node[midway,scale=1.6] {$\vdots$}; 
      }
      
    \end{tikzpicture}}
    \end{lrbox}
 	\resizebox{0.8\textwidth}{!}{
			\tikzstyle{function} = [rectangle, draw, fill=blue!20, text width=.95cm, text badly centered, node distance=1cm, minimum height=.95cm, inner sep=0pt]
			\tikzstyle{line} = [draw, -latex']
			\tikzstyle{function2} = [draw, regular polygon, regular polygon sides=6, fill=red!20, node distance=1.5cm, minimum height=2em]
			\tikzstyle{branch}=[fill,shape=circle,minimum size=3pt,inner sep=0pt]
			\tikzstyle{block} = [circle, draw, fill=green!20, text width=0.3cm, text centered, rounded corners, minimum height=0.3cm]
			
			\begin{tikzpicture}[node distance = 1cm, auto]
            \draw [fill=black!10, rounded corners=1cm] (1.25,-1.50) rectangle (3.825,4.00);
            \draw [fill=black!10, rounded corners=0.5cm] (6.75,-1.50) rectangle (10.375,1.50);
			\node [inner sep=0pt] at (0,0) (Inp) {\includegraphics[width=2cm,height=2cm]{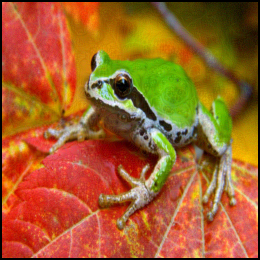}};
			\node [block] at (2.6,0) (add) {+};
			\node [inner sep=0pt] at (1.0,2.5) (Per) {\resizebox{0.10\textwidth}{!}{\begin{axis}[every axis plot post/.append style={
                                                                  mark=none,domain=-4:4,samples=60,smooth,line width=2.5pt},
                                                                  axis line style=ultra thick,
                                                                  x axis line style={opacity=0},
                                                                  xtick=\empty,
                                                                  y axis line style={opacity=0},
                                                                  ytick=\empty,
                                                                  enlargelimits=upper]
                                                                  \addplot {gauss(0,1)};
                                                                \end{axis}}};
			\node [inner sep=0pt] at (5.375,0) (Adv) {\includegraphics[width=2cm,height=2cm]{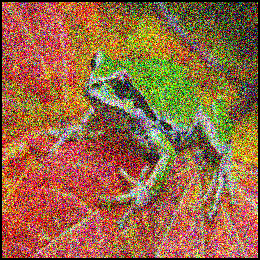}};
            \node [inner xsep=0pt, outer sep=0pt, trapezium, line width=0.4mm, trapezium angle=67.5, draw, rotate=270, fill=blue!20, minimum height=1.2cm, text width=4] (unet1) at (8,0) {};
            \node [inner xsep=0pt, outer sep=0pt, trapezium, line width=0.4mm, trapezium angle=67.5, draw, rotate=90, fill=blue!20, minimum height=1.2cm, text width=4] (unet2) at (9.2,0) {};
            \node [inner sep=0pt] at (12.00,0) (San) {\includegraphics[width=2cm,height=2cm]{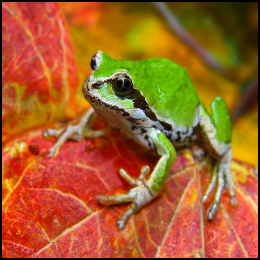}};
            
			\draw[-{Latex[length=3mm, width=1.5mm]}, line width=0.4mm] (Inp) -- (add);
			\draw[-{Latex[length=3mm, width=1.5mm]}, line width=0.4mm] (2.6, 2) -- (add);
			\draw[-{Latex[length=3mm, width=1.5mm]}, line width=0.4mm] (add) -- (Adv);
            \draw[-{Latex[length=3mm, width=1.5mm]}, line width=0.4mm] (Adv) -- (unet1);
            \draw[-{Latex[length=3mm, width=1.5mm]}, line width=0.4mm] (unet2.south east) to [bend right=90, out=-135, in=-65] (unet1.south west);
            \draw[-{Latex[length=3mm, width=1.5mm]}, line width=0.4mm] (unet2) -- (San);
            
			\node [inner sep=0pt] at (0,-1.25) (mul) {\small Training Data};
            \node [inner sep=0pt] at (0,-1.55) (mul) {\small(possibly protected)};
			\node [inner sep=0pt] at (2.6,2.3) (mul) {\small Gaussian Noise};
			\node [inner sep=0pt] at (5.375,-1.25) (mul) {\small Noisy Data};

            \node [inner sep=0pt] at (2.600,-0.75) (mul) {\small Forward Process};

            \node [inner sep=0pt] at (8.625, -1.25) (mul) {\small Reverse Process};
            \node [inner sep=0pt] at (12.00,-1.25) (mul) {\small Sanitized Data};

            \node [inner sep=0pt] at (7.00, 3.00) (mul) {\small Diffusion Model};
            \draw[-{Latex[length=3mm, width=1.5mm]}, line width=0.4mm] (3.825, 3.00) -- (5.75, 3.00);
            \draw[-{Latex[length=3mm, width=1.5mm]}, line width=0.4mm] (7.25, 1.50) -- (7.25, 2.75);
            
			\end{tikzpicture}
   }
   \caption{Overview of \textsc{Avatar}.
            According to a pre-trained diffusion model, we first add a controlled amount of Gaussian noise to the training data.
            Then, we use the reverse diffusion process to denoise the data which is later going to be used for neural network training.}
   \label{fig:avatar}
\end{figure*}
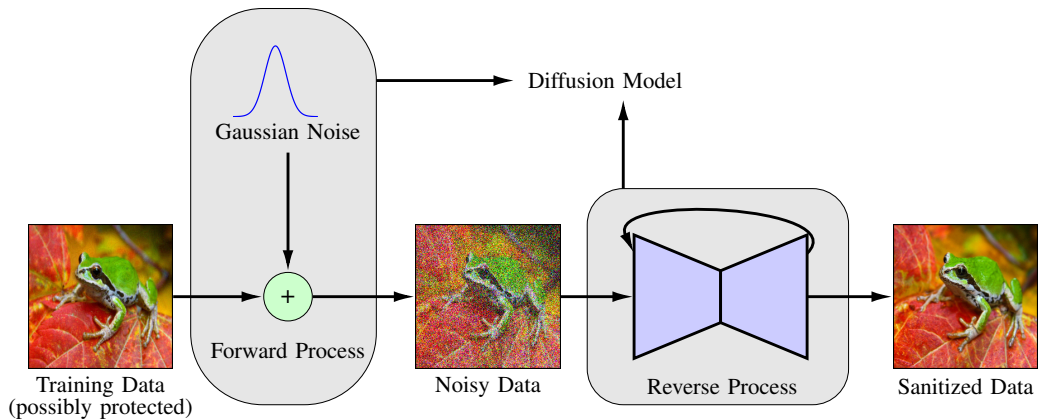

To cancel the effects of the data-protecting perturbations, we propose to first add Gaussian noise to the data.
The amount of noise should be adjusted in a way that each image maintains its visual appearance.
Otherwise, the semantic information of each image would be lost, and since the reverse process is probabilistic, the original image might not be recovered.
In particular, let $\tilde{\boldsymbol{x}}$ be a protected image.
We perform the forward process up to a step $t^{*} < T$ such that the semantic information of the image is preserved:
\begin{equation}\label{eq:noise_addition}
    \small
    \bar{\boldsymbol{x}}_{t^{*}} =  \sqrt{\alpha_{t^{*}}} \tilde{\boldsymbol{x}} + \sqrt{1 - \alpha_{t^{*}}} \boldsymbol{\epsilon}.
\end{equation}
Now, we have managed to diminish the effects of the data-protecting perturbation in $\boldsymbol{x}_{t^{*}}$.
However, this way we would also damage the semantic features of the data which makes it hard to train a neural network model~(see the ablation study in~\Cref{fig:timestep}).
To revert to the normal image space, we use the reverse process of our diffusion model to denoise the data:
\begin{align}\label{eq:denoising}
    \small
    \bar{\boldsymbol{x}}_{t-1}=\frac{1}{\sqrt{1-\beta_{t}}}\left(\bar{\boldsymbol{x}}_{t}+\beta_{t} \mathbf{s}_{\boldsymbol{\phi}}(\bar{\boldsymbol{x}}_{t}, t)\right)+\sqrt{\beta_{t}} \boldsymbol{\epsilon}_{t}.
\end{align}
Recursively solving \Cref{eq:denoising} from $t^{*}$ to $1$, we get a denoised version of the data which we denote by~${\bar{\boldsymbol{x}} = \bar{\boldsymbol{x}}_0}$.
Using this process, , shown in \Cref{fig:avatar}, we unlock the entire dataset $\mathcal{D}_{\rm pr}$, and construct a new one ${\mathcal{D}_{\rm de} = \{(\bar{\boldsymbol{x}}^{(i)}, y^{(i)})\}}$ for neural network training.
\Cref{alg:avatar} shows our final algorithm for training a neural network using \textsc{Avatar}.

\begin{algorithm}[t!]
	\caption{\label{alg:avatar} dAta aVailAbiliTy Attacks defuseR} 
	\begin{small}
	\textbf{Input}: protected dataset ${\mathcal{D}_{\rm pr}=\{(\tilde{\boldsymbol{x}}^{(i)}, y^{(i)})\}_{i=1}^{n}}$, pre-trained diffusion model~$\mathbf{s}_{\boldsymbol{\phi}}(\cdot, t)$.\vspace*{0.25em}
	\\
	\textbf{Output}: trained neural network classifier~${f_{\boldsymbol{\theta}}(\cdot)}$.\vspace*{0.25em}
	\\
	\textbf{Parameters}: noise time-step $t^{*}$, learning rate~$\alpha$, total epochs~$E$, and batch-size $b$.\vspace*{0.25em}
	\begin{algorithmic}[1]
		\State Initialize~${\boldsymbol{\theta}}$ randomly.
        \State Set $\mathcal{D}_{\rm de} = \{\}$.
        \For {$(\tilde{\boldsymbol{x}}, y)$ in $\mathcal{D}_{\rm pr}$}
		    \State $\small \bar{\boldsymbol{x}}_{t^{*}} =  \sqrt{\alpha_{t^{*}}} \tilde{\boldsymbol{x}} + \sqrt{1 - \alpha_{t^{*}}} \boldsymbol{\epsilon}$.\vspace*{0.15em}
            \For {$t$ in $t^{*}, t^{*}-1, \cdots, 0$}
            \State $\small \bar{\boldsymbol{x}}_{t-1}=\frac{1}{\sqrt{1-\beta_{t}}}\left(\bar{\boldsymbol{x}}_{t}+\beta_{t} \mathbf{s}_{\boldsymbol{\phi}}(\bar{\boldsymbol{x}}_{t}, t)\right)+\sqrt{\beta_{t}} \boldsymbol{\epsilon}_{t}$.\vspace*{0.15em}
            \EndFor
            \State Add $(\bar{\boldsymbol{x}}_0, y)$ to the dataset $\mathcal{D}_{\rm de}$.
		\EndFor
		\For {$i=1,2,\ldots, E$}
        \State Assign $\mathcal{D}_{\rm de}$ to batches of size $b$ randomly.
		\For {$\mathrm{batch}$ in $\mathrm{batches}$}
		    \State $\boldsymbol{\theta}\leftarrow\text{\textsc{SGD}}\left(\mathrm{batch}, f_{\boldsymbol{\theta}}, \alpha\right)$.\vspace*{0.15em}
		\EndFor
		\EndFor
	\end{algorithmic}
	\end{small}
\end{algorithm}

\subsection{Conflicting Assumptions in Availability Attacks}\label{sec:sec:setting_t}

So far, we discussed how by using diffusion models we can nullify the effects of the data-protecting perturbations.
Here, we take a theoretical perspective on our proposed solution and show that in this setting, the two constraints of availability attacks conflict with each other.
Specifically, from the perspective of availability attacks our result indicates that for a better data protection against \textsc{Avatar}, we need larger perturbation norms.
However, enlarging the perturbation is in conflict with retaining data utility which is the ultimate aim of availability attacks as discussed in \Cref{sec:avatar}.

\begin{theorem}\label{thm:convergence}
    Let $\boldsymbol{x} \in \mathbb{R}^{d}$ denote a clean image and ${\tilde{\boldsymbol{x}} = \boldsymbol{x} + \boldsymbol{\delta}}$ its protected version, where ${\boldsymbol{\delta}}$ denotes any arbitrary data protection perturbation.
    Also, let $\bar{\boldsymbol{x}}_{0}$ be the sanitized image using the \textsc{Avatar} denoising process given in \Cref{eq:noise_addition,eq:denoising}.
    If we set $t^{*}$ such that
    $$\small 2\log\left(\frac{2\norm{\boldsymbol{\delta}}^{2} + 4d}{\mu \Delta}\right) \leq t^{*} \beta_{t^{*}} \leq \frac{\mu \Delta}{4d},$$
    then the estimation error between the sanitized $\bar{\boldsymbol{x}}_{0}$ and clean image $\boldsymbol{x}$ can be bounded as:
    $$\small \mathbb{E}\left[\norm{\bar{\boldsymbol{x}}_{0} - \boldsymbol{x}}^2\right] \leq 2(\mu + 1)\Delta,$$
    where $\small \Delta = \mathbb{E}[\norm{\boldsymbol{x}_0 - \boldsymbol{x}}^{2}]$ and $\mu > 0$ is a constant.
\end{theorem}
\begin{proof}
    See Appendix~\ref{sec:proofs} for our proof using the contraction property of stochastic difference equations.
\end{proof}

\Cref{thm:convergence} states that for a protected image with a larger perturbation norm $\norm{\boldsymbol{\delta}}$, a larger amount of noise~(determined by~$t^{*} \beta_{t^{*}}$) is required.
However, the amount of noise cannot be arbitrarily large as the semantic information of the image might be lost in the process~(as indicated by the presence of $\Delta$ in the upper-bound).

\subsection{\textsc{Avatar} vs.~Adversarial Training}\label{sec:discussion}

As~\citet{tao2021preventing} have demonstrated, adversarial training~(AT)~\citep{madry2018towards} could also be used to train successful models over unexploitable data.
However, our approach has several key advantages compared to AT:
\begin{enumerate}[1)]\setlength\itemsep{-0.5pt}
        \item First, AT modifies the learning algorithm, and as such, it needs to be applied separately for training each neural network.
        In contrast, \textsc{Avatar} sanitizes the data only \textit{once}.
        As a result, \textsc{Avatar} is more efficient.
        \item Second, as shown by \citet{tsipras2019robustness}, AT greatly affects the clean accuracy in its learning process, and as such, might not be the ultimate method for defending against availability attacks.
        \item Lastly, as \citet{fu2022remn} show, one can build unexploitable data against AT that would essentially render AT vulnerable to availability attacks. However, to the best of our knowledge, no adaptive availability attacks have been proposed against diffusion models so far.
\end{enumerate}

\section{Experimental Results}\label{sec:experiments}
In this section, we run various experiments to analyze the performance of \textsc{Avatar} against availability attacks:
\begin{enumerate}
    \item We conduct extensive experiments on seven SOTA availability attacks and show that given the data distribution, \textsc{Avatar} can counteract them~(\Cref{sec:sec:exploiting}).
    \item We provide detailed comparisons against various pre-processing techniques~(\Cref{sec:sec:augmentation}), early stopping~(\Cref{sec:sec:early_stopping}), and adversarial training~(\Cref{sec:sec:AT}) to show that \textsc{Avatar} delivers the best performance.
    \item We provide extensive ablation studies into different assumptions made by \textsc{Avatar}.
    First, we show that the training data overlap between the diffusion model and the unlearnable example generation has \textbf{no effect} on the performance of \textsc{Avatar}~(\Cref{sec:sec:training_data}). Interestingly, we show that even a similar, different, or even poisoned distribution compared to the true data distribution can counteract availability attacks~(\Cref{sec:sec:mismatch}).
    \item We simulate our scenario given in \Cref{fig:threat_model} for the real-world application of facial recognition to show the plausibility of our approach. Again, here we use a diffusion model trained on a different dataset, but we manage to counteract the unlearnable examples for another dataset~(\Cref{sec:sec:real_world}).
\end{enumerate}
We also include an extended version of our experimental results in Appendix~\ref{sec:additional_experiments}.

\subsection{Details of Experimental Settings}\label{sec:sec:settings}
In this section, we provide the details of our experimental settings.

\paragraph{Datasets}
In our experiments, we use four different datasets.
CIFAR-10 \& 100~\citep{krizhevsky2009learning} are $32\times32$ datasets of colored images, where the classes contain different objects, animals, plants, etc.
SVHN~\citep{netzer2011svhn} is a dataset of house numbers from 0 to 9 in a natural, street view setting.
Finally, ImageNet~\citep{russakovsky2015imagenet} is a dataset of natural images of size $224 \times 224$ with 1000 classes.
In our experiments, we use two simplified versions of this dataset.
First, following the convention of prior research, we select the first 100 classes of this dataset, which we refer to as ImageNet~(IN)-100.
Second, for our distribution mismatch experiments, we follow \citet{huang2021emn} and select 10 classes of ImageNet that are closely aligned with CIFAR-10 and downscale them to $32\times32$ size.
We call this dataset IN-10.
The information on the selected classes can be found in \Cref{tab:imagenet10}.
Finally, we also use the 32$\times$32 version of the ImageNet dataset for some of our experiments, which we denote by IN-1k-32$\times$32.

\begin{table}[tb!]
	\caption{The classes of CIFAR-10 and their matching ones in the ImageNet-10 dataset.}
	\label{tab:imagenet10}
	\begin{center}
		\begin{small}
		    \setlength\tabcolsep{0.45em}
			\begin{tabular}{ll}
				\toprule
                \textbf{CIFAR-10}              & \textbf{IN-10}\\
                \midrule
				Airplane                       & Airliner\\
                Automobile                     & Wagon\\
                Bird                           & Humming Bird\\
                Cat                            & Siamese Cat\\
                Deer                           & Ox\\
                Dog                            & Golden Retriever\\
                Frog                           & Tailed Frog\\
                Horse                          & Zebra\\
                Ship                           & Container Ship\\
                Truck                          & Trailer Truck\\
			    \bottomrule
			\end{tabular}
		\end{small}
	\end{center}
\end{table}

\paragraph{Classifiers}
In our experiments, we use four types of neural network image classifiers, namely: ResNet-18~(RN-18)~\citep{he2016deep}, VGG-16~\citep{simonyan2015vgg}, DenseNet-121~(DN-121)~\citep{huang2017densenet}, and WideResNet-34~(WRN-34)~\citep{zagoruyko2016wresnet}.
For training these classifiers over different datasets and also training objectives (vanilla vs.~adversarial training~(AT)), we follow two different training conventions.
The hyper-parameters of each setting are given in \Cref{tab:settings}.
Furthermore, \Cref{tab:exp_setting} indicates the setting used for each experiment in the paper.

\begin{table*}[tb!]
	\caption{Training hyper-parameters used in our experiments.}
	\label{tab:settings}
	\begin{center}
		\begin{small}
			\begin{tabular}{lcc}
				\toprule
				\textbf{Hyper-parameter}           & \textbf{Setting~\#1}             & \textbf{Setting~\#2} \\
				\midrule
				Optimizer                          & SGD                              & SGD\\
				Scheduler                          & Multi-step                       & Multi-step\\
				Initial lr.                        & 0.1                              & 0.1\\
				lr. decay                          & 0.1 (@epoch: 80 \& 100)          & 0.1 (@iter: 16k \& 32k)\\
				Batch Size                         & 128                              & 128\\
				Training Steps                     & 120 (epochs)                     & 40k (iters)\\
                Weight Decay                       & 0.0005                           & 0.0005\\
                PGD Steps (for AT only)            & -                                & 10\\
                PGD Step Size (for AT only)        & -                                & 0.8\\
				\bottomrule
			\end{tabular}
		\end{small}
	\end{center}
    \vskip -0.1 in
\end{table*}

\begin{table}[tb!]
	\caption{Setting number used for each experiment.}
	\label{tab:exp_setting}
	\begin{center}
		\begin{small}
			\begin{tabular}{lcc}
				\toprule
				\textbf{Experiment}                                  & \textbf{Setting~\#1}             & \textbf{Setting~\#2} \\
				\midrule
				\Cref{tab:architecture,tab:architecture}~(CIFAR-10)  & \checkmark                       & -\\
                \Cref{tab:architecture,tab:architecture}~(CIFAR-100) & \checkmark                       & -\\
                \Cref{tab:architecture,tab:architecture}~(SVHN)      & \checkmark                       & -\\
                \Cref{tab:architecture,tab:architecture}~(IN-100)    & -                                & \checkmark\\
                \Cref{tab:data_aug}                                  & \checkmark                       & -\\
                \Cref{tab:early}                                     & \checkmark                       & -\\
                \Cref{tab:dist_mismatch}                             & \checkmark                       & -\\
                \Cref{fig:timestep}                                  & \checkmark                       & -\\
                \Cref{fig:at_all}                                    & -                                & \checkmark\\
				\bottomrule
			\end{tabular}
		\end{small}
	\end{center}
\end{table}

\paragraph{Diffusion Models}
For the diffusion models used during the denoising process of \textsc{Avatar} (shown in \Cref{fig:avatar}), we follow the implementation of DiffPure\footnote{\url{https://github.com/NVlabs/DiffPure}} and use score SDE~\citep{song2021scoresde}~(for CIFAR-10, CIFAR-100, SVHN, IN-10) and the guided DDPM (for IN-100 and IN-1k-32$\times$32.)~\citep{dhariwal2021diffusion}.
For CIFAR-10 and IN-100, we download the pre-trained versions available online.\footnote{For CIFAR-10, we used the checkpoint for the \texttt{vp/cifar10\_ddpmpp\_deep\_continuous} setting on score SDE repository: \url{https://github.com/yang-song/score_sde_pytorch}. 
Moreover, we used the unconditional $256 \times 256$ model available on the guided DDPM code-base for IN-100 expriments: \url{https://github.com/openai/guided-diffusion}.
Finally, we use the pre-trained DDPM-IP~\citep{ning2023ddpmip} models available on \url{https://github.com/forever208/DDPM-IP}} for IN-1k-32$\times$32 dataset.
Additionally, for CIFAR-100, IN-10, and SVHN we use the PyTorch repository of score SDE~\citep{song2021scoresde}, and train variance-preserving diffusion models with continuous DDPM++ architecture, similar to the one used for CIFAR-10.
The FID score of the trained diffusion models is given in \Cref{tab:FID}.

\paragraph{Availability Attacks}
We use seven SOTA availability attacks in our experiments: DeepConfuse~(CON)~\citep{feng2019con}, Neural Tangent Generalization Attacks~(NTGA)~\citep{yuan2021ntga}, Error-minimizing Noise~(EMN)~\citep{huang2021emn}, Targeted Adversarial Poisoning~(TAP)~\citep{fowl2021tap}, Robust EMN~(REMN)~\citep{fu2022remn}, Shortcut~(SHR)~\citep{geirhos2020shortcut}, and Autoregressive attacks~(AR)~\citep{sandoval2022ar}.
The details of each availability attack are given below:
\begin{itemize}
    \item For CON~\citep{feng2019con}, we use the released protected CIFAR-10 dataset, available online at SHR~\citep{yu2022shr} repository.\footnote{\label{shr_repo}\url{https://github.com/dayu11/Availability-Attacks-Create-Shortcuts}}
    Note that since generating this attack for the CIFAR-10 dataset would take 5-7 days, we just used the available data for CIFAR-10 and skipped generating the attack for the other datasets.
    \item For NTGA~\citep{yuan2021ntga}, we use their code\footnote{\url{https://github.com/lionelmessi6410/ntga}} to generate availability attacks for our datasets.
    For CIFAR-10, we used the data published online.
    For CIFAR-100 and SVHN, we used the online repository, and generate NTGA protected data using the \texttt{CNN} surrogate model, \texttt{time-step} of 64, and \texttt{block-size} of 100 to generate perturbations of magnitude $\norm{\boldsymbol{\delta}}_{\infty} \leq 8/255$.
    Due to limited GPU memory, we used the \texttt{FNN} surrogate model to generate perturbations of magnitude $\norm{\boldsymbol{\delta}}_{\infty} \leq 0.1$  for IN-100.
    The rest of the hyper-parameters were set similarly to CIFAR-100 and SVHN.
    \item For EMN~\citep{huang2021emn}, TAP~\citep{fowl2021tap}, and REMN~\citep{fu2022remn}, we use the online repository of REMN\footnote{\url{https://github.com/fshp971/robust-unlearnable-examples}} which contains an implementation of EMN and TAP as well.
    We use the default CIFAR-10 configurations of this repository for CIFAR-10, CIFAR-100, and SVHN.
    For IN-10, we used the default MiniIN configurations of the REMN code-base.
    \item Moreover, we use the SHR GitHub repository\footnoteref{shr_repo} to generate shortcut attacks.
    For CIFAR-10, CIFAR-100, and SVHN, we use the default settings.
    For IN-100, we use \texttt{patchsize} of 32 as advised by the authors.
    \item Finally, we use the official data released on the AR GitHub repository for this attack.\footnote{\url{https://github.com/psandovalsegura/autoregressive-poisoning}}
\end{itemize}
A few samples for each availability attack are shown in \Cref{fig:IN_samples_III}.

\begin{table}[tb!]
	\caption{The FID of the diffusion models used for denoising. * denotes that the FID has been computed using ~10k generated samples only. \textdagger~indicates that the scores have been adapted from relative literature.}
	\label{tab:FID}
	\begin{center}
		\begin{small}
		    \setlength\tabcolsep{0.45em}
			\def\arraystretch{1.5}
			\begin{tabular}{cc|cc}
				\toprule
                \textbf{Dataset}                                  & \textbf{FID}       & \textbf{Dataset}                  & \textbf{FID}\\
                \midrule
				CIFAR-10\textsuperscript{\textdagger}             & $2.41$             & SVHN\textsuperscript{*}           & $2.59$\\
                CIFAR-10 (TAP)\textsuperscript{*}                 & $4.11$             & CIFAR-100\textsuperscript{*}      & $4.85$ \\
                IN-10\textsuperscript{*}                          & $17.32$            & IN\textsuperscript{\textdagger}   & $4.59$\\
			    \bottomrule
			\end{tabular}
		\end{small}
	\end{center}
\end{table}

\subsection{Exploiting Protected Data}\label{sec:sec:exploiting}
\Cref{tab:architecture} shows our results for breaking availability attacks against for four different datasets.
As can be seen, \textsc{Avatar} can significantly improve the performance of neural network training in almost all cases.
Moreover, although the training data was produced using diffusion models, the trained neural networks can generalize to unseen test data easily.
This trend is more evident in the CIFAR-10 and SVHN datasets where the pre-trained diffusion model can better represent the image data density, as indicated by their low FID scores.

\begin{table*}[tb!]
	\caption{Test accuracy (\%) of RN-18 architectures trained over data availability attacks on CIFAR-10, CIFAR-100, and SVHN, and ImageNet-100 datasets without and with our denoising approach. The mean and standard deviation are computed over 5 seeds. For our results over other architectures, please see \Cref{tab:architecture_appendix}.}
	\label{tab:architecture}
	\begin{center}
		\begin{small}
		    \setlength\tabcolsep{0.45em}
			\def\arraystretch{1.65}
			\begin{tabular}{lccccccccc}
				\toprule
                \multirow{2}{*}{\rotatebox[origin=c]{90}{{\textbf{Data}}}}
				&\multirow{2}{*}{\textbf{Method}}
                &\multirow{2}{*}{\textbf{Clean}}
				&\multicolumn{6}{c}{\textbf{Data Availability Attacks}}\\
				\cmidrule(lr){4-9}
				&&                                            & NTGA& EMN & TAP & REMN  & SHR & AR\\
				\midrule
                \multirow{2}{*}{\rotatebox[origin=c]{90}{\scriptsize CIFAR-10}}
				& Vanilla   &\multirow{2}{*}{$94.50 \pm 0.09$} & $11.49 \pm 0.69$  & $24.85 \pm 0.71$ & $7.86 \pm 0.90$  & $20.50 \pm 1.16$ & $10.82 \pm 0.22$ & $12.09 \pm 1.12$\\
				& \textsc{Avatar}                            && $87.95 \pm 0.28$  & $90.95 \pm 0.10$ & $90.71 \pm 0.19$ & $88.49 \pm 0.24$ & $85.69 \pm 0.27$ & $91.57 \pm 0.18$\\
                \midrule
                \multirow{2}{*}{\rotatebox[origin=c]{90}{\scriptsize SVHN}}
				& Vanilla   &\multirow{2}{*}{$96.29 \pm 0.12$} & $9.65 \pm 0.70$   & $9.13 \pm 2.00$  & $65.97 \pm 1.99$  & $11.55 \pm 0.19$ & $10.59 \pm 3.98$ & $6.76 \pm 0.07$\\
				& \textsc{Avatar}                            && $89.84 \pm 0.32$  & $93.84 \pm 0.12$ & $93.35 \pm 0.10$  & $88.51 \pm 0.23$ & $83.82 \pm 0.39$ & $94.13 \pm 0.17$\\
                \midrule
                \multirow{2}{*}{\rotatebox[origin=c]{90}{\scriptsize CIFAR-100}}
				& Vanilla   &\multirow{2}{*}{$75.01 \pm 0.41$} & $1.32 \pm 0.31$   & $2.05 \pm 0.18$  & $14.10 \pm 0.19$  & $10.88 \pm 0.33$  & $1.39 \pm 0.10$  & $2.15 \pm 0.46$\\
				& \textsc{Avatar}                            && $63.98 \pm 0.55$  & $65.73 \pm 0.36$ & $64.99 \pm 0.10$  & $64.88 \pm 0.08$  & $58.52 \pm 0.46$ & $64.54 \pm 0.23$\\
                \midrule
                \multirow{2}{*}{\rotatebox[origin=c]{90}{\scriptsize IN-100}}
				& Vanilla   &\multirow{2}{*}{$80.05 \pm 0.13$} & $74.74 \pm 0.52$  & $1.78 \pm 0.17$  & $9.14 \pm 0.40$   & $13.28 \pm 0.51$  & $43.48 \pm 1.56$\\
				& \textsc{Avatar}                            && $71.08 \pm 0.48$  & $72.84 \pm 0.90$ & $76.52 \pm 0.46$  & $39.79 \pm 0.98$  & $59.85 \pm 1.01$\\
			    \bottomrule
			\end{tabular}
		\end{small}
	\end{center}
\end{table*}

\begin{table*}[t!]
	\caption{\label{tab:data_aug} Test accuracy (\%) of RN-18 models trained over data availability attacks on CIFAR-10 dataset using different data augmentation/pre-processing techniques. The results are averaged over 5 runs. The best results are highlighted in bold.}
	\begin{center}
		\begin{small}
		    \setlength\tabcolsep{0.15em}
			\def\arraystretch{1.5}
			\begin{tabular}{ccccccccc}
				\toprule
				\multirow{2}{*}{\textbf{Method}}
                &\multirow{2}{*}{\textbf{Clean}}
				&\multicolumn{7}{c}{\textbf{Data Availability Attacks}}\\
				\cmidrule(lr){3-9}
				&& CON & NTGA & EMN & TAP & REMN & SHR & AR\\
				\midrule
				Vanilla                              & $94.50 \pm 0.09$ & $15.75 \pm 0.82$ & $11.49 \pm 0.69$ & $24.85 \pm 0.71$ & $7.86 \pm 0.90$  & $20.50 \pm 1.16$ & $10.82 \pm 0.22$ & $12.09 \pm 1.12$\\
                Cutout                               & $94.39 \pm 0.12$ & $13.53 \pm 0.34$ & $13.43 \pm 1.15$ & $23.79 \pm 1.28$ & $9.73 \pm 1.06$  & $20.48 \pm 1.09$ & $11.78 \pm 0.81$ & $11.21 \pm 1.01$\\
                MixUp                                & $94.87 \pm 0.05$ & $28.58 \pm 2.88$ & $13.54 \pm 0.36$ & $51.48 \pm 0.97$ & $30.09 \pm 1.93$ & $26.61 \pm 1.65$ & $19.69 \pm 0.71$ & $12.67 \pm 1.02$\\
                CutMix                               & $\mathbf{95.16 \pm 0.03}$           & $19.04 \pm 2.74$ & $14.16 \pm 1.64$ & $25.30 \pm 1.18$ & $7.45 \pm 1.21$  & $26.83 \pm 1.99$ & $10.89 \pm 0.34$ & $11.36 \pm 0.50$\\
                FAutoAug.                            & $95.11 \pm 0.14$ & $51.62 \pm 1.28$ & $27.56 \pm 2.45$ & $56.31 \pm 1.13$ & $20.39 \pm 0.81$ & $26.65 \pm 0.89$ & $25.88 \pm 0.62$ & $13.53 \pm 0.79$\\
                Median Blur                          & $85.83 \pm 0.71$ & $15.14 \pm 0.38$ & $28.43 \pm 1.41$ & $26.97 \pm 0.39$ & $57.16 \pm 0.75$ & $23.32 \pm 0.69$ & $17.50 \pm 0.38$ & $14.97 \pm 0.40$\\
                Gaus.~Blur                           & $94.33 \pm 0.08$ & $15.36 \pm 0.69$ & $11.86 \pm 0.81$ & $24.08 \pm 0.40$ & $8.87 \pm 0.65$  & $20.89 \pm 1.56$ & $11.39 \pm 1.91$ & $13.39 \pm 1.08$\\
                Quantization                         & $94.57 \pm 0.14$ & $15.17 \pm 0.79$ & $16.29 \pm 1.03$ & $25.38 \pm 0.68$ & $7.38 \pm 1.59$ & $22.33 \pm 1.21$ & $11.12 \pm 0.24$ & $12.87 \pm 0.69$\\
                TVM                                  & $73.20 \pm 1.32$ & $42.82 \pm 2.00$ & $47.41 \pm 1.37$ & $54.86 \pm 2.17$ & $70.66 \pm 0.58$ & $19.28 \pm 1.12$ & $25.35 \pm 2.54$ & $34.09 \pm 2.07$\\
                Grayscale                            & $92.89 \pm 0.20$ & $83.30 \pm 0.40$ & $63.21 \pm 0.85$ & $\mathbf{92.09 \pm 0.22}$  & $9.57 \pm 0.59$  & $75.84 \pm 1.36$ & $57.13 \pm 0.87$ & $35.88 \pm 0.99$\\
                JPEG                                 & $84.99 \pm 0.28$ & $82.87 \pm 0.23$ & $79.26 \pm 0.12$ & $84.65 \pm 0.16$ & $83.44 \pm 0.36$ & $83.66 \pm 0.30$ & $69.03 \pm 0.62$ & $85.03 \pm 0.23$\\
				\textsc{Avatar}~(Ours)               & $90.16 \pm 0.21$ & $\mathbf{89.43 \pm 0.09}$ & $\mathbf{87.95 \pm 0.28}$ & $90.95 \pm 0.10$ & $\mathbf{90.71 \pm 0.19}$ & $\mathbf{88.49 \pm 0.24}$ & $\mathbf{85.69 \pm 0.27}$ & $\mathbf{91.57 \pm 0.18}$\\
			    \bottomrule
			\end{tabular}
		\end{small}
	\end{center}
    \vskip -0.1in
\end{table*}

\begin{table*}[tb!]
	\caption{Test accuracy (\%) of RN-18 models trained over data availability attacks on CIFAR-10 dataset. The early stopping rows contain the highest achievable accuracy over the course of training. The results are averaged over 5 runs.}
    \label{tab:early}
	\begin{center}
		\begin{small}
		    \setlength\tabcolsep{0.40em}
			\def\arraystretch{1.5}
			\begin{tabular}{lccccccc}
				\toprule
				\multirow{2}{*}{\textbf{Method}}
				&\multicolumn{7}{c}{\textbf{Data Availability Attacks}}\\
				\cmidrule(lr){2-8}
				& CON & NTGA & EMN & TAP& REMN & SHR & AR\\
				\midrule
				Vanilla
                                                     & $15.75 \pm 0.82$ & $11.49 \pm 0.69$ & $24.85 \pm 0.71$ & $7.86 \pm 0.90$  & $20.50 \pm 1.16$ & $10.82 \pm 0.22$ & $12.09 \pm 1.12$\\
                ~~~+~Early Stopping                  & $23.99 \pm 6.22$ & $31.71 \pm 3.97$ & $27.23 \pm 1.83$ & $67.13 \pm 2.03$ & $21.90 \pm 0.57$ & $22.72 \pm 0.83$ & $38.78 \pm 8.65$\\
				\textsc{Avatar}~(Ours)               & $89.43 \pm 0.09$ & $87.95 \pm 0.28$ & $90.95 \pm 0.10$ & $90.71 \pm 0.19$ & $88.49 \pm 0.24$ & $85.69 \pm 0.27$ & $91.57 \pm 0.18$\\
                ~~~+~Early Stopping                  & $89.55 \pm 0.15$ & $88.07 \pm 0.22$ & $91.07 \pm 0.11$ & $91.00 \pm 0.11$ & $88.59 \pm 0.26$ & $85.76 \pm 0.25$ & $91.63 \pm 0.17$\\
			    \bottomrule
			\end{tabular}
		\end{small}
	\end{center}
    \vskip -0.2in
\end{table*}

\subsection{Comparison with Data Augmentation Techniques}\label{sec:sec:augmentation}
\textsc{Avatar} can be regarded as a type of data pre-processing where the inner mechanics of the learning algorithms are not modified.
As such, here we compare our approach with various SOTA data augmentation techniques that can be utilized during model training.
To this end, we follow the settings of \citep{huang2021emn}, and adopt four widely used data augmentation techniques.
In addition, we employ the JPEG and grayscale pre-processing~\citep{liu2023image} as well as two blurring techniques in~\Cref{tab:data_aug}.
Finally, we also test the quantization and total variation minimization~(TVM) approaches that have shown to be effective against adversarial attacks~\citep{guo2018tvm}.
\Cref{tab:data_aug} shows the performance of these methods compared to \textsc{Avatar}.
As shown, our approach outperforms various types of pre-processing/data augmentation methods.

\subsection{The Effect of Early Stopping}\label{sec:sec:early_stopping}
It has been previously shown that early stopping can also be beneficial against availability attacks~\citep{huang2021emn}.
As such, here we run the same set of experiments over availability attacks for the CIFAR-10 dataset, but this time we record the highest accuracy attainable during training.
\Cref{tab:early} shows our results.
As seen, using our approach one achieves stable training, where the variance between the final model accuracy and the highest attainable accuracy is very low.
Notably, while these results indicate that existing availability attacks are less powerful than what is thought, early stopping is not sufficient to recover the best model performance.
In contrast, \textsc{Avatar} can significantly cancel the effects of availability attacks.

\subsection{Comparison with Adversarial Training}\label{sec:sec:AT}
As mentioned in \Cref{sec:related}, adversarial training~(AT)~\citep{madry2018towards} is the most successful defense technique against availability attacks~\citep{tao2021preventing}.
For the next set of experiments, we follow the settings of \citet{fu2022remn} and compare our approach with AT.
To this end, we run two different scenarios.
First, we perform AT over the protected data.
Then, we run AT over the data that is defused (i.e., counteracted) by \textsc{Avatar}.
In both cases, we vary the perturbation bound $\varepsilon$ from 0 to 4, where 0 is the vanilla training.
\Cref{fig:at_all} shows our results.
Apart from what we discussed in \Cref{sec:discussion}, two additional insights are worth mentioning here:
\begin{enumerate}[(1)]\setlength\itemsep{-0.5pt}
        \item In most cases, \textsc{Avatar} without AT~(i.e.,~${\varepsilon = 0}$) performs on-par or better than AT with ${\varepsilon > 0}$.
        Thus, \textsc{Avatar} delivers the SOTA against availability attacks.
        \item As seen in \Cref{fig:at_all}, AT yields the worst performance against REMN~\citep{fu2022remn}.
        However, our approach can combat REMN~\citep{fu2022remn} successfully, and it is the first approach that does so.
\end{enumerate}

\begin{figure*}[p!]
	\centering
    \begin{subfigure}{0.8\textwidth}
	\centering
	\includegraphics[width=1.0\textwidth]{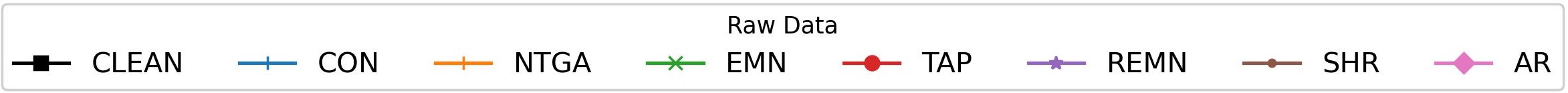}
    \end{subfigure}\\\vspace*{0.2em}
    \begin{subfigure}{0.7\textwidth}
	\centering
	\includegraphics[width=1.0\textwidth]{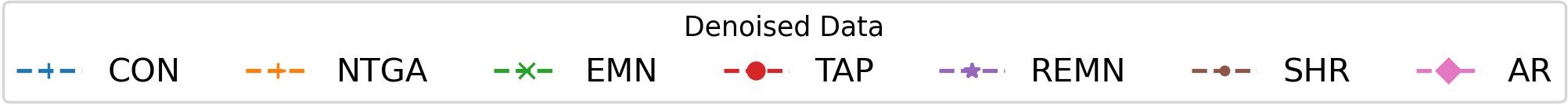}
    \end{subfigure}\\
    \begin{subfigure}{1.0\textwidth}
	\begin{subfigure}{.3\textwidth}
		\centering
		\includegraphics[width=1.0\textwidth]{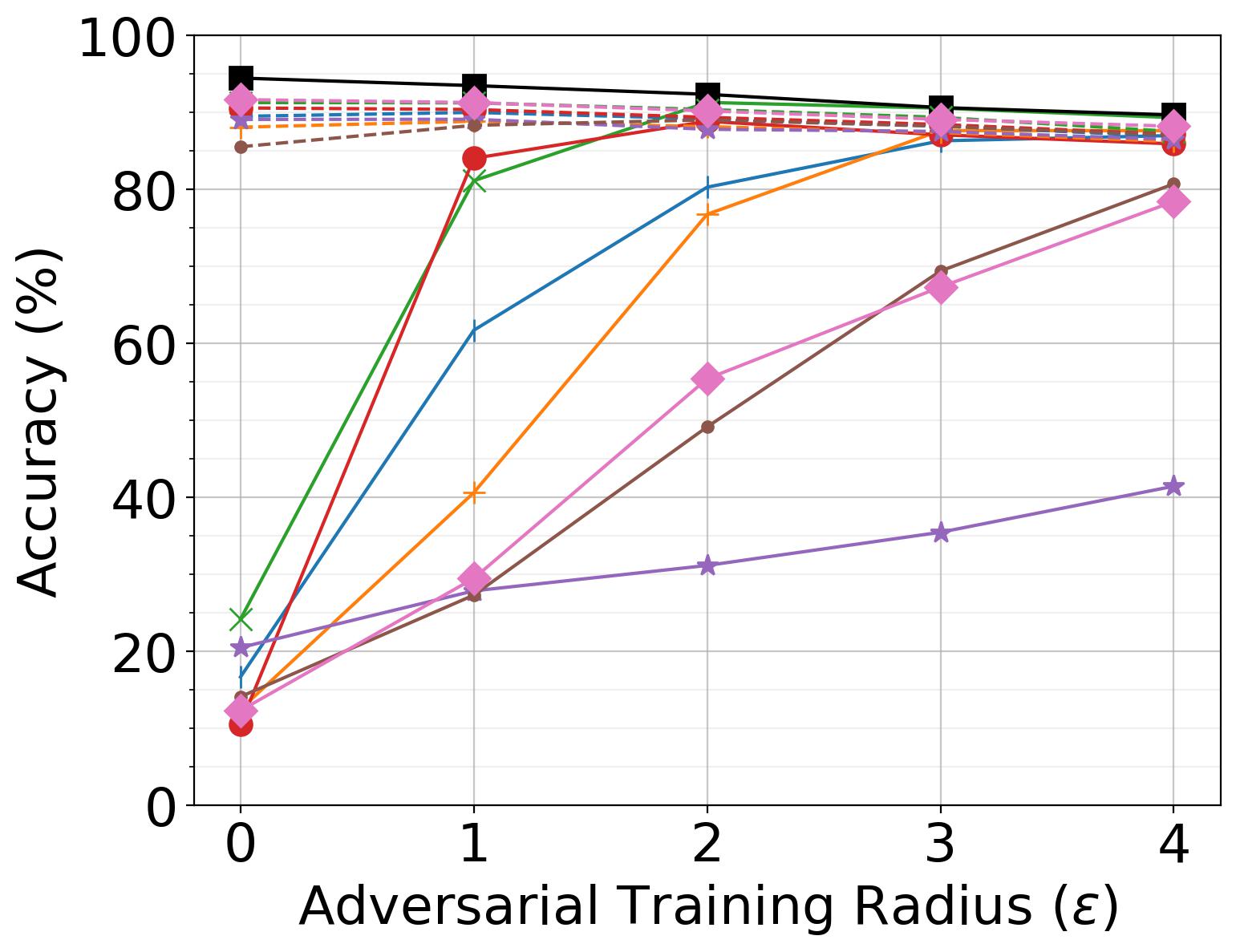}
		\caption*{CIFAR-10}
		\label{fig:at_cifar10:rn18}
	\end{subfigure}\hspace*{0.75em}
	\begin{subfigure}{.3\textwidth}
		\centering
		\includegraphics[width=1.0\textwidth]{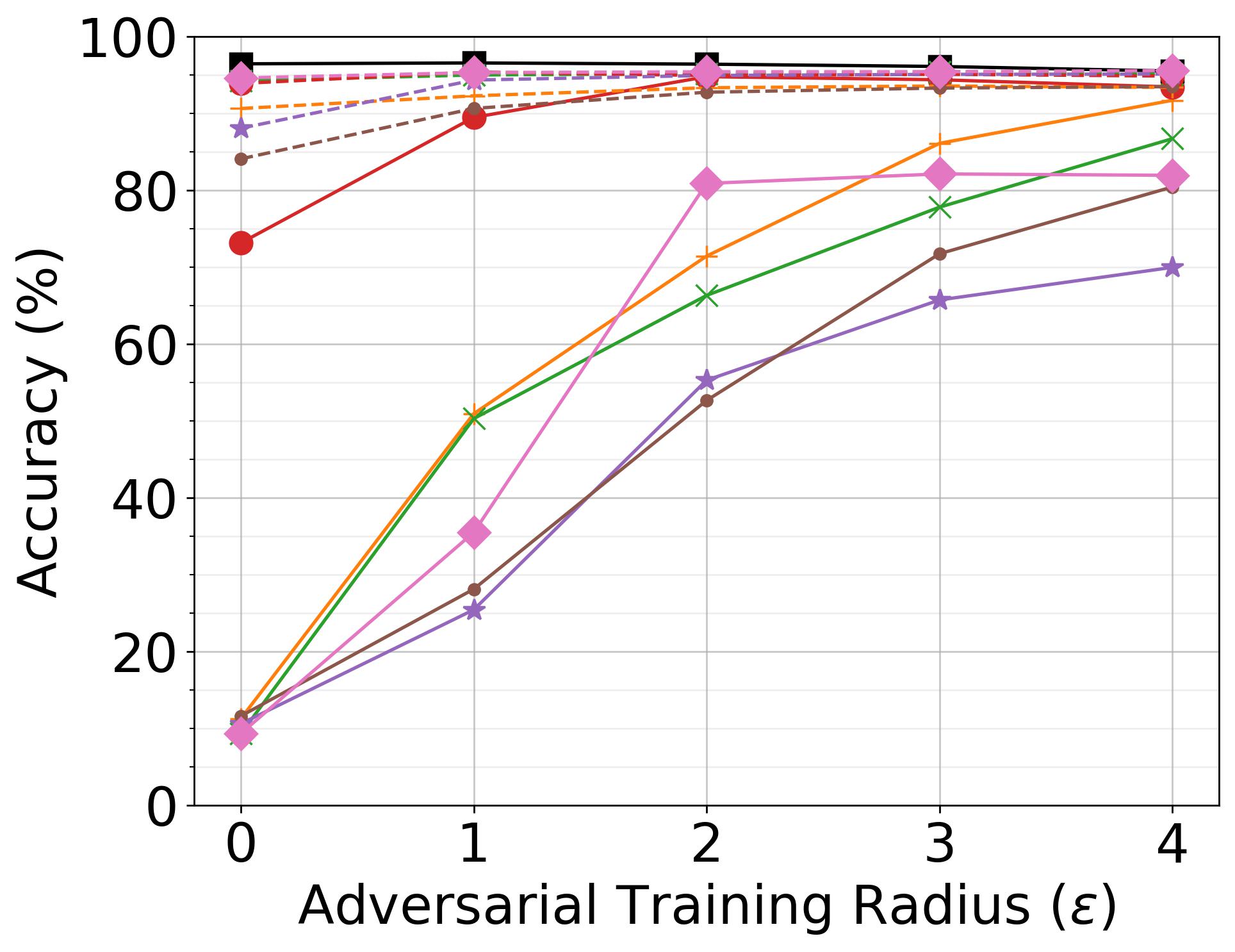}
		\caption*{SVHN}
		\label{fig:at_svhn:rn18}
	\end{subfigure}\hspace*{0.75em}
	\begin{subfigure}{.3\textwidth}
		\centering
		\includegraphics[width=1.0\textwidth]{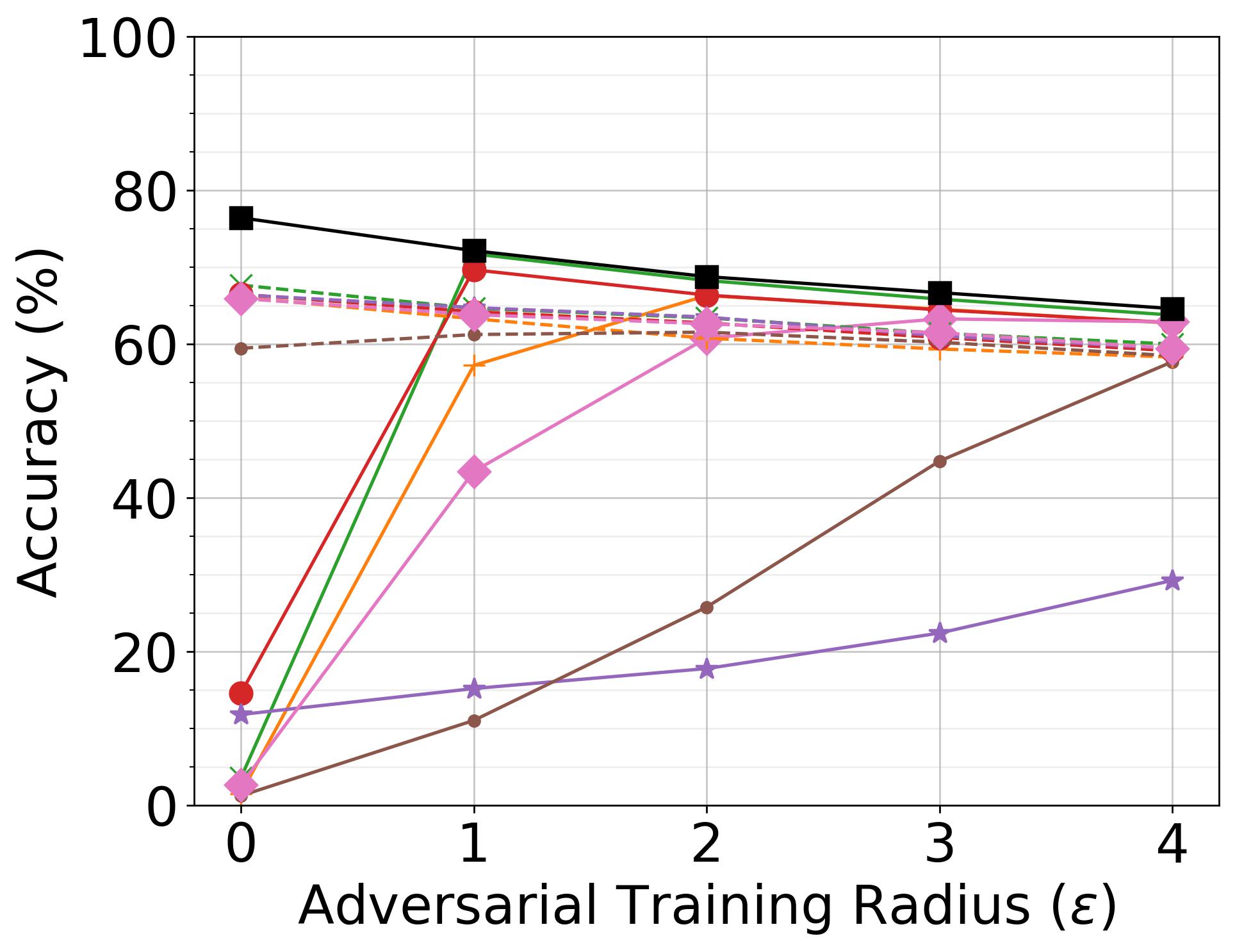}
		\caption*{CIFAR-100}
		\label{fig:at_cifar100:rn18}
	\end{subfigure}
    \caption{RN-18}
	\label{fig:at_rn18}
    \end{subfigure}\\\vspace*{0.5em}
	    \begin{subfigure}{1.0\textwidth}
	\begin{subfigure}{.3\textwidth}
		\centering
		\includegraphics[width=1.0\textwidth]{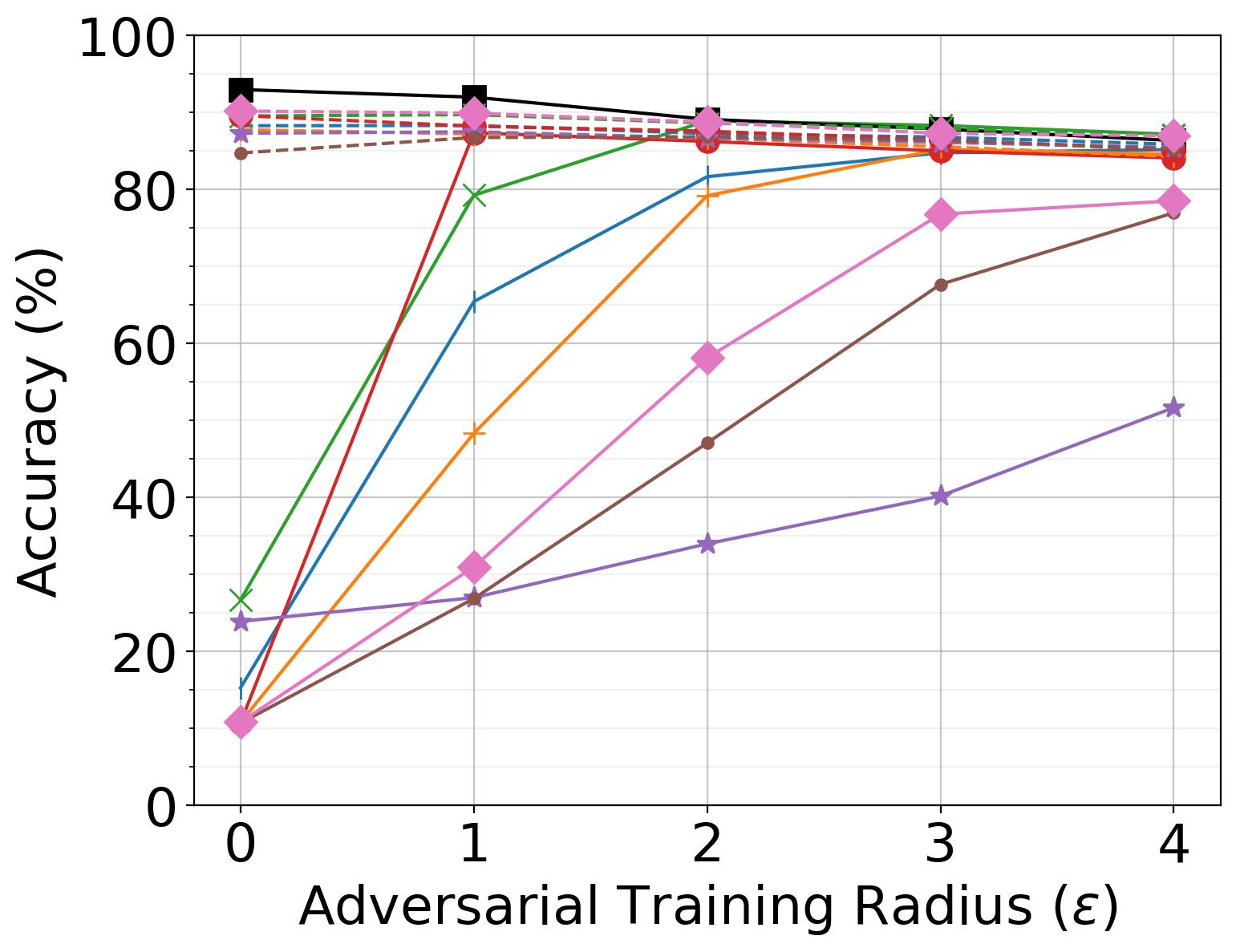}
		\caption*{CIFAR-10}
		\label{fig:at_cifar10:vgg16}
	\end{subfigure}\hspace*{0.75em}
	\begin{subfigure}{.3\textwidth}
		\centering
		\includegraphics[width=1.0\textwidth]{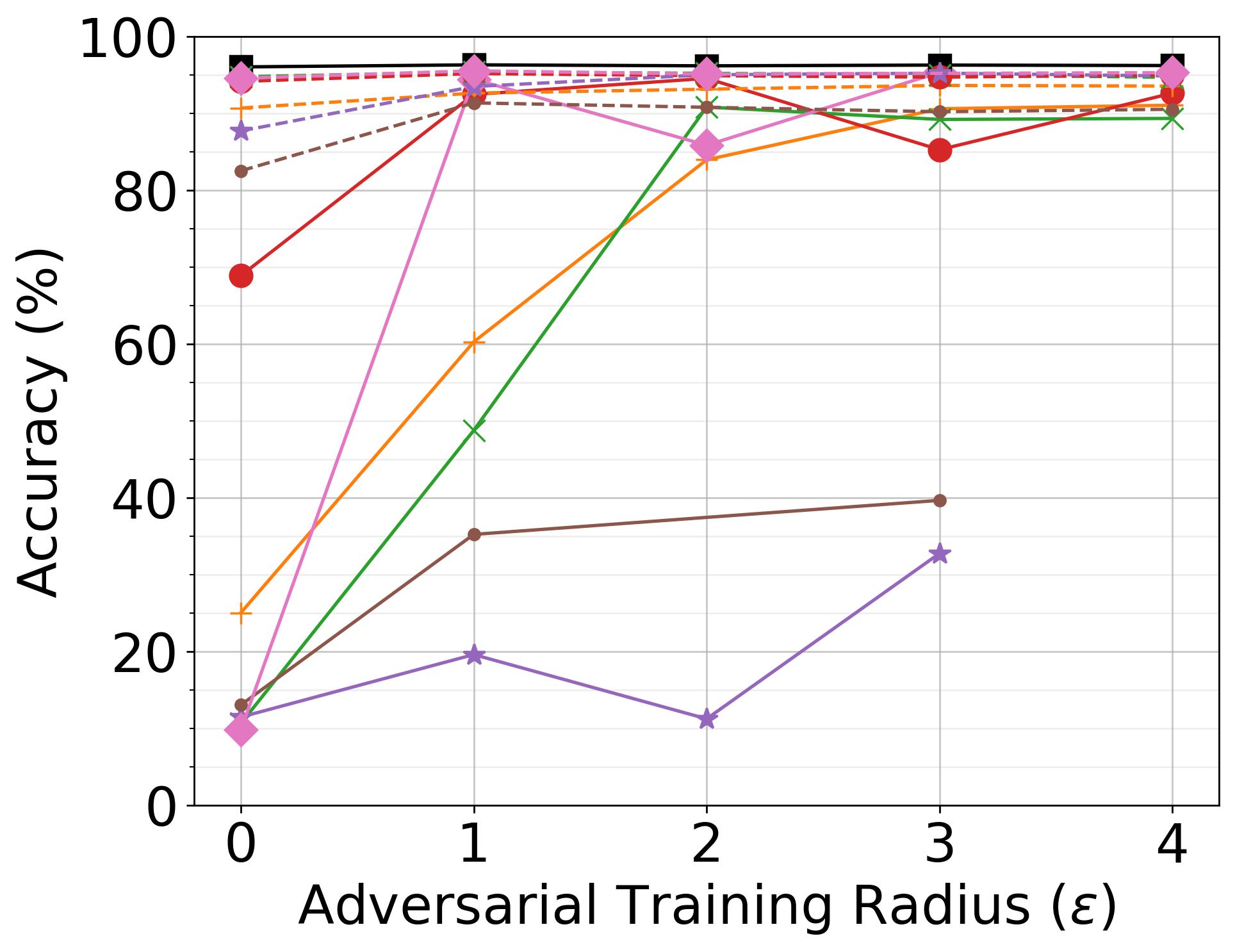}
		\caption*{SVHN}
		\label{fig:at_svhn:vgg16}
	\end{subfigure}\hspace*{0.75em}
	\begin{subfigure}{.3\textwidth}
		\centering
		\includegraphics[width=1.0\textwidth]{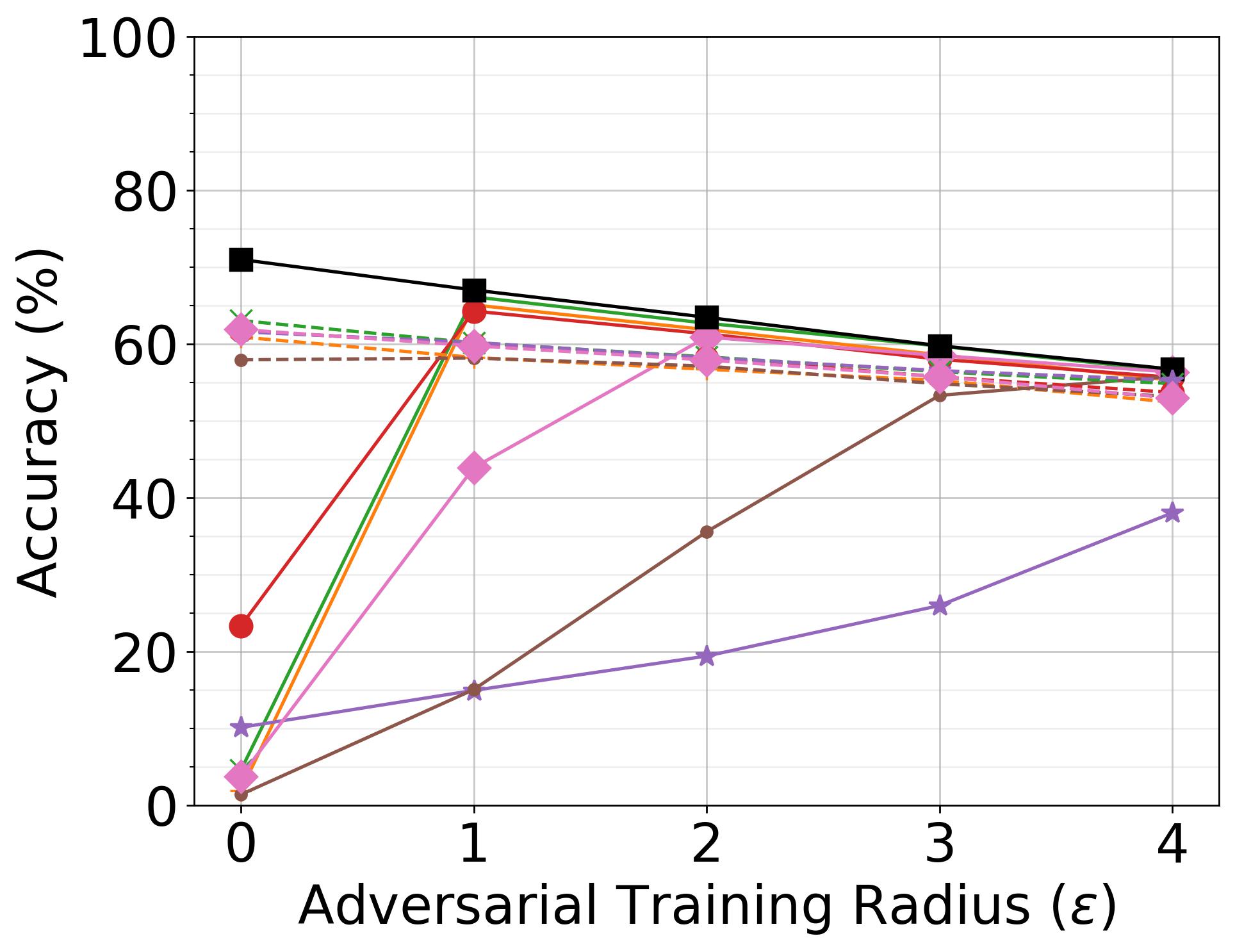}
		\caption*{CIFAR-100}
		\label{fig:at_cifar100:vgg16}
	\end{subfigure}
    \caption{VGG-16}
	\label{fig:at_vgg16}
    \end{subfigure}\\\vspace*{0.5em}
    \begin{subfigure}{1.0\textwidth}
	\begin{subfigure}{.3\textwidth}
		\centering
		\includegraphics[width=1.0\textwidth]{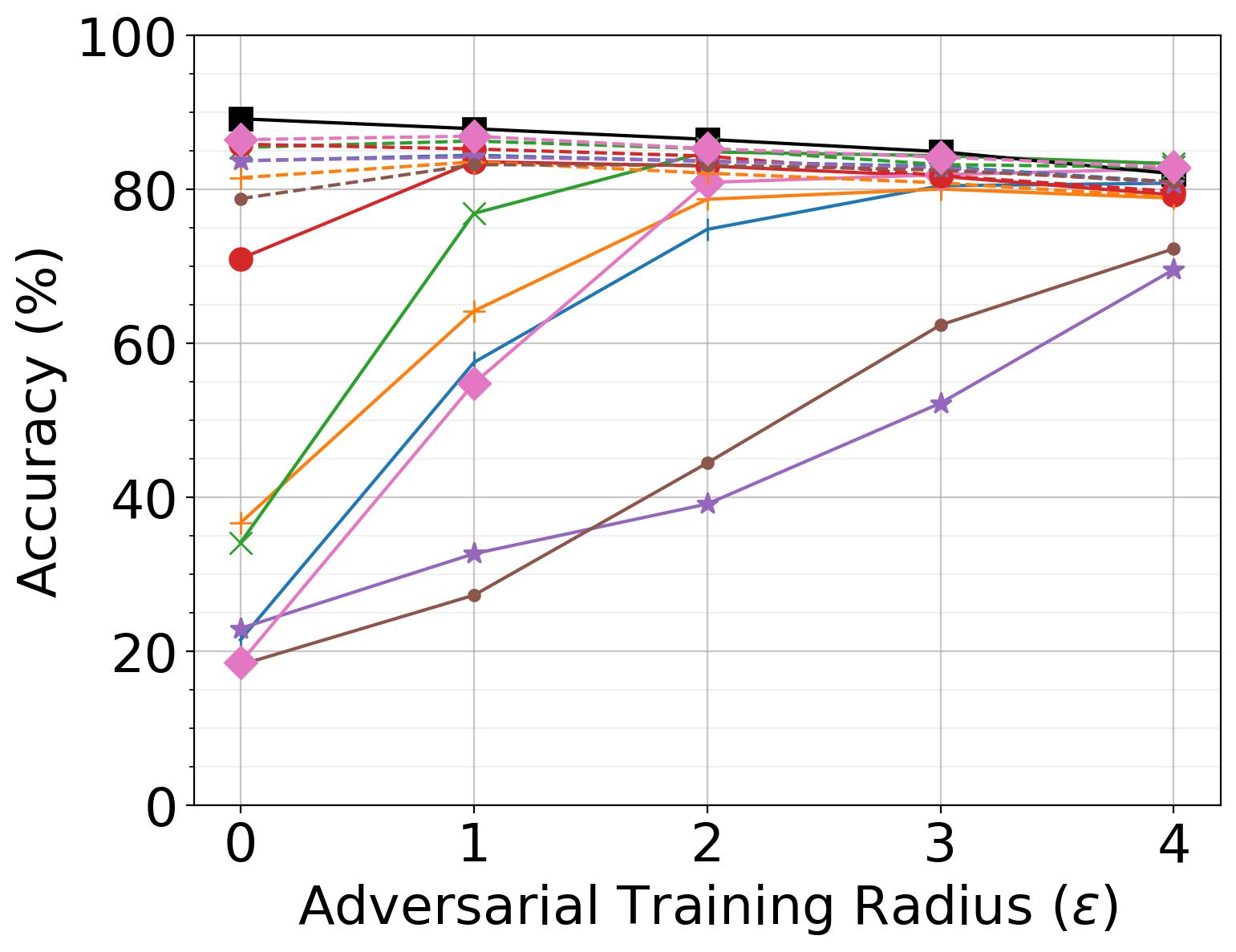}
		\caption*{CIFAR-10}
		\label{fig:at_cifar10:dn121}
	\end{subfigure}\hspace*{0.75em}
	\begin{subfigure}{.3\textwidth}
		\centering
		\includegraphics[width=1.0\textwidth]{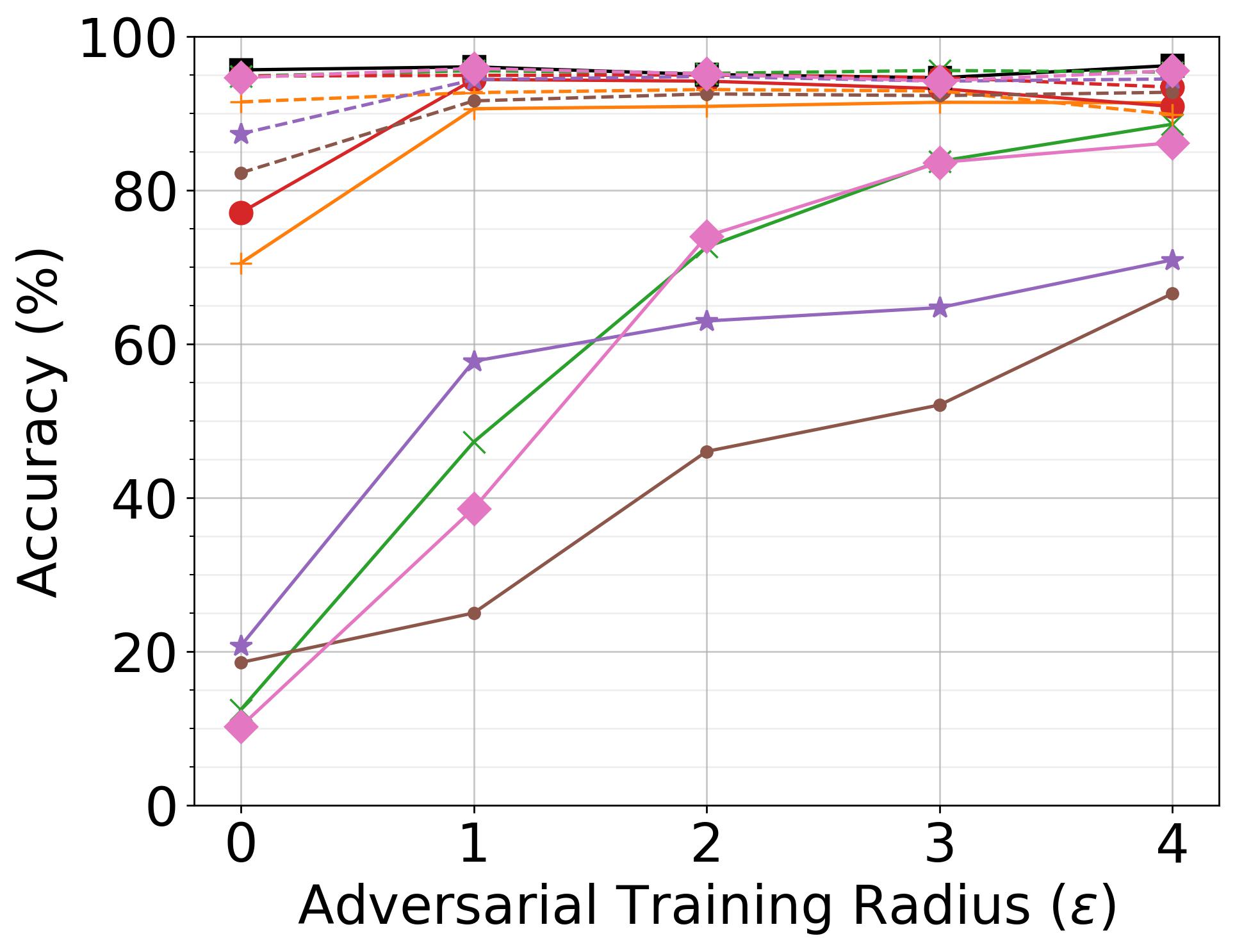}
		\caption*{SVHN}
		\label{fig:at_svhn:dn121}
	\end{subfigure}\hspace*{0.75em}
	\begin{subfigure}{.3\textwidth}
		\centering
		\includegraphics[width=1.0\textwidth]{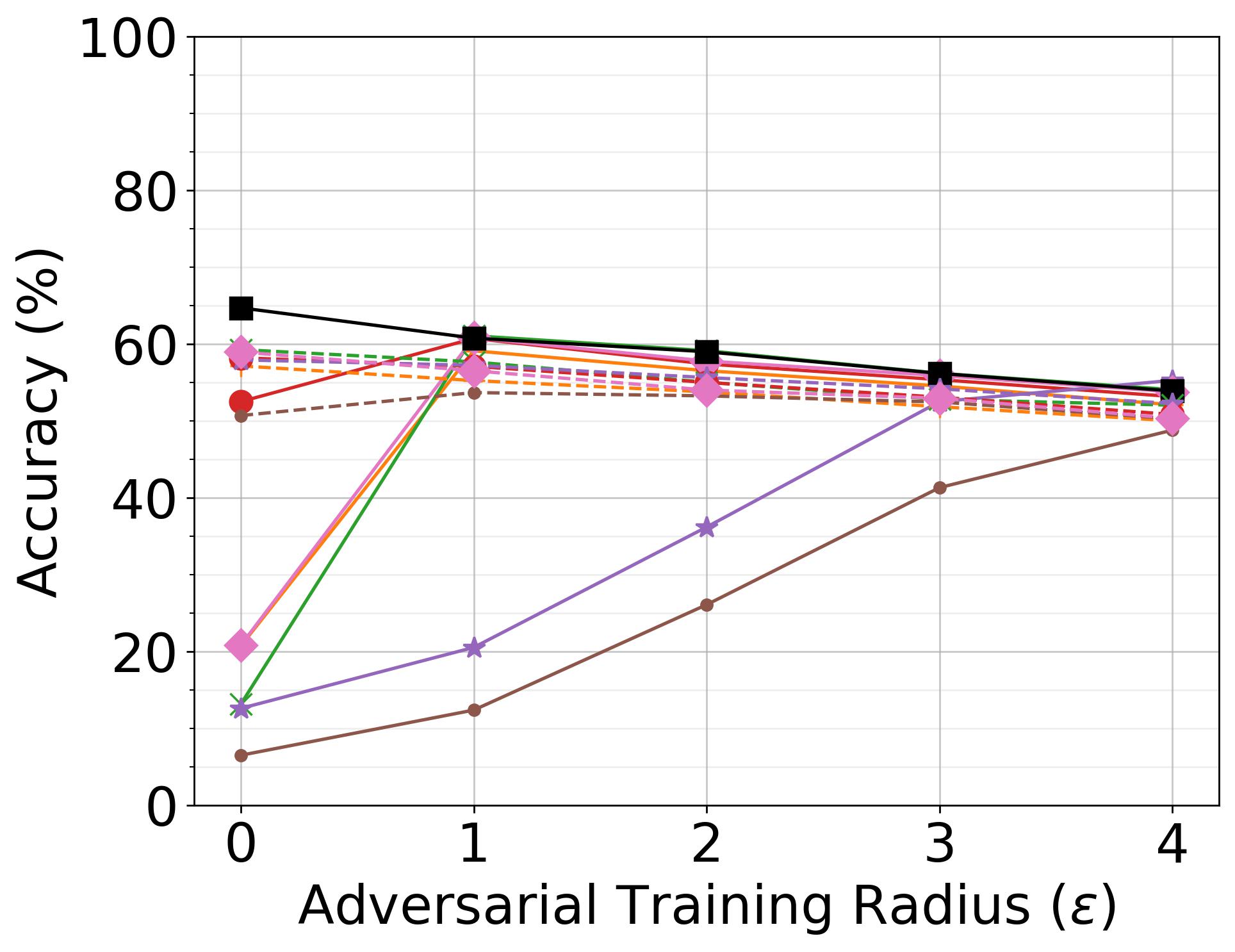}
		\caption*{CIFAR-100}
		\label{fig:at_cifar100:dn121}
	\end{subfigure}
    \caption{DN-121}
	\label{fig:at_dn121}
    \end{subfigure}\\\vspace*{0.5em}
    \begin{subfigure}{1.0\textwidth}
	\begin{subfigure}{.3\textwidth}
		\centering
		\includegraphics[width=1.0\textwidth]{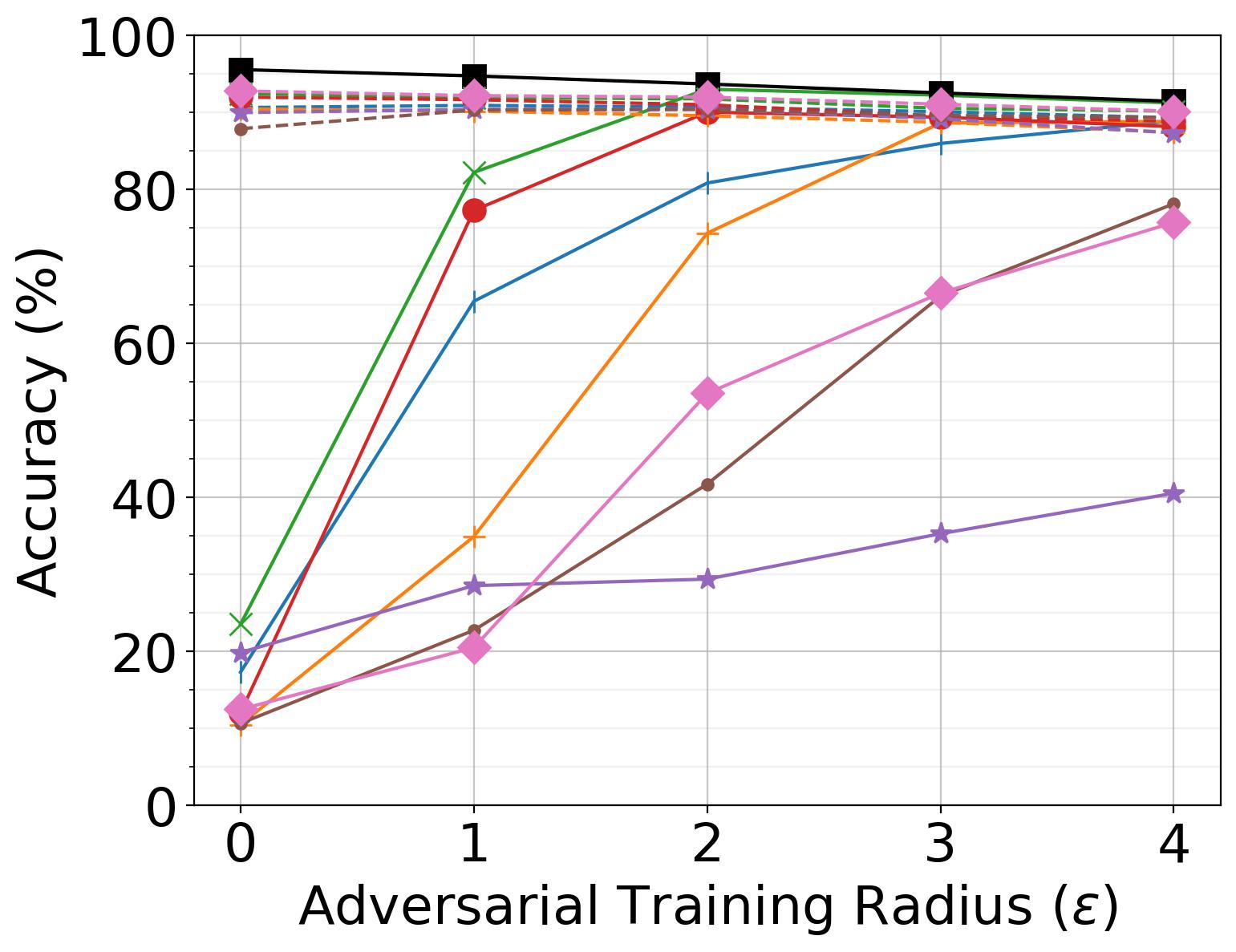}
		\caption*{CIFAR-10}
		\label{fig:at_cifar10:wrn34}
	\end{subfigure}\hspace*{0.75em}
	\begin{subfigure}{.3\textwidth}
		\centering
		\includegraphics[width=1.0\textwidth]{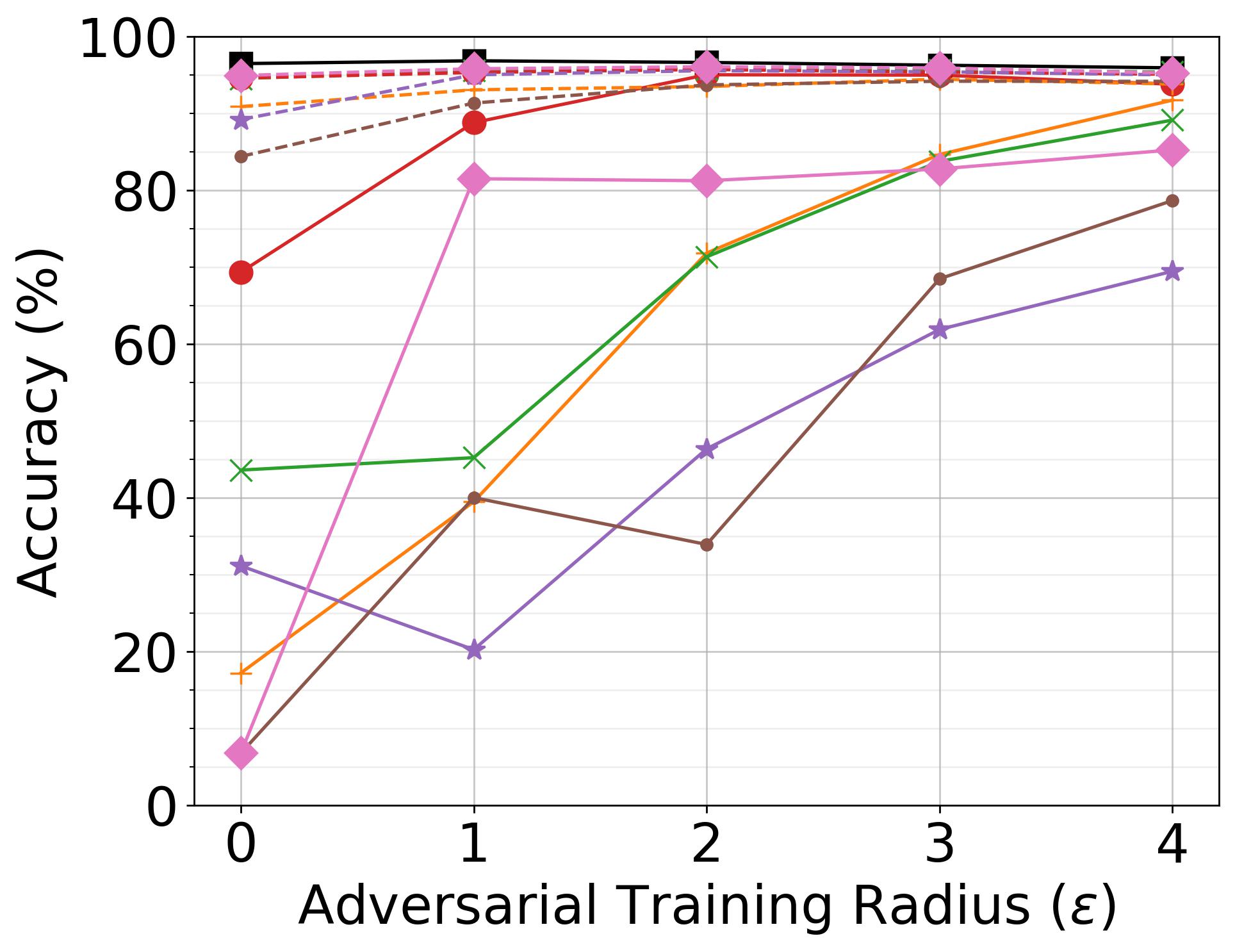}
		\caption*{SVHN}
		\label{fig:at_svhn:wrn34}
	\end{subfigure}\hspace*{0.75em}
	\begin{subfigure}{.3\textwidth}
		\centering
		\includegraphics[width=1.0\textwidth]{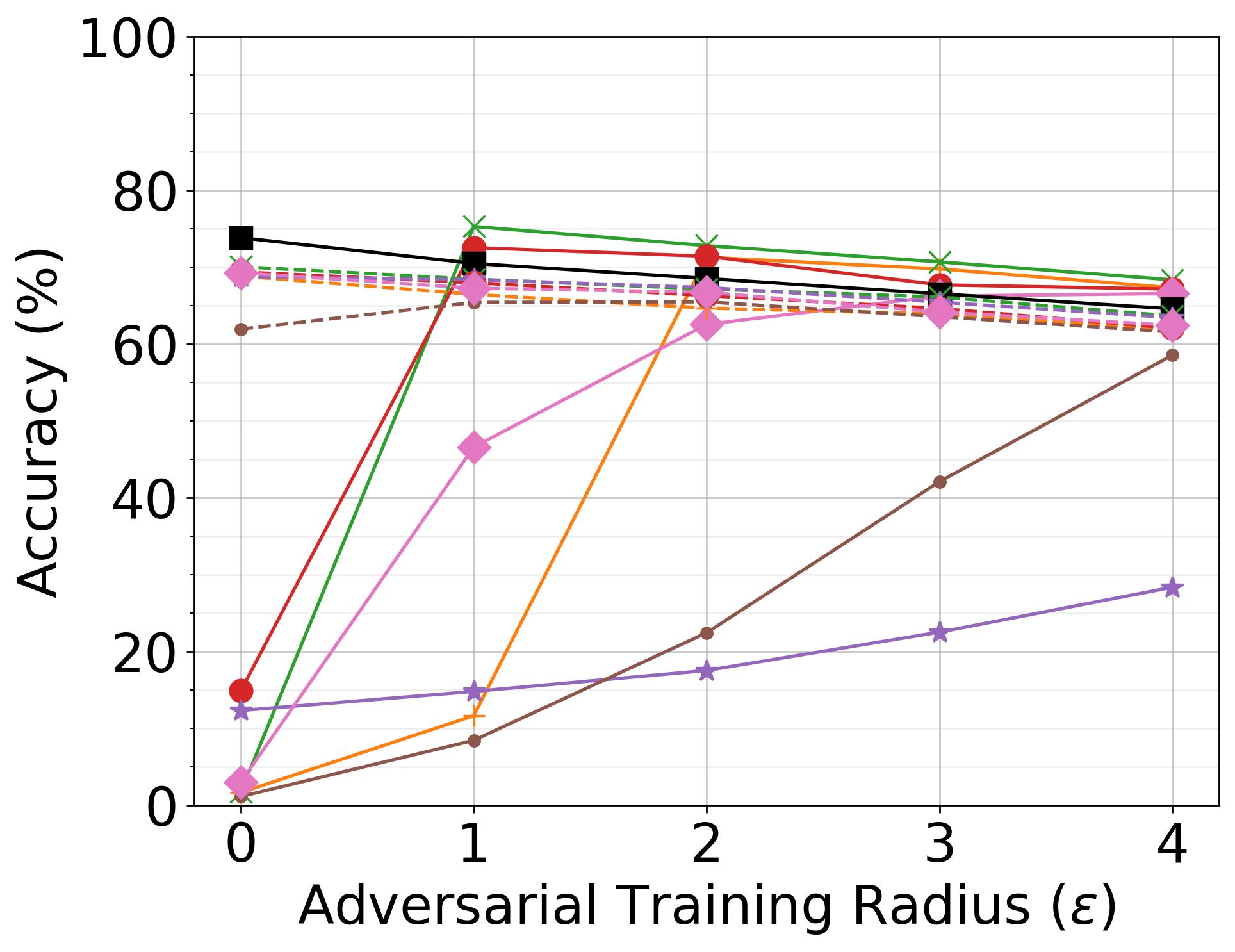}
		\caption*{CIFAR-100}
		\label{fig:at_cifar100:wrn34}
	\end{subfigure}
    \caption{WRN-34}
	\label{fig:at_wrn34}
    \end{subfigure}
    \caption{Test accuracy of CIFAR-10, SVHN, and CIFAR-100 classifiers against availability attacks using adversarial training with different perturbation radii.}
    \label{fig:at_all}
\end{figure*}

\begin{figure*}[tb!]
	\centering
	\begin{subfigure}{.375\textwidth}
		\centering
		\includegraphics[width=1.0\textwidth]{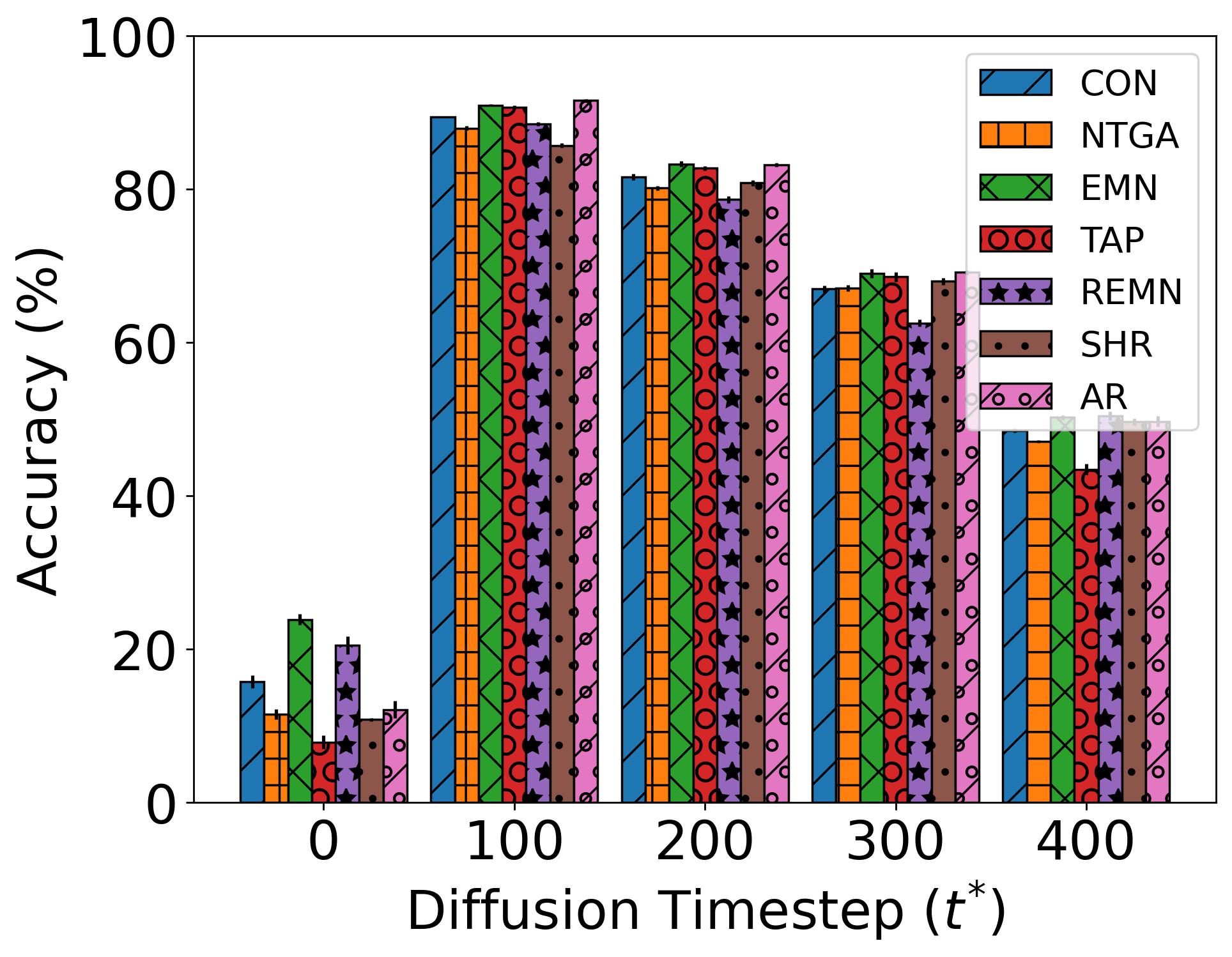}
		\caption{RN-18}
		\label{fig:timestep_ext:rn18}
	\end{subfigure}\hspace*{2em}
	\begin{subfigure}{.375\textwidth}
		\centering
		\includegraphics[width=1.0\textwidth]{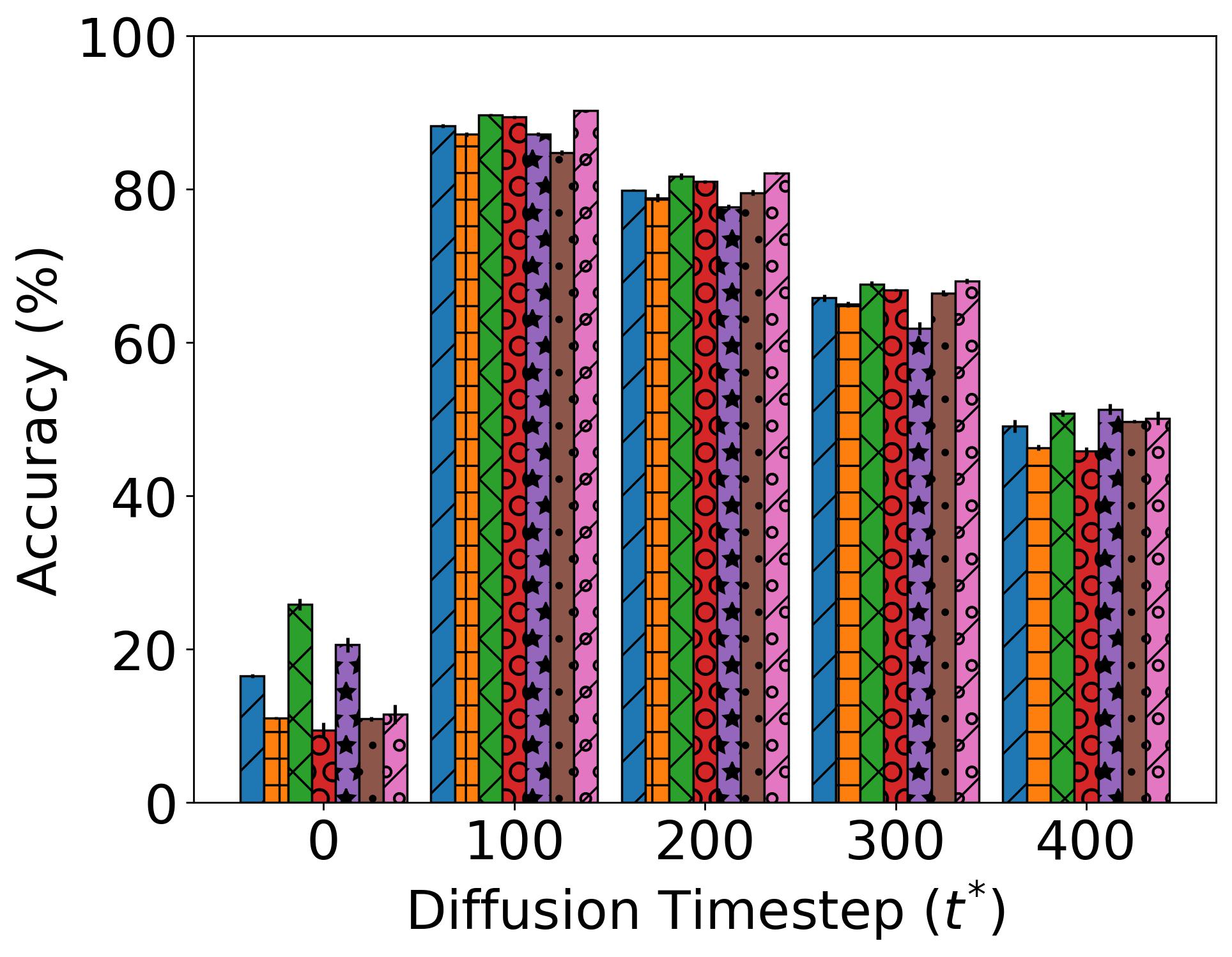}
		\caption{VGG-16}
		\label{fig:timestep_ext:vgg16}
	\end{subfigure}\\
	\begin{subfigure}{.375\textwidth}
		\centering
		\includegraphics[width=1.0\textwidth]{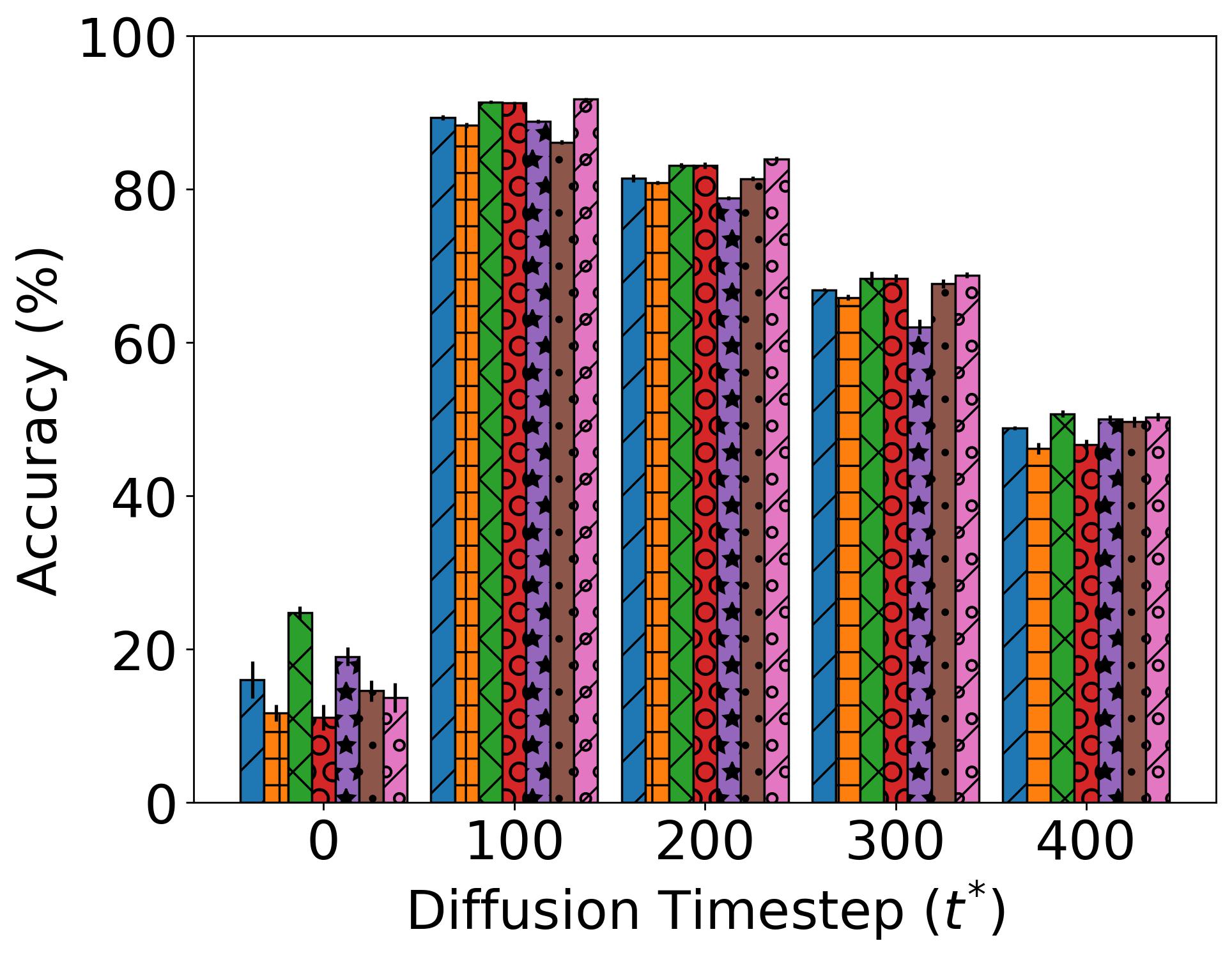}
		\caption{DN-121}
		\label{fig:timestep_ext:dn121}
	\end{subfigure}\hspace*{2em}
	\begin{subfigure}{.375\textwidth}
		\centering
		\includegraphics[width=1.0\textwidth]{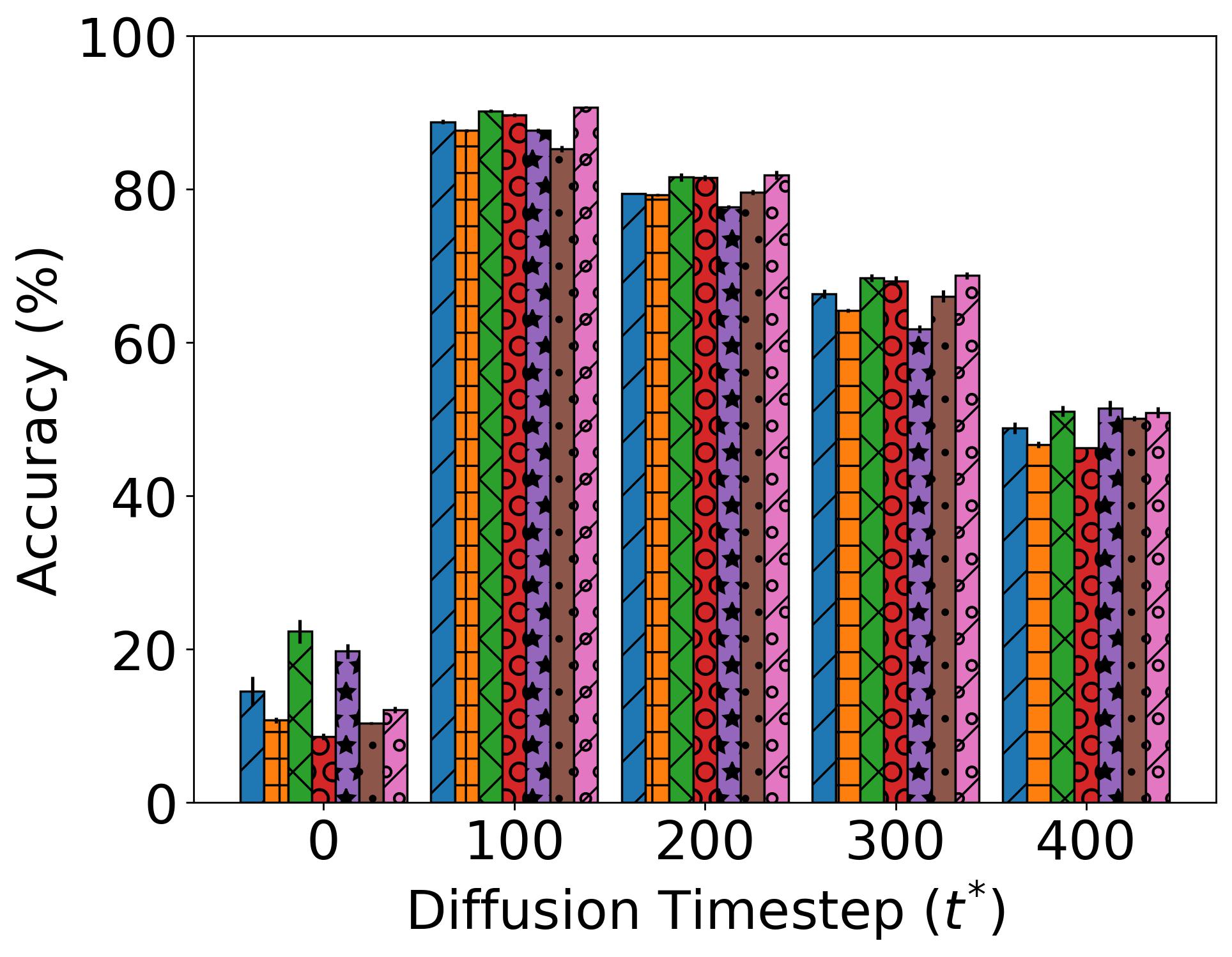}
		\caption{WRN-34}
		\label{fig:timestep_ext:wrn34}
	\end{subfigure}
	\caption{Effect of changing the forward process diffusion timestep in \textsc{Avatar} on the final test accuracy in CIFAR-10 classifiers.}
	\label{fig:timestep}
\end{figure*}

\subsection{Setting Diffusion Step $t^{*}$}
As discussed in \Cref{sec:sec:setting_t}, setting the diffusion timestep should be performed carefully.
Otherwise, either the data-protecting noise is not eliminated, or the semantic information of the image is lost.
Here, we run an ablation study over the diffusion timestep.
In particular, for our CIFAR-10 experiments, we run \textsc{Avatar} with five different timesteps from ${\{0, 100, 200, 300, 400\}}$.
Then, we measure the test accuracy of the trained neural networks over the clean test set.
As shown in~\Cref{fig:timestep}, setting $t^{*}$ too small means that the data-protecting perturbations are not removed.
In contrast, setting $t^{*}$ to a large value might remove the semantic information which in turn damages the generalizability of the trained model.
For a more thorough discussion on selecting $t^{*}$, please see~Appendix~\ref{sec:additional_experiments:selection}.

\subsection{The Effect of Diffusion Models' Training Data}\label{sec:sec:training_data}
It is well-known from the literature that diffusion models are not a mere memorization of their training data~\citep{song2021scoresde} and can further enhance the accuracy of down-stream tasks~\citep{azizi2023imagenet, wang2023better}.
To empirically eradicate the influence of training data overlap on our results, we perform the following experiment.
Apart from our results in~\Cref{tab:architecture}, we run a second set of experiments where we create disjoint subsets of training data for training diffusion models and those used as unlearnable examples.
Then, we train our in-house diffusion model and perform a similar experiment to that of \Cref{tab:architecture}, but this time with this new, non-overlapping set of data.
Finally, we measure the performance over the unseen test data.
We report the relative error rate with respect to the clean data performance in \Cref{fig:nonoverlap}.
As seen, the overlap in diffusion models' training data has no impact on \textsc{Avatar}'s final performance.
We further validate this through our real-world experiments in \Cref{sec:sec:real_world}.

\begin{figure*}[t!]
	\centering
	\begin{subfigure}{.40\textwidth}
		\centering
		\includegraphics[width=1.0\textwidth]{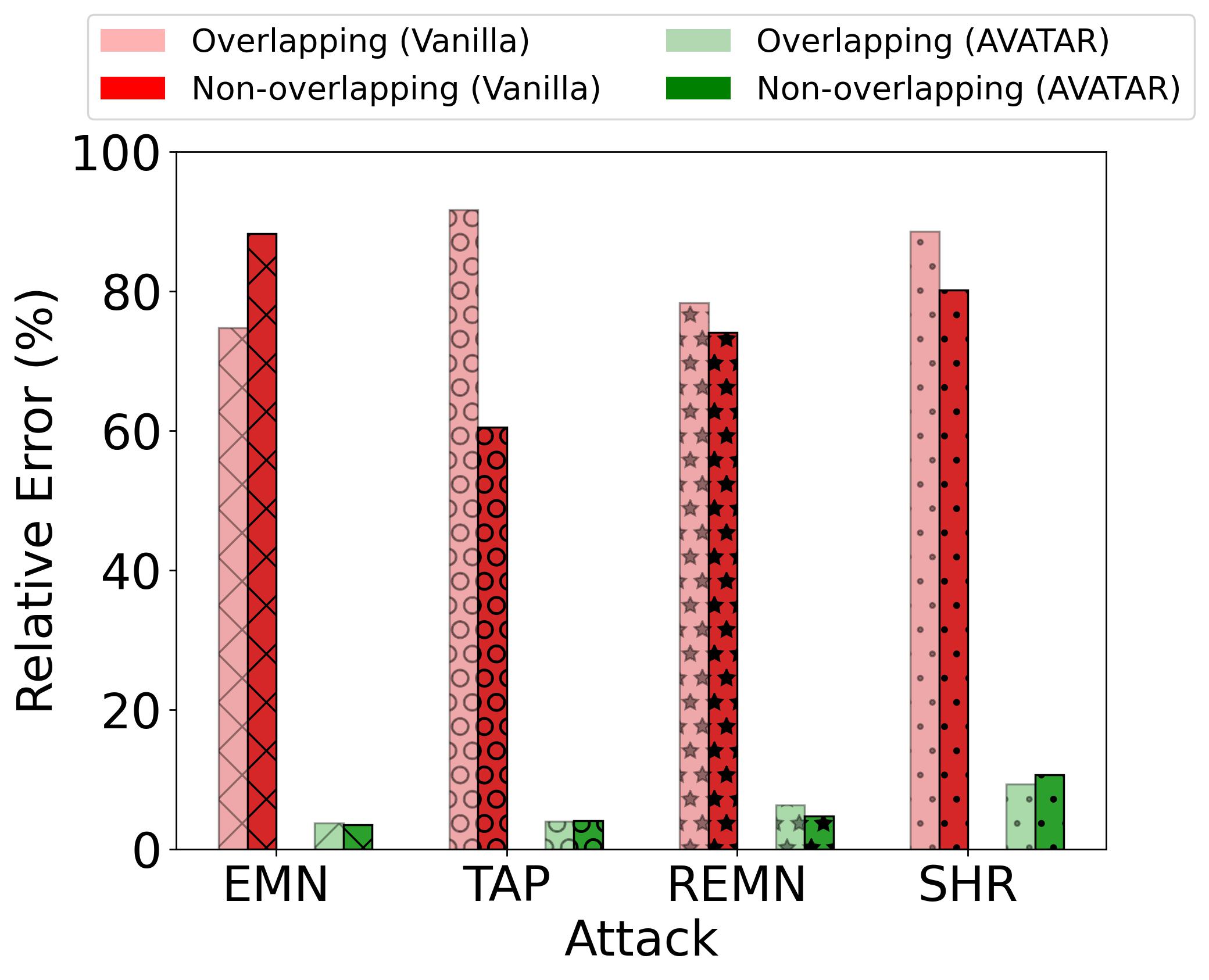}
		\caption{CIFAR-10}
		\label{fig:nonoverlap:cifar10}
	\end{subfigure}\hspace*{5em}
	\begin{subfigure}{.40\textwidth}
		\centering
		\includegraphics[width=1.0\textwidth]{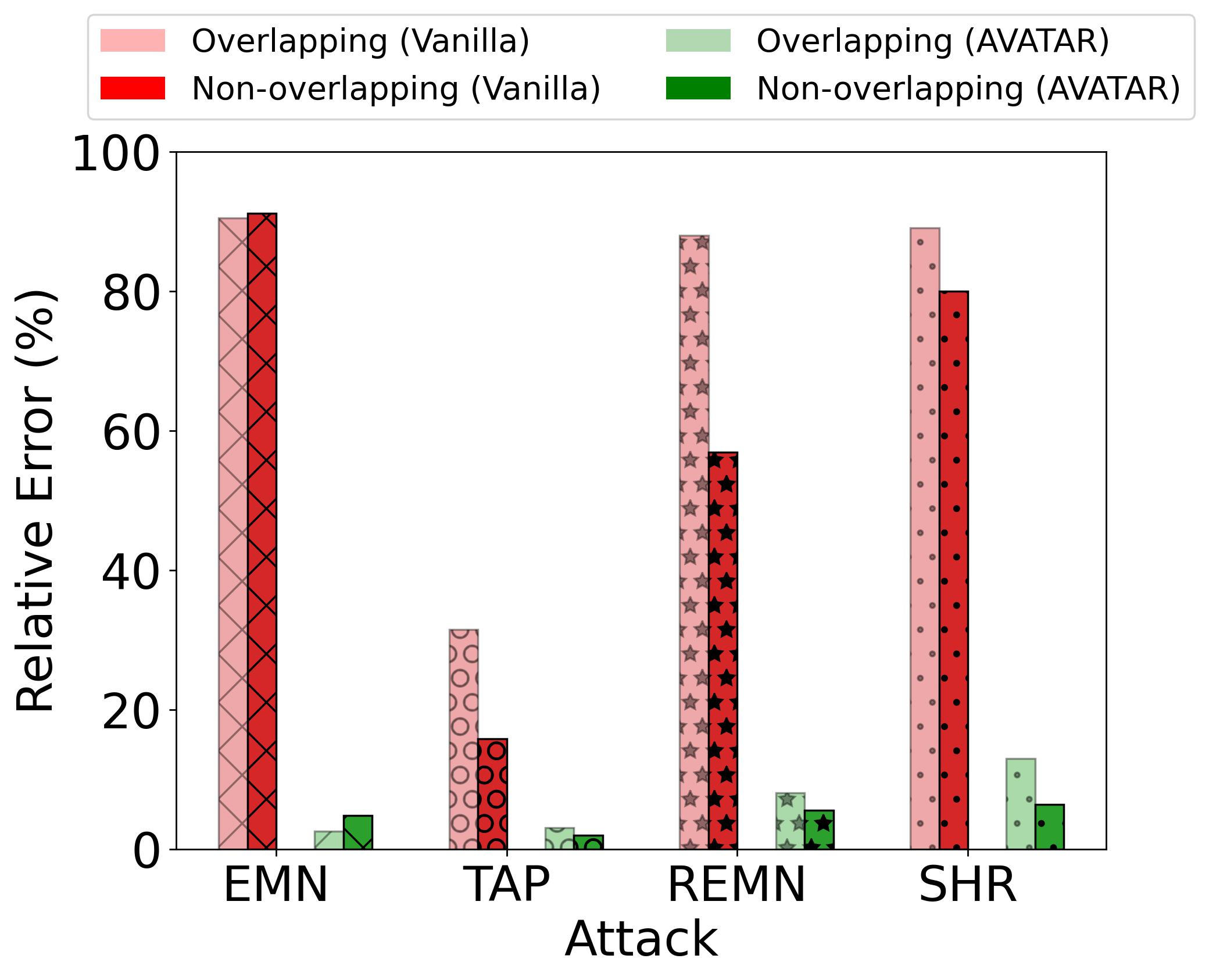}
		\caption{SVHN}
		\label{fig:nonoverlap:svhn}
	\end{subfigure}
    \caption{Relative error rate of RN-18 models trained against availability attacks on CIFAR-10 and SVHN averaged over 5 runs. \textbf{Overlapping} indicates that the diffusion model and availability attacks use the same subset as training data. \textbf{Non-overlapping} means that the diffusion model and availability attacks are trained on \textbf{disjoint} subsets of data.}
	\label{fig:nonoverlap}
\end{figure*}
\begin{table*}[t!]
	\caption{Test accuracy (\%) of RN-18 models trained over data availability attacks on the CIFAR-10 dataset. For \textsc{Avatar}, we use different pre-trained distributions over CIFAR-10, poisoned CIFAR-10~(TAP), ImageNet-10~(IN-10)~\citep{huang2021emn}, and CIFAR-100. The results are averaged over 5 runs.}
	\label{tab:dist_mismatch}
	\begin{center}
		\begin{small}
		    \setlength\tabcolsep{0.35em}
			\def\arraystretch{1.5}
			\begin{tabular}{cccccccc}
				\toprule
				\multirow{2}{*}{\textbf{Distribution}}
				&\multicolumn{7}{c}{\textbf{Data Availability Attacks}}\\
				\cmidrule(lr){2-8}
				                                              & CON & NTGA & EMN & TAP & REMN  & SHR & AR\\
				\midrule
				Vanilla           & $15.75 \pm 0.82$ & $11.49 \pm 0.69$ & $24.85 \pm 0.71$ & $7.86 \pm 0.90$ & $20.50 \pm 1.16$ & $10.82 \pm 0.22$  & $12.09 \pm 1.12$\\
                CIFAR-10~(TAP)    & $61.70 \pm 2.02$ & $75.62 \pm 3.75$ & $64.03 \pm 0.98$ & $35.09 \pm 2.28$ & $60.16 \pm 1.44$ & $74.96 \pm 2.82$ & $60.36 \pm 2.29$\\
                IN-10             & $80.98 \pm 0.06$ & $79.42 \pm 0.25$ & $83.78 \pm 0.39$ & $82.71 \pm 0.24$ & $82.83 \pm 0.28$ & $75.91 \pm 0.06$ & $84.88 \pm 0.19$\\
                CIFAR-100         & $84.85 \pm 0.49$ & $83.07 \pm 0.33$ & $87.81 \pm 0.14$ & $86.55 \pm 0.26$ & $85.84 \pm 0.19$ & $79.52 \pm 0.22$ & $88.59 \pm 0.15$\\
				CIFAR-10          & $89.43 \pm 0.09$ & $87.95 \pm 0.28$ & $90.95 \pm 0.10$ & $90.71 \pm 0.19$ & $88.49 \pm 0.24$ & $85.69 \pm 0.27$ & $91.57 \pm 0.18$\\
			    \bottomrule
			\end{tabular}
		\end{small}
	\end{center}
\end{table*}

\subsection{Distribution Mismatch}\label{sec:sec:mismatch}
To go even further, we show that \textsc{Avatar} is even resilient to a distribution mismatch between the diffusion model and the training data.
In particular, we train three diffusion models over the protected CIFAR-10 dataset with TAP~\citep{fowl2021tap}, IN-10 which contains 10 classes of ImageNet that are most similar to CIFAR-10 dataset~\citep{huang2021emn}~(see \Cref{tab:imagenet10} for more details), and CIFAR-100.
Then, we use these surrogate distributions to sanitize protected CIFAR-10 data and train a neural network over the denoised data.
We report our results in \Cref{tab:dist_mismatch}.
Surprisingly, our approach can tolerate the distribution mismatch to some extent.
As the diffusion model density gets closer to the true training data, the performance gap is gradually closed.
Interestingly, even using a diffusion model that is trained over protected data can be beneficial in removing the effects of availability attacks.
Note that according to our threat model discussed in~\Cref{fig:threat_model}, this case is too extreme, meaning that the data protector needs to add a perturbation to all the data on the web which is almost impossible.
Interestingly, our method using the sub-optimal CIFAR-100 distribution is still performing better than grayscale and JPEG compression techniques of \citet{liu2023image}.

These results motivates us to run~\textsc{Avatar} in a real-world case.
In particular, we employ the off-the-shelf diffusion model, DDPM-IP~\citep{ning2023ddpmip}, that is trained over the $32 \times 32$ version of the ImageNet dataset in \textsc{Avatar}.
Then, we re-run our experiments of~\Cref{tab:architecture} on CIFAR-10, CIFAR-100, and SVHN using this diffusion model.
As this DDPM-IP~\citep{ning2023ddpmip} uses a cosine schedule~\citep{dhariwal2021diffusion}, we need to adjust the value of $t^{*}$ to reflect this change.
As we discuss in~Appendix~\ref{sec:additional_experiments:selection}, we set $t^{*}=200$ to have an equivalent performance to the linear schedule that was used in our earlier experiments.

Our results are shown in \Cref{tab:dist_mismatch_new}.
As seen, \textsc{Avatar} is resilient to the choice of the diffusion model.
Even though there is a distribution mismatch between our test datasets and ImageNet-32$\times$32, our results are on par with the use of the matching data distribution.
These results indicate the real-world value of \textsc{Avatar} which can serve as a strong baseline against availability attacks.

\begin{table*}[tb!]
	\caption{Test accuracy (\%) of RN-18 models trained over data availability attacks on CIFAR-10, CIFAR-100, SVHN with our denoising approach using the matching distribution and ImageNet-32$\times$32. The mean and standard deviation are computed over 5 seeds.}
	\label{tab:dist_mismatch_new}
	\begin{center}
		\begin{small}
		    \setlength\tabcolsep{0.45em}
			\def\arraystretch{1.65}
			\begin{tabular}{lccccccccc}
				\toprule
                \multirow{2}{*}{\rotatebox[origin=c]{90}{{\textbf{Data}}}}
				&\multirow{2}{*}{\textbf{Distribution}}
                &\multirow{2}{*}{\textbf{Clean}}
				&\multicolumn{6}{c}{\textbf{Data Availability Attacks}}\\
				\cmidrule(lr){4-9}
				&&                                            & NTGA& EMN & TAP & REMN  & SHR & AR\\
				\midrule
                \multirow{2}{*}{\rotatebox[origin=c]{90}{\scriptsize CIFAR-10}}
				& CIFAR-10  &\multirow{2}{*}{$94.50 \pm 0.09$} & $87.95 \pm 0.28$   & $90.95 \pm 0.10$ & $90.71 \pm 0.19$ & $88.49 \pm 0.24$ & $85.69 \pm 0.27$ & $91.57 \pm 0.18$\\
				& ImageNet-32$\times$32                        && $86.41 \pm 0.21$  & $90.17 \pm 0.15$ & $89.02 \pm 0.15$ & $88.26 \pm 0.24$ & $82.97 \pm 0.24$ & $90.61 \pm 0.18$\\
                \midrule
                \multirow{2}{*}{\rotatebox[origin=c]{90}{\scriptsize SVHN}}
				& SVHN      &\multirow{2}{*}{$96.29 \pm 0.12$} & $89.84 \pm 0.32$   & $93.84 \pm 0.12$ & $93.35 \pm 0.10$  & $88.51 \pm 0.23$ & $83.82 \pm 0.39$ & $94.13 \pm 0.17$\\
				& ImageNet-32$\times$32                        && $91.32 \pm 0.17$  & $94.82 \pm 0.10$ & $95.01 \pm 0.21$  & $91.00 \pm 0.27$ & $83.12 \pm 0.30$ & $94.29 \pm 0.22$\\
                \midrule
                \multirow{2}{*}{\rotatebox[origin=c]{90}{\scriptsize CIFAR-100}}
				& CIFAR-100 &\multirow{2}{*}{$75.01 \pm 0.41$} & $63.98 \pm 0.55$   & $65.73 \pm 0.36$ & $64.99 \pm 0.10$  & $64.88 \pm 0.08$  & $58.52 \pm 0.46$ & $64.54 \pm 0.23$\\
				&  ImageNet-32$\times$32                       && $65.22 \pm 0.55$  & $67.09 \pm 0.18$ & $66.52 \pm 0.25$  & $66.52 \pm 0.15$  & $58.32 \pm 0.56$ & $66.44 \pm 0.17$\\
			    \bottomrule
			\end{tabular}
		\end{small}
	\end{center}
\end{table*}

\begin{figure*}[tb!]
	\centering
	\includegraphics[width=0.52\textwidth]{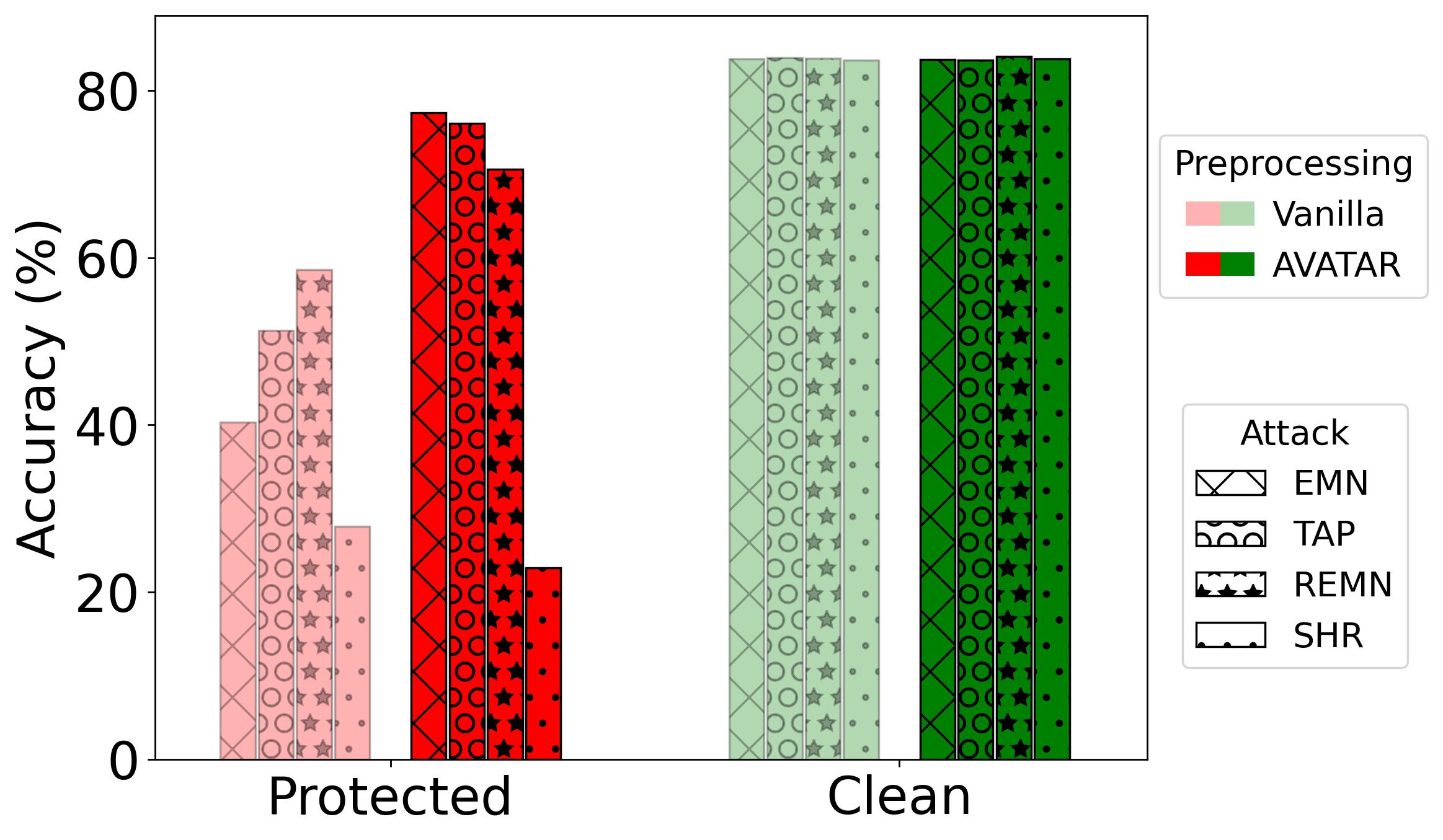}
	\caption{Test accuracy for protected vs.~clean identities in WebFace~\citep{yi2014} facial recognition. The protected users protect their images using data-protecting perturbations. Our approach uses a diffusion model trained over the CelebA~\citep{liu2015deep} dataset. For all the stealthy data-protecting perturbations our approach manages to recover the performance over protected data. The only exception is SHR, which according to~\Cref{fig:WebFace_Samples}, leaves a noticeable trace over the image, rendering them not useful anymore.}
	\label{fig:webface}
\end{figure*}

\subsection{Real-world Example I: Face Recognition}\label{sec:sec:real_world}
In \Cref{sec:intro}, we discussed in detail that the threat model of existing availability attacks is fragile and a malicious adversary might still exploit the personal data.
This means that possibly no \textit{imperceptible} adversary can protect the image data from being maliciously used.
To show this, we discussed a real-world example in \Cref{sec:experiments} following a similar experiment from \citet{huang2021emn}.
In particular, we create a set of clean and protected identities in the WebFace~\citep{yi2014} dataset by randomly selecting 50 identities from this dataset.
As a result, the remaining 10522 identities constitute our clean data.
For all of the identities, we randomly split the data so that 80\% of that data is allocated to a training set and the rest is the test set.
We assume that the protected identities would add data-protecting perturbations to their images before sharing them.
To this end, we use class-wise EMN~\citep{huang2021emn}, TAP~\citep{fowl2021tap}, REMN~\citep{fu2022remn}, and SHR~\citep{yu2022shr} with a perturbation radius of $\norm{\boldsymbol{\delta}}_{\infty} \leq 16/255$.
For perturbation generation using the first three attacks, we follow the settings of \citet{huang2021emn}.
Specifically, we select 100 random identities from the CelebA~\citep{liu2015deep} dataset and create an auxiliary dataset consisting of these 100 identities and the 50 protected WebFace~\citep{yi2014} identities.
Then, using these 150 identities we generate data protecting perturbations against a neural network with 150 classes.
For SHR, however, we generate the data for all the 10572 WebFace identities and select the relevant data for protecting our above-mentioned 50 identities.
Once we have the protected data, we train an InceptionResNet~\citep{szegedy2017inception} facial recognition over the training set with or without our approach and evaluate the models over the test set.
In our case, we assume that the malicious entity has access to a pre-trained diffusion model over CelebA~\cite{liu2015deep} faces\footnote{For this experiment, we use a pre-trained DDPM model over CelebA-HQ: \url{https://github.com/ermongroup/SDEdit}.}, and can run \textsc{Avatar} over the protected data that it acquires from crawling the web.
Since the WebFace photos are of size ${112 \times 112}$ but the diffusion model generates ${256 \times 256}$ images, we use bi-linear up- and down-sampling to connect the two.
Like the CIFAR-10 experiments, here we also denoise the data with timestep set to 100.
Samples of the WebFace dataset along with the protected data are shown in~\Cref{fig:WebFace_Samples}.
To evaluate the performance of our method, we test the models over the clean test set and record the recognition accuracy for both the protected and clean identities.\footnote{Running the identity overlap removal of \citet{wang2018additive}, we found that only 8 out of 50 protected identities had overlap between CelebA-HQ and WebFace. After removing these identities, we saw no major drop in the final performance of \textsc{Avatar}.}

As shown in~\Cref{fig:webface}, \textsc{Avatar} can recover the recognition accuracy over protected identities in all cases except the SHR~\citep{yu2022shr} perturbations.
The reason behind this might be two-fold.
First, we are using a sub-optimal diffusion model as both the domain and, more importantly, size of the images have a mismatch.
Second, looking at~\Cref{fig:WebFace_Samples}, we see that while the SHR perturbations can protect the data, they trade the stealthiness of the original data due to their large patches.
As such, the images would lose their utility.
Now, the question is: 
\begin{center}
    \textit{Can we protect the data using stealthy patterns without losing the data utility?}
\end{center}
Interestingly, our theoretical result in~\Cref{thm:convergence} says that this might not be possible.
According to \Cref{thm:convergence}, if the data curator wants to makes the denoising process harder, they need to increase the data-protecting perturbation.
This increase is naturally at odds with the data utility, since by adding more powerful perturbations we lose the data utility.

\section{Conclusion}\label{sec:conclusion}
In this paper, we introduced a countermeasure against data protection algorithms that use availability attacks.
In particular, we show that by adding a controlled amount of Gaussian noise to the images and subsequently denoising them one can eliminate data-protecting perturbations.
To this end, we use the forward and reverse diffusion processes of pre-trained models.
We theoretically analyze our approach and show that the amount of Gaussian noise required to defuse the data-protecting perturbations is directly related to their norm.
We conduct extensive experiments over various availability attacks.
Our experiments demonstrate the superiority of our approach compared to adversarial training, setting a new SOTA defense against availability attacks.
\textsc{Avatar} demonstrates brittleness of availability attacks and calls for more research to protect personal data.
Future work involves investigating the applicability of \textsc{Avatar} to other models such as text-to-image generative models~\citep{shan2023glaze} and its relationship with techniques such as randomized smoothing~\citep{cohen2019randomized}.

\section*{Acknowledgments}
This research was undertaken using the LIEF HPC-GPGPU Facility hosted at the University of Melbourne. 
This Facility was established with the assistance of LIEF Grant LE170100200. 
Sarah Erfani is in part supported by Australian Research Council~(ARC) Discovery Early Career Researcher Award~(DECRA) DE220100680. 
Moreover, this research was partially supported by the ARC Centre of Excellence for Automated Decision-Making and Society~(CE200100005), and funded partially by the Australian Government through the Australian Research Council.

\bibliographystyle{ieee_fullname_natbib}
\bibliography{references}

\clearpage
\newpage
\onecolumn
\appendices
\begin{center}
\Large{The Devil's Advocate:\\Shattering the Illusion of Unexploitable Data using Diffusion Models}
\end{center}
\section{Proofs}\label{sec:proofs}
Here we provide our proof for \Cref{thm:convergence}.
First, we provide the theoretical results that would be used in our proof.
Then, we re-state \Cref{thm:convergence} and provide its detailed proof.
Our proofs heavily borrow from the contraction properties of stochastic difference equations~\citep{pham2008analysis, pham2009concentration, chung2022come}.

\begin{theorem}[Discrete stochastic contraction~\citep{pham2008analysis, chung2022come}]\label{thm:contraction}
    Let
    \begin{equation}\label{eq:sde}
        \boldsymbol{x}_{t-1}=\mathbf{h}(\boldsymbol{x}_t, t)+\sigma(\boldsymbol{x}_t, t) \boldsymbol{\epsilon}_t,
    \end{equation}
    denote a stochastic difference equation where:
    \begin{enumerate}[(a)]\setlength\itemsep{-0.5pt}
        \item $\mathbf{h}: \mathbb{R}^{d} \times \mathbb{N} \rightarrow \mathbb{R}^{d}$ is a contraction mapping, i.e., for every $t \in \mathbb{N}$ there exists a $\lambda_{t} \in [0, 1)$ such that
        \begin{equation}\label{eq:contraction}
            \norm{\mathbf{h}(\boldsymbol{x}, t) - \mathbf{h}(\boldsymbol{y}, t)} \leq \lambda_{t} \norm{\boldsymbol{x} - \boldsymbol{y}} \quad \forall~\boldsymbol{x}, \boldsymbol{y} \in \mathbb{R}^{d},
        \end{equation}
        \item $\sigma: \mathbb{R}^{d} \times \mathbb{N} \rightarrow \mathbb{R}$ is a function such that for every $t \in \mathbb{N}$ and $\boldsymbol{x} \in \mathbb{R}^{d}$
        \begin{equation}\label{eq:trace}
            \mathrm{Tr}\big(\sigma(\boldsymbol{x}, t)\mathbf{I}\sigma(\boldsymbol{x}, t)\big) \leq C_{t},
        \end{equation}
        \item and $\boldsymbol{\epsilon}_t \sim \mathcal{N}(\mathbf{0}, \mathbf{I})$.
    \end{enumerate}
    Then, for two sample trajectories $\boldsymbol{x}_{t-1}$ and $\bar{\boldsymbol{x}}_{t-1}$ that satisfy \Cref{eq:sde} we have:
    \begin{equation}\label{eq:sde_contraction}
        \mathbb{E}\left[\norm{\boldsymbol{x}_{t-1}-\bar{\boldsymbol{x}}_{t-1}}^2\right] \leq \lambda_{t}^{2}~\mathbb{E}\left[\norm{\boldsymbol{x}_{t}-\bar{\boldsymbol{x}}_{t}}^2\right] + 2 C_{t}.
    \end{equation}
\end{theorem}

Using \Cref{thm:contraction} and \Cref{eq:reverse_process} we can get the following result~\citep{chung2022come}.

\begin{corollary}\label{cor:ddpm_contraction}
    The reverse diffusion process of DDPMs are contracting stochastic difference equations.
\end{corollary}
\begin{proof}
    Our proof closely follows that of \citet{chung2022come}.
    Specifically, we need to show that for the reverse diffusion process given in \Cref{eq:reverse_process}, the conditions of \Cref{eq:contraction,eq:trace} hold.
    To show this, note that if we set:
    $${\mathbf{h}(\boldsymbol{x}_t, t) = \frac{1}{\sqrt{1-\beta_{t}}}\left(\boldsymbol{x}_{t}+\beta_{t} \mathbf{s}_{\boldsymbol{\phi}}(\boldsymbol{x}_{t}, t)\right)}$$
    and
    $$\sigma(\boldsymbol{x}_t, t) = \sqrt{\beta_{t}}$$
    then \Cref{eq:reverse_process,eq:sde} coincide.
    Using Lemma A.1.~from \citet{chung2022come}, one can show that for
    \begin{equation}\label{eq:lambda_t}
        \lambda_{t} = \sqrt{1 - \beta_{t}} \frac{1-\alpha_{t-1}}{1-\alpha_{t}}
    \end{equation}
    and
    \begin{equation}\label{eq:C_t}
        C_{t} = d \beta_{t}
    \end{equation}
    the conditions of \Cref{eq:contraction,eq:trace} are satisfied.
    As such, for two reverse sample trajectories $\boldsymbol{x}_{t-1}$ and $\bar{\boldsymbol{x}}_{t-1}$ that satisfy the reverse diffusion process of \Cref{eq:reverse_process}, \Cref{eq:sde_contraction} holds. 
\end{proof}
Next, we present two lemmas that are going to be used in our proof of \Cref{thm:convergence}.

\begin{lemma}[\citep{chung2022come}]\label{lem:exp}
    For $\lambda_t$'s given in \Cref{eq:lambda_t} the following holds:
    \begin{equation}\label{eq:exp}
        \prod_{s=1}^{t^{*}} \lambda_{s}^{2} \leq \exp({-\frac{t^{*}\beta_{t^{*}}}{2}}).
    \end{equation}
\end{lemma}
\begin{proof}
    See Lemma~C.1.~in~\citep{chung2022come}.
\end{proof}
\begin{lemma}\label{lem:mean}
    For two random vectors $\mathbf{x}$ and $\mathbf{y}$ we have:
    \begin{equation}\label{eq:mean}
        \mathbb{E}\left[\norm{\mathbf{x} + \mathbf{y}}^2\right] \leq 2~\mathbb{E}\left[\norm{\mathbf{x}}^2\right] + 2~\mathbb{E}\left[\norm{\mathbf{y}}^2\right].
    \end{equation}
\end{lemma}
\begin{proof}
    We know that:
    \begin{align}\nonumber
        \mathbb{E}\left[\norm{\mathbf{x} + \mathbf{y}}^2\right] &= \mathbb{E}\left[\norm{\mathbf{x}}^2\right] + \mathbb{E}\left[\norm{\mathbf{y}}^2\right] + 2~\mathbb{E}\left[\mathbf{x}^\top\mathbf{y}\right]\\\nonumber
        & \leq 2~\mathbb{E}\left[\norm{\mathbf{x}}^2\right] + 2~\mathbb{E}\left[\norm{\mathbf{y}}^2\right],
    \end{align}
    where the last inequality follows from the fact that $\mathbb{E}\left[\norm{\mathbf{x} - \mathbf{y}}^2\right] \geq 0$.
\end{proof}

We are now ready to prove our theoretical result.

\begin{customthm}{1}[restated]\label{thm:convergence:rep}
    Let $\boldsymbol{x} \in \mathbb{R}^{d}$ denote a clean image and $\tilde{\boldsymbol{x}} = \boldsymbol{x} + \boldsymbol{\delta}$ its protected version, where ${\boldsymbol{\delta}}$ denotes any arbitrary data protection perturbation.
    Also, let $\bar{\boldsymbol{x}}_{0}$ be the sanitized image using the \textsc{Avatar} denoising process given in \Cref{eq:noise_addition,eq:denoising}.
    If we set $t^{*}$ such that
    $$2\log\left(\frac{2\norm{\boldsymbol{\delta}}^{2} + 4d}{\mu \Delta}\right) \leq t^{*} \beta_{t^{*}} \leq \frac{\mu \Delta}{4d},$$
    then the estimation error between the sanitized $\bar{\boldsymbol{x}}_{0}$ and clean image $\boldsymbol{x}$ can be bounded as:
    $$\mathbb{E}\left[\norm{\bar{\boldsymbol{x}}_{0} - \boldsymbol{x}}^2\right] \leq 2(\mu + 1)\Delta,$$
    where $\Delta = \mathbb{E}[\norm{\boldsymbol{x}_0 - \boldsymbol{x}}^{2}]$ and $\mu > 0$ is a constant.
\end{customthm}
\begin{proof}
    We are looking to find an upper-bound for the estimation error between the sanitized image and its clean version.
    Using \Cref{lem:mean} we can write:
    \begin{align}\label{eq:addition}\nonumber
        \mathbb{E}\left[\norm{\bar{\boldsymbol{x}}_{0} - \boldsymbol{x}}^2\right] &= \mathbb{E}\left[\norm{(\bar{\boldsymbol{x}}_{0} - \boldsymbol{x}_0) + (\boldsymbol{x}_0 - \boldsymbol{x})}^2\right]\\\nonumber
        & \leq 2~\mathbb{E}\left[\norm{\bar{\boldsymbol{x}}_{0} - \boldsymbol{x}_0}^2\right] + 2~\mathbb{E}\left[\norm{\boldsymbol{x}_0 - \boldsymbol{x}}^2\right]\\
        & \leq 2~\mathbb{E}\left[\norm{\bar{\boldsymbol{x}}_{0} - \boldsymbol{x}_0}^2\right] + 2~\Delta.
    \end{align}
    Now, we need to find an upper-bound for the first term.
    To this end, we are going to use the contraction property of the DDPMs~(\Cref{cor:ddpm_contraction}).
    In particular, given the noisy versions of the clean $\boldsymbol{x}$ and the protected image $\tilde{\boldsymbol{x}} = \boldsymbol{x} + \boldsymbol{\delta}$, in other words:
    \begin{align}\label{eq:forward_addition}\nonumber
        \boldsymbol{x}_{t^{*}} &= \sqrt{\alpha_{t^{*}}} \boldsymbol{x} + \sqrt{1 - \alpha_{t^{*}}} \boldsymbol{\epsilon}_0
        \\
        \bar{\boldsymbol{x}}_{t^{*}} &= \sqrt{\alpha_{t^{*}}} \tilde{\boldsymbol{x}} + \sqrt{1 - \alpha_{t^{*}}} \boldsymbol{\epsilon}_0',
    \end{align}
    we know that both $\boldsymbol{x}_0$ and $\bar{\boldsymbol{x}}_0$ satisfy the reverse diffusion process, or:
    \begin{align}\label{eq:reverse_addition}\nonumber
        \boldsymbol{x}_{t-1}&=\frac{1}{\sqrt{1-\beta_{t}}}\left(\boldsymbol{x}_{t}+\beta_{t} \mathbf{s}_{\boldsymbol{\phi}}(\boldsymbol{x}_{t}, t)\right)+\sqrt{\beta_{t}} \boldsymbol{\epsilon}_{t}
        \\
        \bar{\boldsymbol{x}}_{t-1}&=\frac{1}{\sqrt{1-\beta_{t}}}\left(\bar{\boldsymbol{x}}_{t}+\beta_{t} \mathbf{s}_{\boldsymbol{\phi}}(\bar{\boldsymbol{x}}_{t}, t)\right)+\sqrt{\beta_{t}} \boldsymbol{\epsilon}_{t}', \quad \forall~t \in \{1, 2, \dots, t^{*}\},
    \end{align}
    where $\boldsymbol{\epsilon}_{t}, \boldsymbol{\epsilon}_{t}' \sim \mathcal{N}(\mathbf{0}, \mathbf{I})$.
    As such, we can treat $\boldsymbol{x}_0$ and $\bar{\boldsymbol{x}}_0$ as two sample trajectories of the same stochastic difference equation.
    Thus, by recursively applying \Cref{eq:sde_contraction} we would get:
    \begin{equation}\label{eq:contraction_cumulative}
        \mathbb{E}\left[\norm{\bar{\boldsymbol{x}}_{0}-\boldsymbol{x}_{0}}^2\right] \leq \textcolor{red}{\mathbb{E}\left[\norm{\bar{\boldsymbol{x}}_{t^{*}}-\boldsymbol{x}_{t^{*}}}^2\right]} \textcolor{blue}{\prod_{s=1}^{t^{*}}\lambda_{s}^{2}}+ \textcolor{darkgreen}{2 \sum_{s=1}^{t^{*}}C_{s} \prod_{r=1}^{s-1}\lambda_{r}^{2}}.
    \end{equation}
    Now, let us consider each term on the RHS of \Cref{eq:contraction_cumulative} separately.
    For the \textcolor{red}{red} term, we can write:
    \begin{align}\label{eq:red_term_1}\nonumber
        \textcolor{red}{\mathbb{E}\left[\norm{\bar{\boldsymbol{x}}_{t^{*}}-\boldsymbol{x}_{t^{*}}}^2\right]} &\stackrel{\mathrm{(1)}}{=} \mathbb{E}\left[\norm{\sqrt{\alpha_{t^{*}}} (\tilde{\boldsymbol{x}}-\boldsymbol{x}) + \sqrt{1 - \alpha_{t^{*}}} (\boldsymbol{\epsilon}_0'-\boldsymbol{\epsilon}_0)}^2\right] \\\nonumber
        &\stackrel{\mathrm{(2)}}{=} \mathbb{E}\left[\norm{\sqrt{\alpha_{t^{*}}} \boldsymbol{\delta} + \sqrt{1 - \alpha_{t^{*}}} (\boldsymbol{\epsilon}_0'-\boldsymbol{\epsilon}_0)}^2\right]\\\nonumber
        &= \norm{\sqrt{\alpha_{t^{*}}} \boldsymbol{\delta}}^2 + \mathbb{E}\left[\norm{\sqrt{1 - \alpha_{t^{*}}} (\boldsymbol{\epsilon}_0'-\boldsymbol{\epsilon}_0)}^2\right] + 2 \sqrt{\alpha_{t^{*}}}\sqrt{1 - \alpha_{t^{*}}} \boldsymbol{\delta}^\top\mathbb{E}\left[\boldsymbol{\epsilon}_0'-\boldsymbol{\epsilon}_0\right]\\
        &\stackrel{\mathrm{(3)}}{=} \alpha_{t^{*}}\norm{\boldsymbol{\delta}}^2 + (1 - \alpha_{t^{*}})\mathbb{E}\left[\norm{(\boldsymbol{\epsilon}_0'-\boldsymbol{\epsilon}_0)}^2\right].
    \end{align}
    where (1) is derived from \Cref{eq:forward_addition}, (2) holds since $\tilde{\boldsymbol{x}} = \boldsymbol{x} + \boldsymbol{\delta}$, and (3) is valid as $\boldsymbol{\epsilon}_{0}, \boldsymbol{\epsilon}_{0}' \sim \mathcal{N}(\mathbf{0}, \mathbf{I})$.
    Given that:
    $$\boldsymbol{\epsilon}_{0}' - \boldsymbol{\epsilon}_{0} \sim \mathcal{N}(\mathbf{0}, 2\mathbf{I}),$$
    we can simplify \Cref{eq:red_term_1} as:
    $$
        \textcolor{red}{\mathbb{E}\left[\norm{\bar{\boldsymbol{x}}_{t^{*}}-\boldsymbol{x}_{t^{*}}}^2\right]} = \alpha_{t^{*}}\norm{\boldsymbol{\delta}}^2 + 2 (1 - \alpha_{t^{*}})\mathbb{E}\left[\chi\right],
    $$
    where $\chi$ follows the chi-squared distribution with $d$ degrees of freedom.
    Using the fact that $0 < \alpha_{t^{*}} < 1$, we can finally write:
    \begin{align}\label{eq:red_term_final}\nonumber
        \textcolor{red}{\mathbb{E}\left[\norm{\bar{\boldsymbol{x}}_{t^{*}}-\boldsymbol{x}_{t^{*}}}^2\right]}
        &= \alpha_{t^{*}}\norm{\boldsymbol{\delta}}^2 + 2 (1 - \alpha_{t^{*}})d \\ 
        &\leq \norm{\boldsymbol{\delta}}^2 + 2d.
    \end{align}
    Using \Cref{lem:exp}, for the \textcolor{blue}{blue} term in \Cref{eq:contraction_cumulative} we can write:
    \begin{equation}\label{eq:blue_term}
        \textcolor{blue}{\prod_{s=1}^{t^{*}} \lambda_{s}^{2}} \leq \exp({-\frac{t^{*}\beta_{t^{*}}}{2}}).
    \end{equation}
    Finally, for the \textcolor{darkgreen}{green} term we have:
    \begin{align}\label{eq:green_term}\nonumber
        \textcolor{darkgreen}{2 \sum_{s=1}^{t^{*}}C_{s} \prod_{r=1}^{s-1}\lambda_{r}^{2}}
        &\stackrel{\mathrm{(1)}}{=} 2 \sum_{s=1}^{t^{*}}d\beta_s \prod_{r=1}^{s-1}\lambda_{r}^{2} \\\nonumber
        &\stackrel{\mathrm{(2)}}{\leq} 2 \sum_{s=1}^{t^{*}}d\beta_s\\
        &\stackrel{\mathrm{(3)}}{\leq} 2d t^{*}\beta_{t^{*}}.
    \end{align}
    Here, (1) is the result of \Cref{eq:C_t}, (2) holds since $0 < \lambda_r < 1$~(see~\Cref{eq:lambda_t}), and (3) is derived from ${0 < \beta_{1} < \cdots < \beta_{t} < 1}$.

    Putting \Cref{eq:red_term_final,eq:blue_term,eq:green_term} together, we have:
    \begin{equation}\label{eq:contraction_cumulative_simplified}
        \mathbb{E}\left[\norm{\bar{\boldsymbol{x}}_{0}-\boldsymbol{x}_{0}}^2\right] \leq \left(\norm{\boldsymbol{\delta}}^2 + 2d\right) \exp({-\frac{t^{*}\beta_{t^{*}}}{2}}) + 2d t^{*}\beta_{t^{*}}.
    \end{equation}
    Given that:
    $$2\log\left(\frac{2\norm{\boldsymbol{\delta}}^{2} + 4d}{\mu \Delta}\right) \leq t^{*} \beta_{t^{*}} \leq \frac{\mu \Delta}{4d},$$
    we can simplify \Cref{eq:contraction_cumulative_simplified} as:
    \begin{align}\label{eq:contraction_cumulative_final}\nonumber
        \mathbb{E}\left[\norm{\bar{\boldsymbol{x}}_{0}-\boldsymbol{x}_{0}}^2\right] &\leq \left(\norm{\boldsymbol{\delta}}^2 + 2d\right) \exp({-\frac{t^{*}\beta_{t^{*}}}{2}}) + 2d t^{*}\beta_{t^{*}}\\\nonumber
        &\leq \left(\norm{\boldsymbol{\delta}}^2 + 2d\right)\frac{\mu \Delta}{2\norm{\boldsymbol{\delta}}^{2} + 4d} + 2d \frac{\mu \Delta}{4d}\\
        &\leq \mu \Delta.
    \end{align}
    Replacing \Cref{eq:contraction_cumulative_final} into \Cref{eq:addition}, the proof can be completed.
\end{proof}

\newpage
\section{Additional Experimental Results}\label{sec:additional_experiments}
In this section, we provide additional experiments and insights that were omitted from the main paper due to space limitations.

\subsection{Denoising Samples}\label{sec:additional_experiments:samples}
\Cref{fig:IN_samples_III,fig:WebFace_Samples} include samples from the protected IN-100 and WebFace datasets alongside their denoised ones. 
As seen, \textsc{Avatar} can successfully recover the benign data except cases where the perturbations are sever enough to remain visible.
In these cases, however, the protected data has lost its normal utility due to the visibility of the protecting perturbation.

\begin{figure*}[htp!]
	\centering
        \begin{subfigure}{.90\textwidth}
        \rotatebox[origin=c]{90}{Pert.}\hspace{0.95em}
    	\begin{subfigure}{.90\textwidth}
    		\centering
    		\includegraphics[width=1.0\textwidth, valign=c]{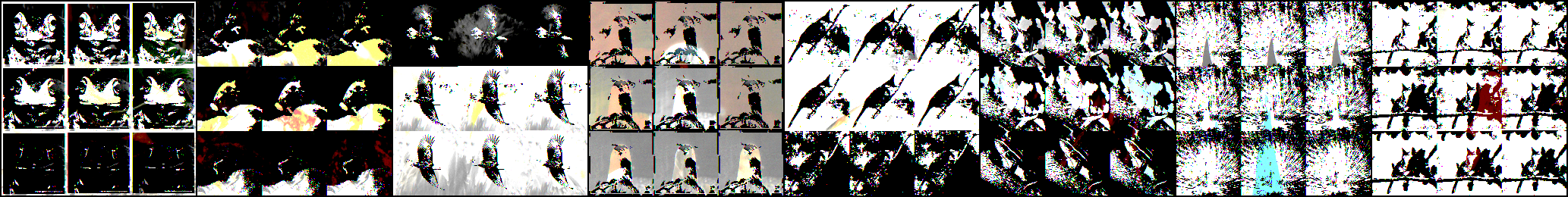}
    	\end{subfigure}\\
        \rotatebox[origin=c]{90}{Input}\hspace{0.75em}
    	\begin{subfigure}{.90\textwidth}
    		\centering
    		\includegraphics[width=1.0\textwidth, valign=c]{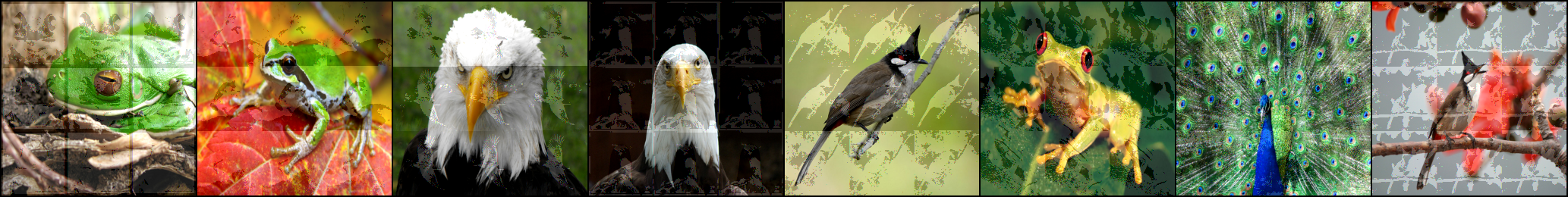}
    	\end{subfigure}\\
        \rotatebox[origin=c]{90}{Noisy}\hspace{0.75em}
    	\begin{subfigure}{.90\textwidth}
    		\centering
    		\includegraphics[width=1.0\textwidth, valign=c]{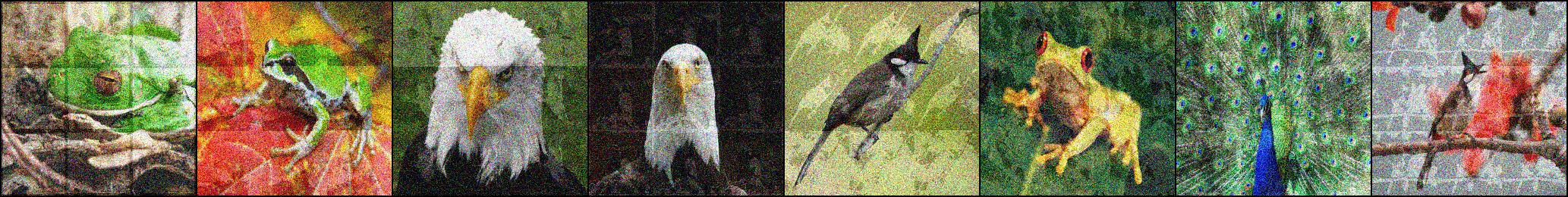}
    	\end{subfigure}\\
        \rotatebox[origin=c]{90}{Denoised}\hspace{0.95em}
    	\begin{subfigure}{.90\textwidth}
    		\centering
    		\includegraphics[width=1.0\textwidth, valign=c]{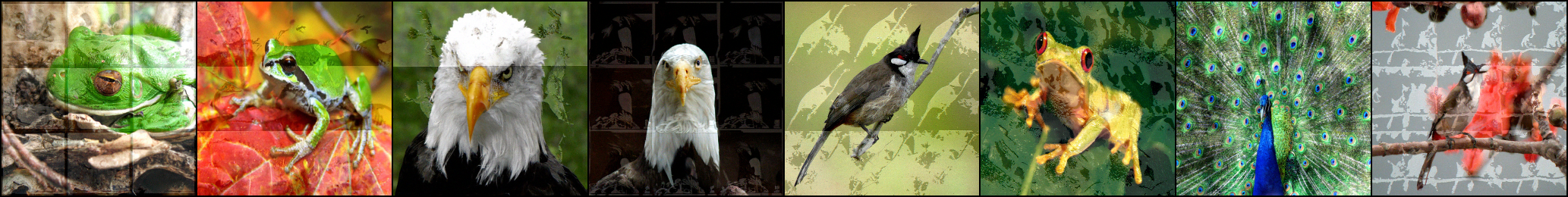}
    	\end{subfigure}
	\caption{NTGA~\citep{yuan2021ntga}}
    \vspace{2em}
    \end{subfigure}
    \begin{subfigure}{.90\textwidth}
        \rotatebox[origin=c]{90}{Pert.}\hspace{0.95em}
    	\begin{subfigure}{.90\textwidth}
    		\centering
    		\includegraphics[width=1.0\textwidth, valign=c]{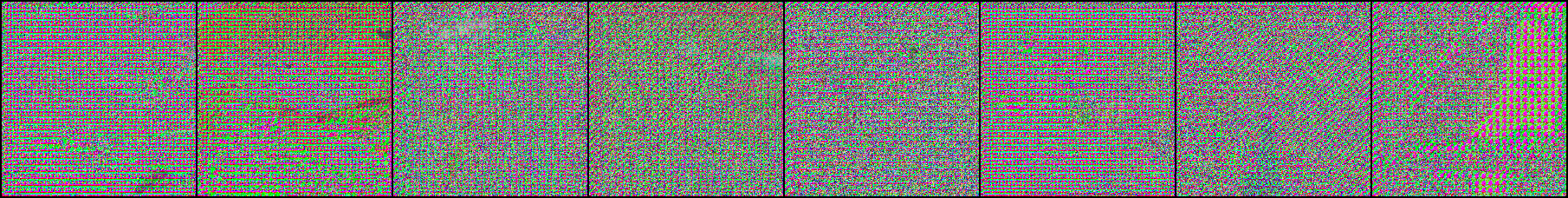}
    	\end{subfigure}\\
        \rotatebox[origin=c]{90}{Input}\hspace{0.75em}
    	\begin{subfigure}{.90\textwidth}
    		\centering
    		\includegraphics[width=1.0\textwidth, valign=c]{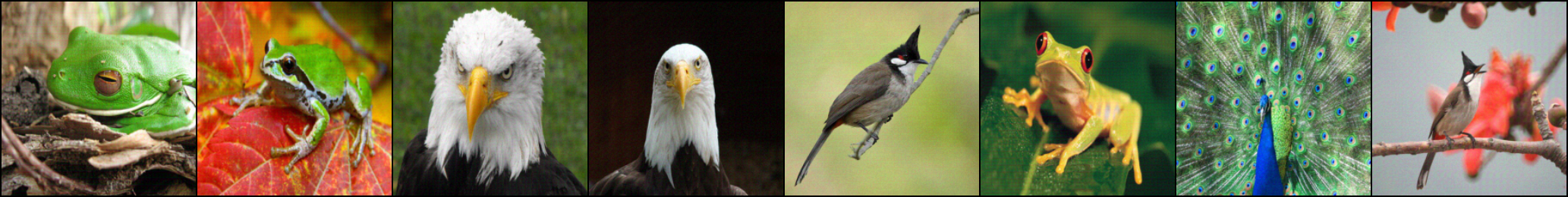}
    	\end{subfigure}\\
        \rotatebox[origin=c]{90}{Noisy}\hspace{0.75em}
    	\begin{subfigure}{.90\textwidth}
    		\centering
    		\includegraphics[width=1.0\textwidth, valign=c]{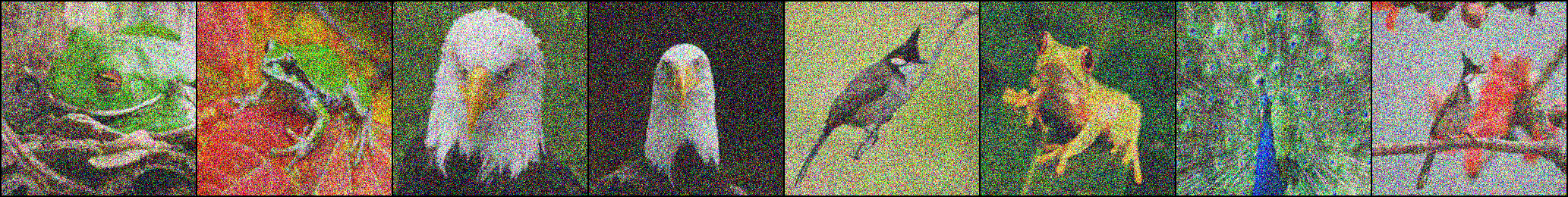}
    	\end{subfigure}\\
        \rotatebox[origin=c]{90}{Denoised}\hspace{0.95em}
    	\begin{subfigure}{.90\textwidth}
    		\centering
    		\includegraphics[width=1.0\textwidth, valign=c]{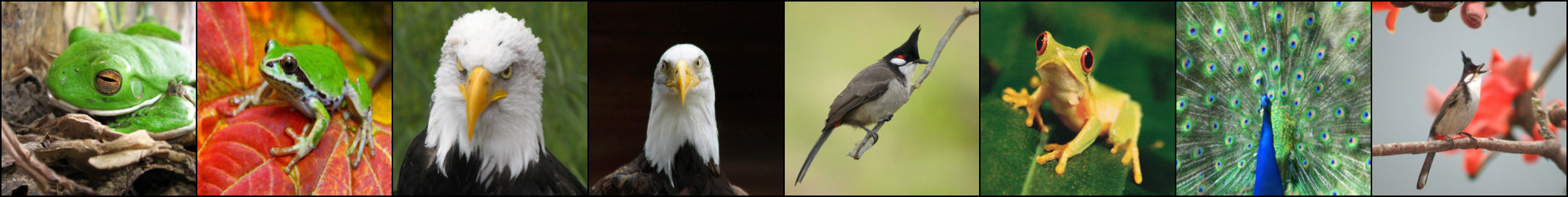}
    	\end{subfigure}
	\caption{EMN~\citep{huang2021emn}}
    \end{subfigure}
	\label{fig:IN_samples_I}
\end{figure*}

\begin{figure*}[htp!]
	\ContinuedFloat\centering
    \vspace{-3em}
    \begin{subfigure}{.90\textwidth}
        \rotatebox[origin=c]{90}{Pert.}\hspace{0.95em}
    	\begin{subfigure}{.90\textwidth}
    		\centering
    		\includegraphics[width=1.0\textwidth, valign=c]{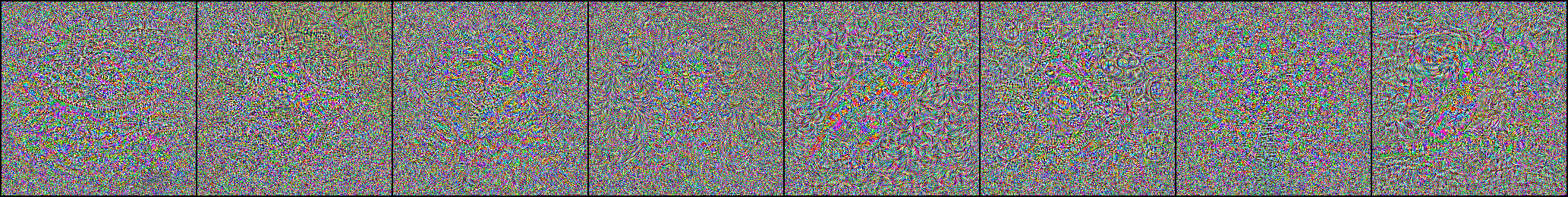}
    	\end{subfigure}\\
        \rotatebox[origin=c]{90}{Input}\hspace{0.75em}
    	\begin{subfigure}{.90\textwidth}
    		\centering
    		\includegraphics[width=1.0\textwidth, valign=c]{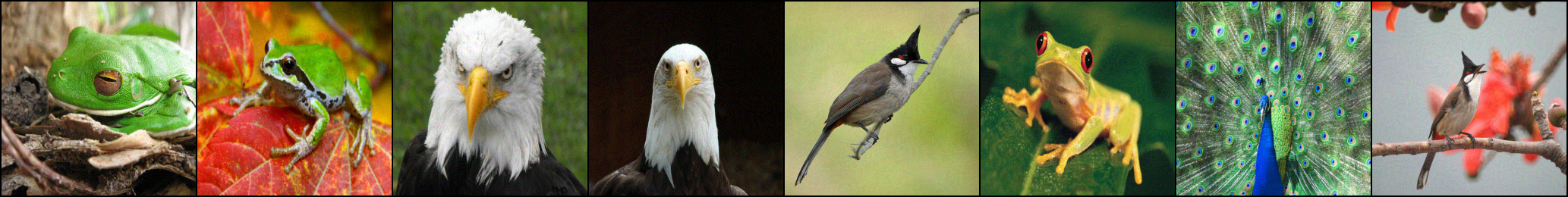}
    	\end{subfigure}\\
        \rotatebox[origin=c]{90}{Noisy}\hspace{0.75em}
    	\begin{subfigure}{.90\textwidth}
    		\centering
    		\includegraphics[width=1.0\textwidth, valign=c]{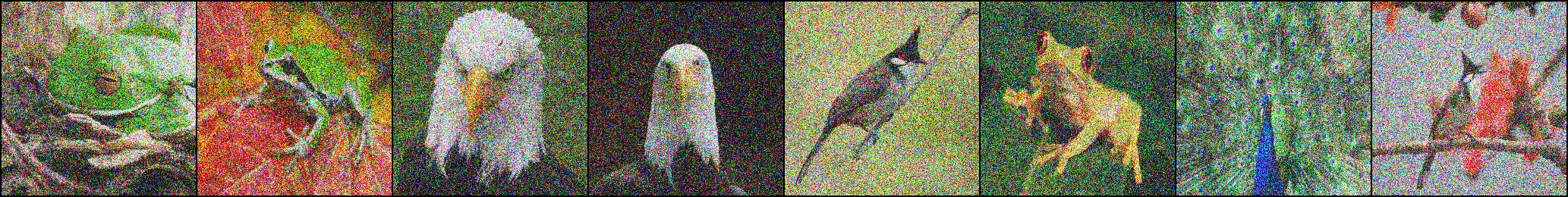}
    	\end{subfigure}\\
        \rotatebox[origin=c]{90}{Denoised}\hspace{0.95em}
    	\begin{subfigure}{.90\textwidth}
    		\centering
    		\includegraphics[width=1.0\textwidth, valign=c]{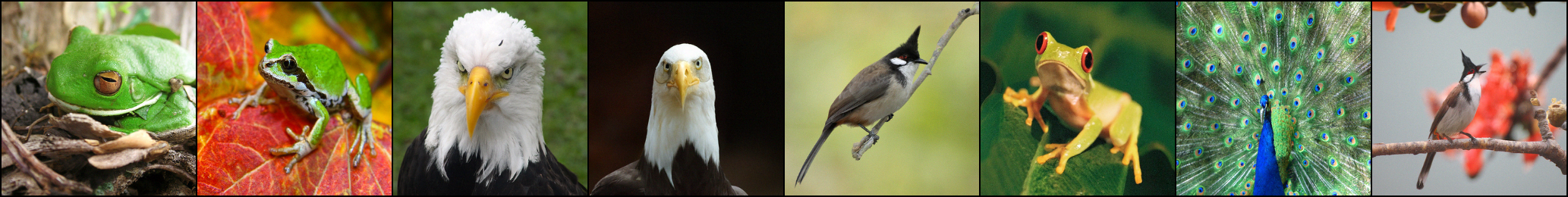}
    	\end{subfigure}
	\caption{TAP~\citep{fowl2021tap}}
    \vspace{2em}
    \end{subfigure}
     \begin{subfigure}{.90\textwidth}
        \rotatebox[origin=c]{90}{Pert.}\hspace{0.95em}
    	\begin{subfigure}{.90\textwidth}
    		\centering
    		\includegraphics[width=1.0\textwidth, valign=c]{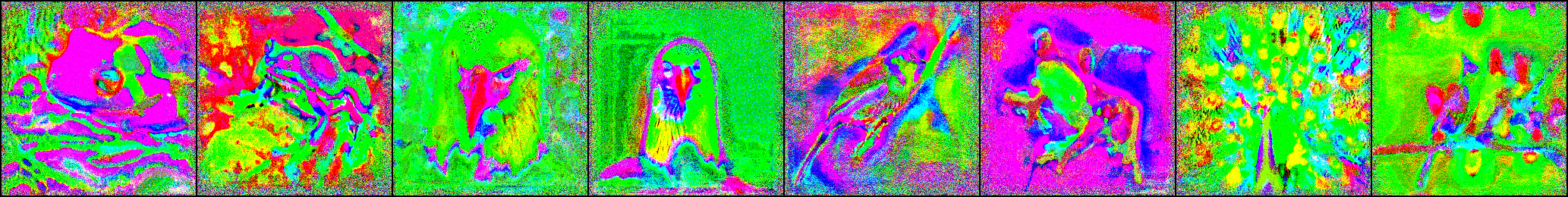}
    	\end{subfigure}\\
        \rotatebox[origin=c]{90}{Input}\hspace{0.75em}
    	\begin{subfigure}{.90\textwidth}
    		\centering
    		\includegraphics[width=1.0\textwidth, valign=c]{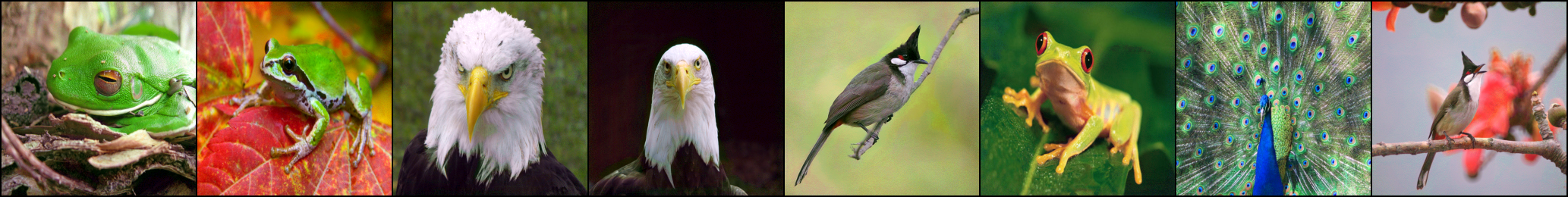}
    	\end{subfigure}\\
        \rotatebox[origin=c]{90}{Noisy}\hspace{0.75em}
    	\begin{subfigure}{.90\textwidth}
    		\centering
    		\includegraphics[width=1.0\textwidth, valign=c]{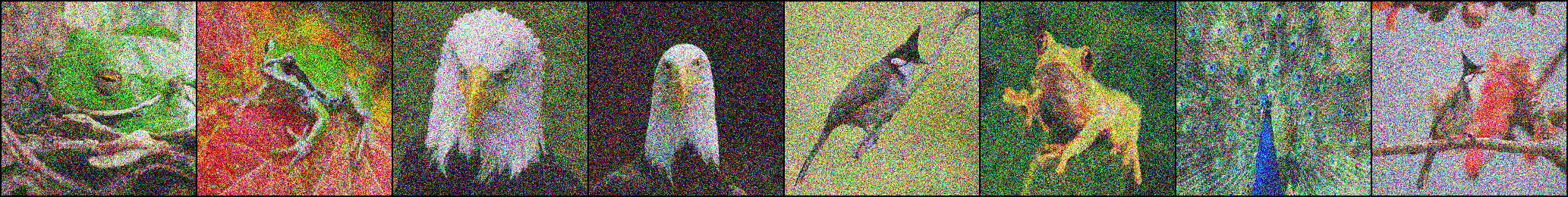}
    	\end{subfigure}\\
        \rotatebox[origin=c]{90}{Denoised}\hspace{0.95em}
    	\begin{subfigure}{.90\textwidth}
    		\centering
    		\includegraphics[width=1.0\textwidth, valign=c]{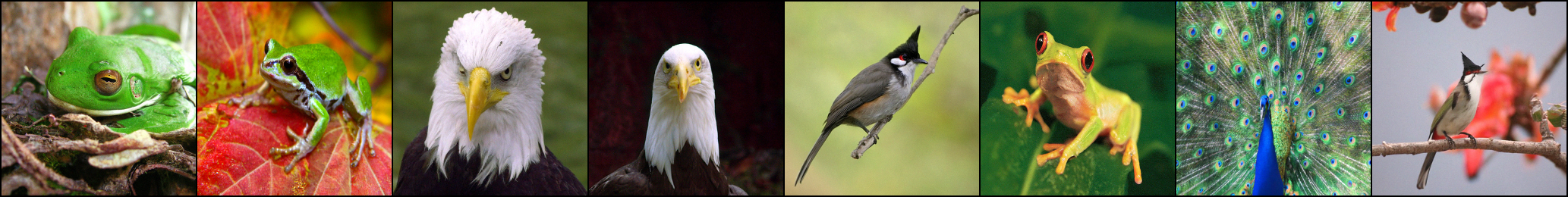}
    \end{subfigure}
	\caption{REMN~\citep{fu2022remn}}
    \vspace{2em}
    \end{subfigure}
    \begin{subfigure}{.90\textwidth}
        \rotatebox[origin=c]{90}{Pert.}\hspace{0.95em}
    	\begin{subfigure}{.90\textwidth}
    		\centering
    		\includegraphics[width=1.0\textwidth, valign=c]{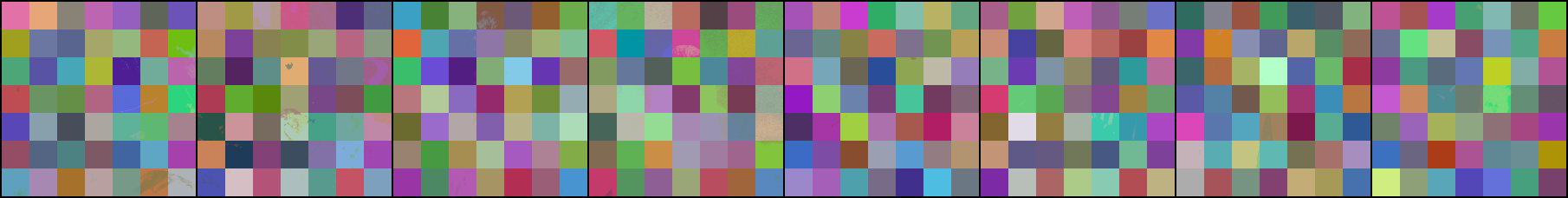}
    	\end{subfigure}\\
        \rotatebox[origin=c]{90}{Input}\hspace{0.75em}
    	\begin{subfigure}{.90\textwidth}
    		\centering
    		\includegraphics[width=1.0\textwidth, valign=c]{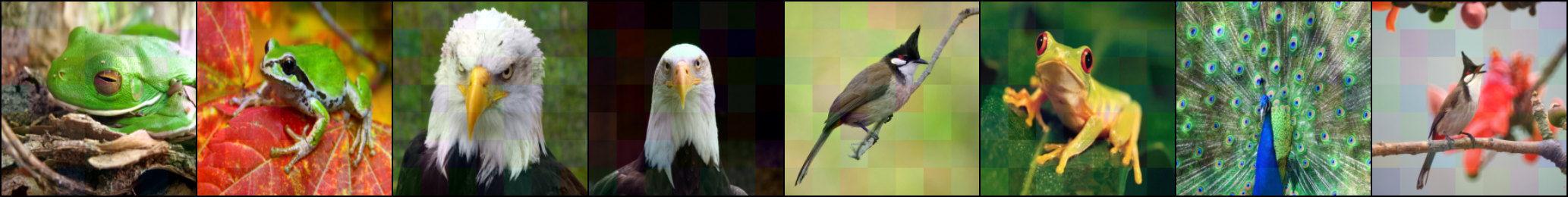}
    	\end{subfigure}\\
        \rotatebox[origin=c]{90}{Noisy}\hspace{0.75em}
    	\begin{subfigure}{.90\textwidth}
    		\centering
    		\includegraphics[width=1.0\textwidth, valign=c]{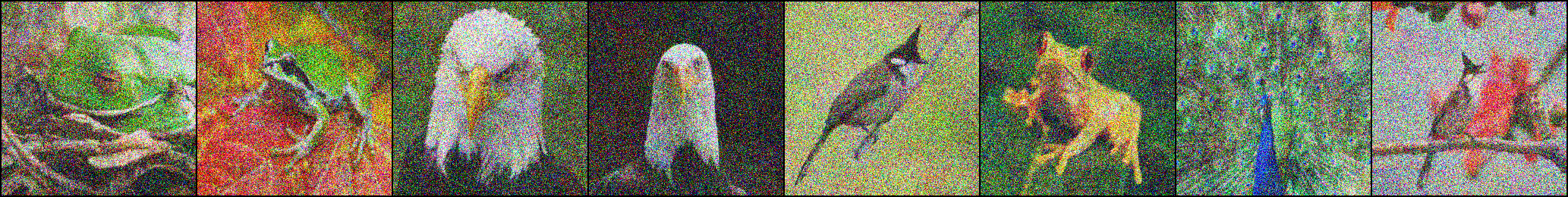}
    	\end{subfigure}\\
        \rotatebox[origin=c]{90}{Denoised}\hspace{0.95em}
    	\begin{subfigure}{.90\textwidth}
    		\centering
    		\includegraphics[width=1.0\textwidth, valign=c]{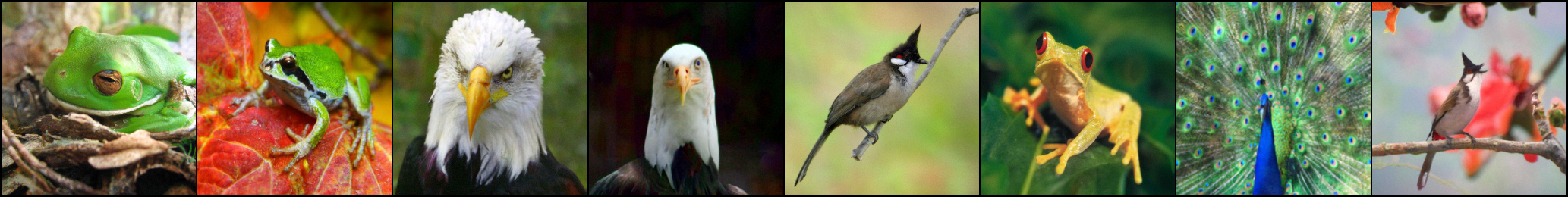}
    	\end{subfigure}
	\caption{SHR~\citep{yu2022shr}}
    \end{subfigure}
    \caption{Samples from IN-100 dataset. For each attack, we show the perturbation, the protected image, the noisy version of the image, and the denoised one using \textsc{Avatar}.}
	\label{fig:IN_samples_III}
\end{figure*}

\begin{figure*}[p!]
    \begin{center}
    	\begin{subfigure}{.5\textwidth}
    		\centering
    		\includegraphics[width=1.0\textwidth, valign=c]{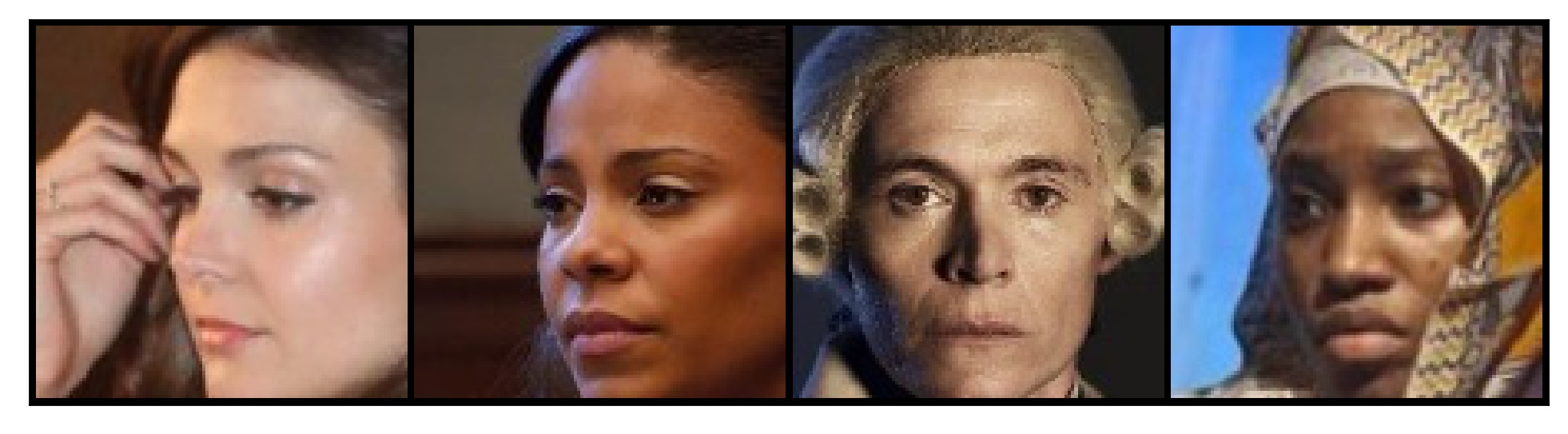}
            \caption*{Clean Image}
    	\end{subfigure}
    \end{center}
    \begin{subfigure}{1.0\textwidth}
        \begin{center}
        	\begin{subfigure}{.5\textwidth}
        		\centering
        		\includegraphics[width=1.0\textwidth, valign=c]{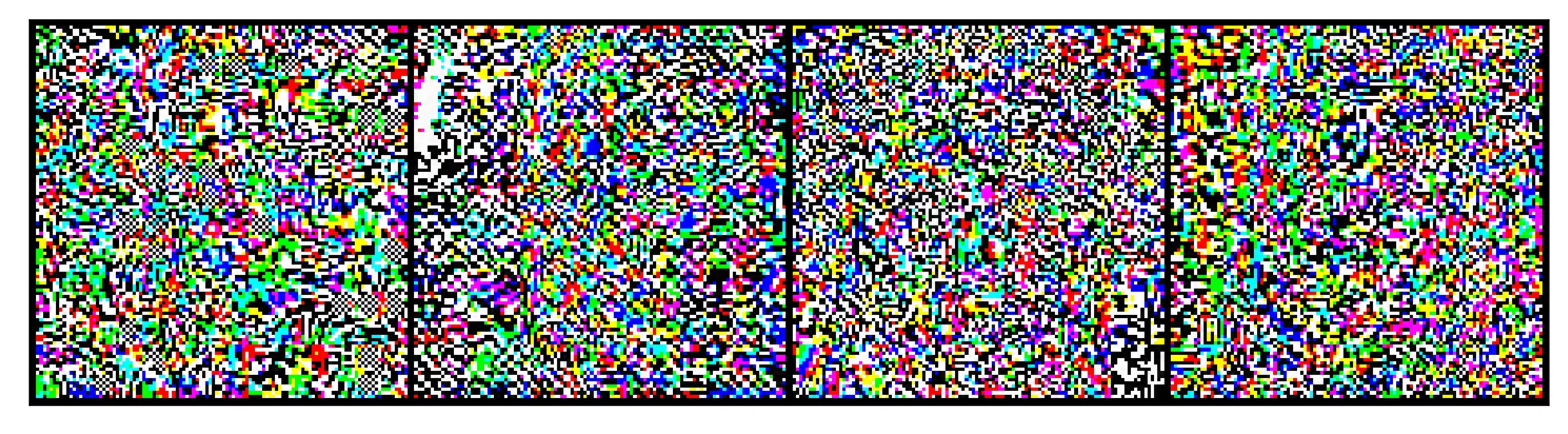}
                \caption*{Perturbation}
        	\end{subfigure}
        \end{center}
    	\begin{subfigure}{.5\textwidth}
    		\centering
    		\includegraphics[width=1.0\textwidth, valign=c]{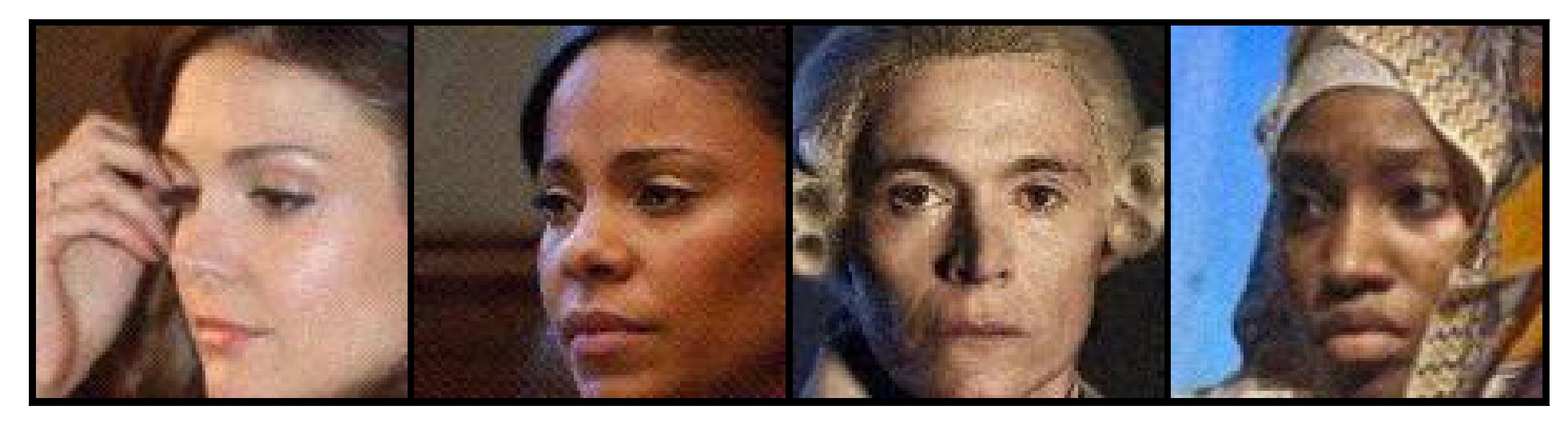}
            \caption*{Protected Image}
    	\end{subfigure}
    	\begin{subfigure}{.5\textwidth}
    		\centering
    		\includegraphics[width=1.0\textwidth, valign=c]{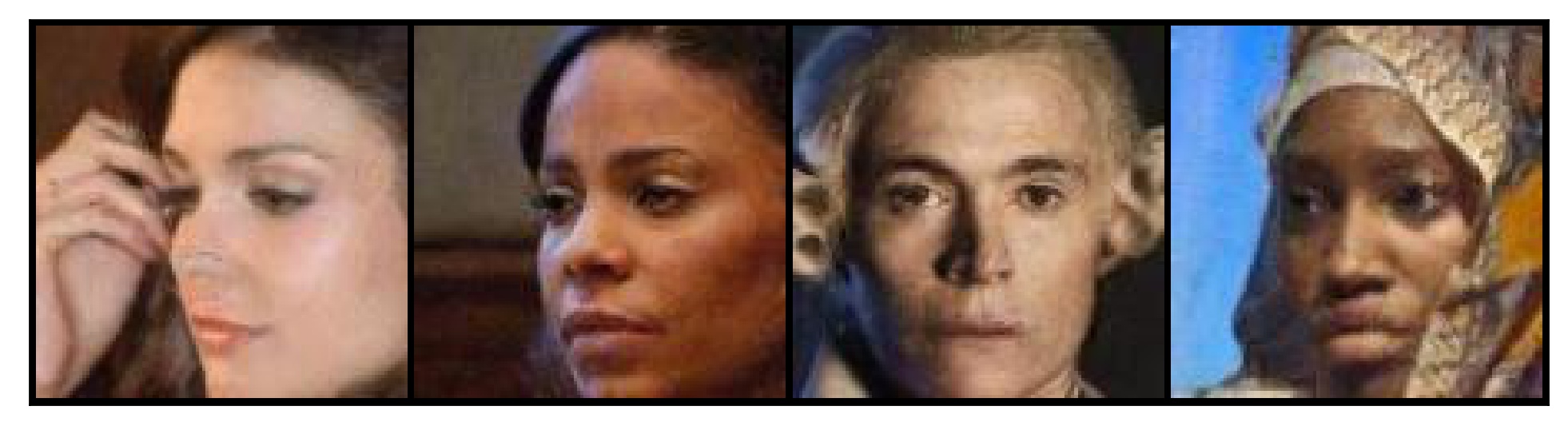}
            \caption*{Denoised Image}
    	\end{subfigure}
	\caption{EMN~\citep{huang2021emn}}
    \end{subfigure}
    \begin{subfigure}{1.0\textwidth}
        \begin{center}
        	\begin{subfigure}{.5\textwidth}
        		\centering
        		\includegraphics[width=1.0\textwidth, valign=c]{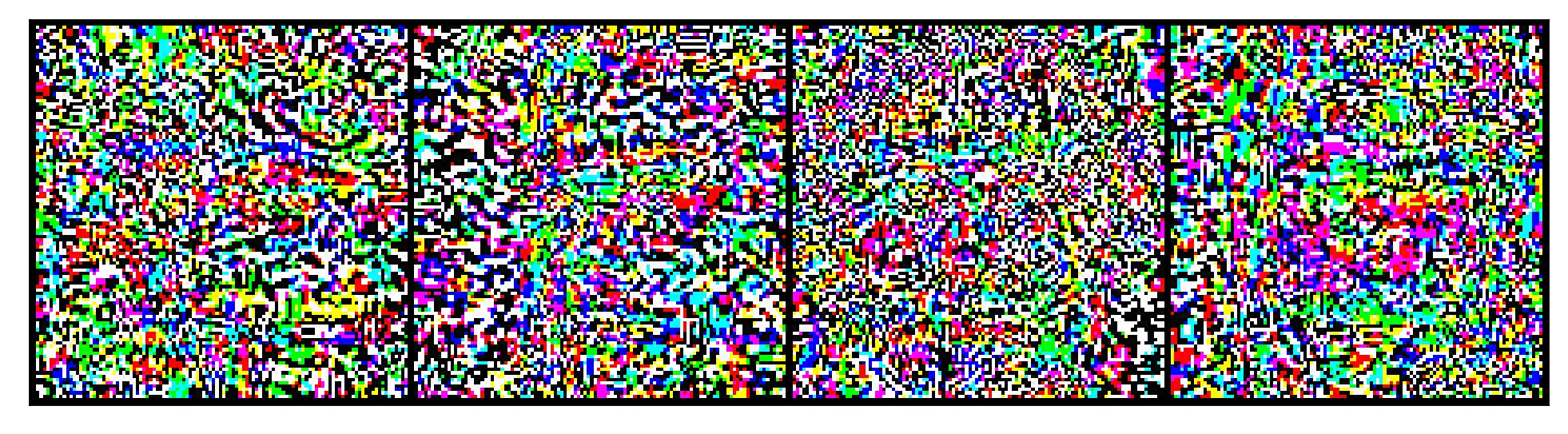}
                \caption*{Perturbation}
        	\end{subfigure}
        \end{center}
    	\begin{subfigure}{.5\textwidth}
    		\centering
    		\includegraphics[width=1.0\textwidth, valign=c]{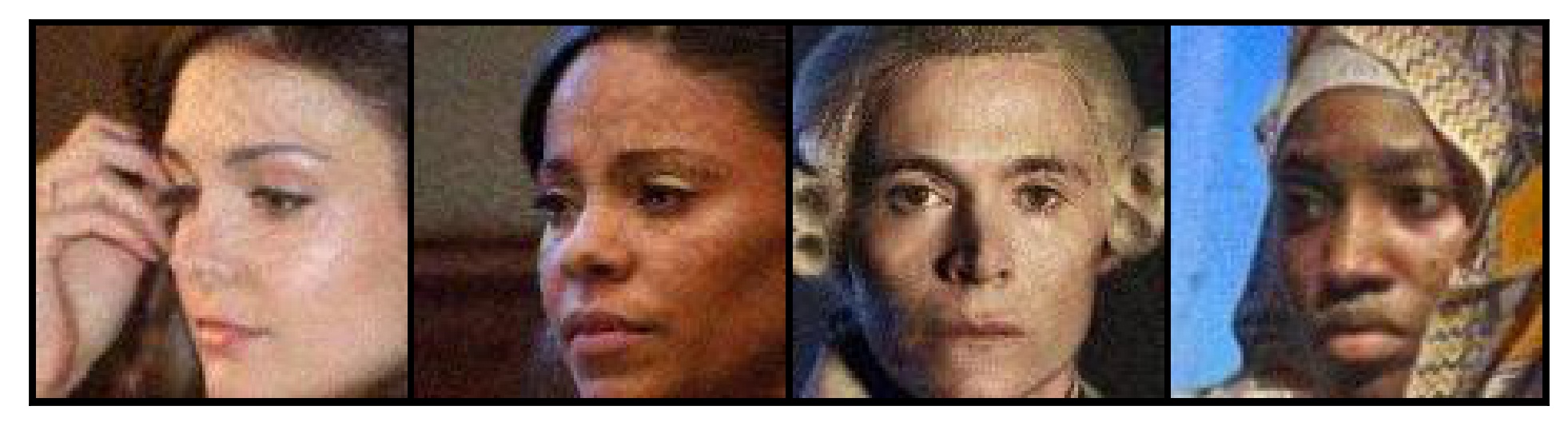}
            \caption*{Protected Image}
    	\end{subfigure}
    	\begin{subfigure}{.5\textwidth}
    		\centering
    		\includegraphics[width=1.0\textwidth, valign=c]{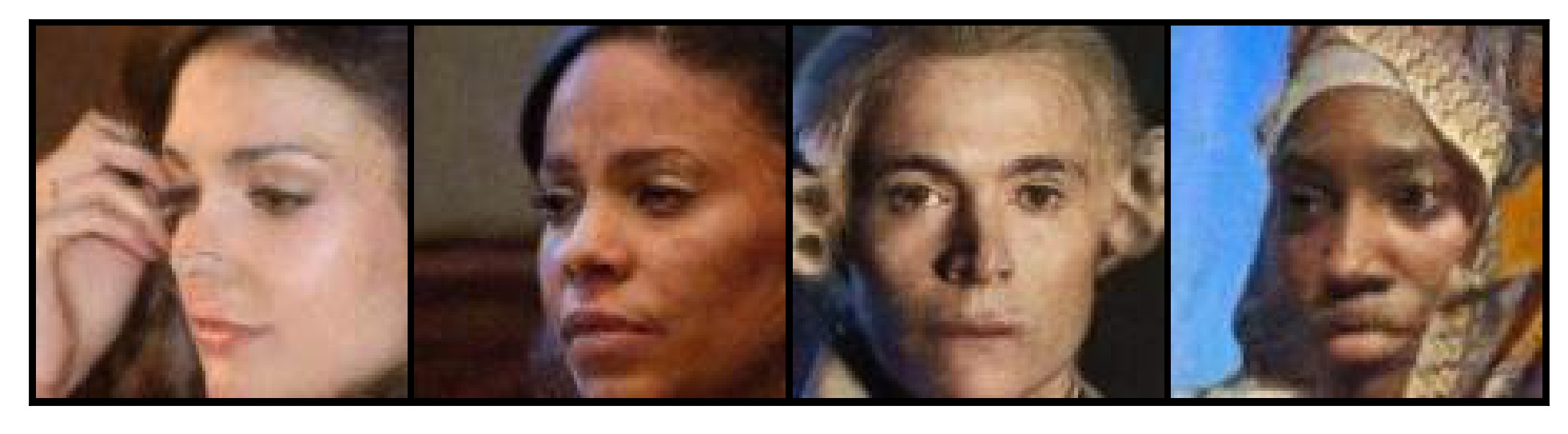}
            \caption*{Denoised Image}
    	\end{subfigure}
	\caption{TAP~\citep{fowl2021tap}}
    \end{subfigure}
    \begin{subfigure}{1.0\textwidth}
        \begin{center}
        	\begin{subfigure}{.5\textwidth}
        		\centering
        		\includegraphics[width=1.0\textwidth, valign=c]{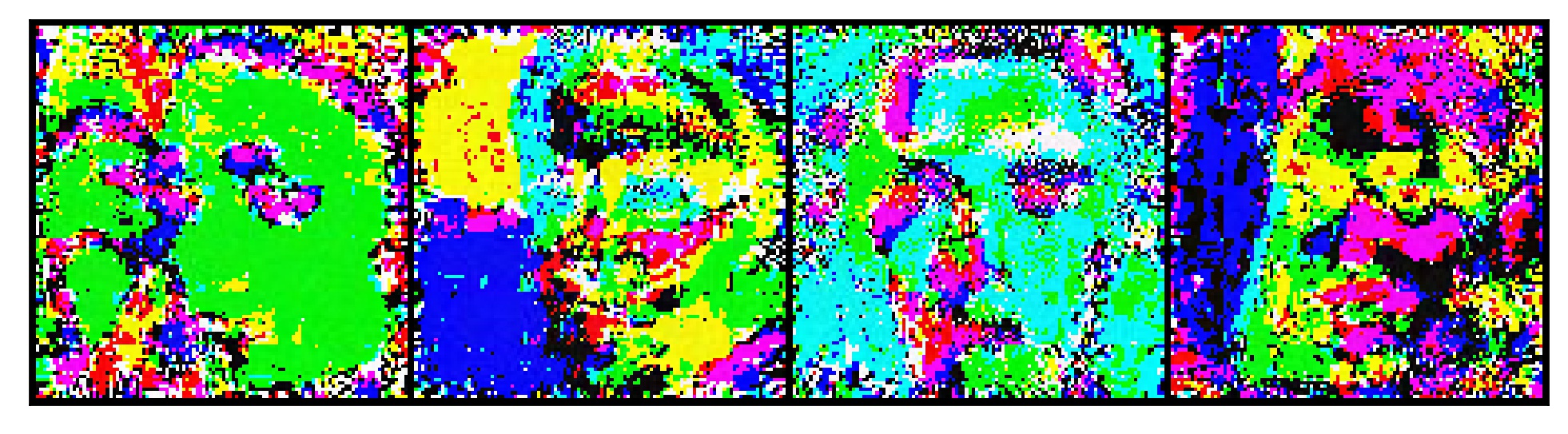}
                \caption*{Perturbation}
        	\end{subfigure}
        \end{center}
    	\begin{subfigure}{.5\textwidth}
    		\centering
    		\includegraphics[width=1.0\textwidth, valign=c]{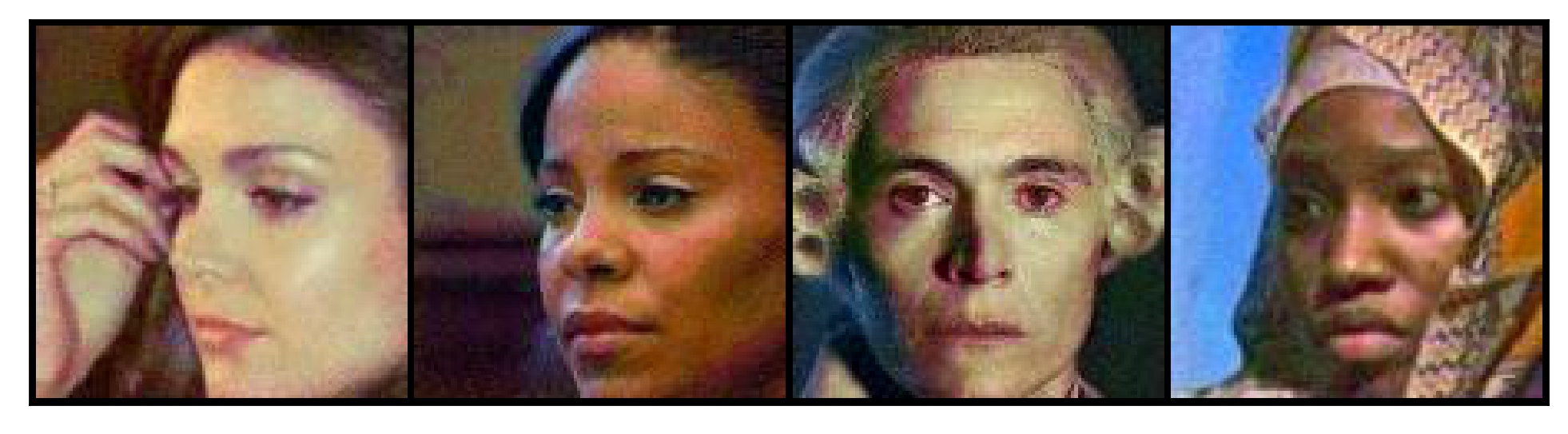}
            \caption*{Protected Image}
    	\end{subfigure}
    	\begin{subfigure}{.5\textwidth}
    		\centering
    		\includegraphics[width=1.0\textwidth, valign=c]{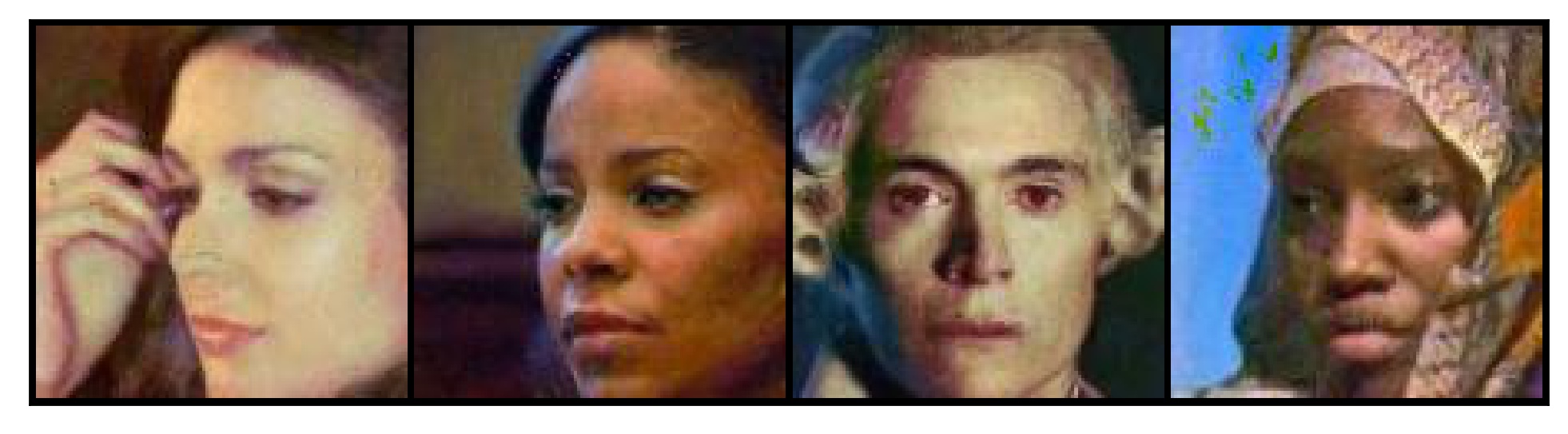}
            \caption*{Denoised Image}
    	\end{subfigure}
	\caption{REMN~\citep{huang2021emn}}
    \end{subfigure}
\end{figure*}

\begin{figure*}[tb!]
    \ContinuedFloat
    \begin{subfigure}{1.0\textwidth}
        \begin{center}
        	\begin{subfigure}{.5\textwidth}
        		\centering
        		\includegraphics[width=1.0\textwidth, valign=c]{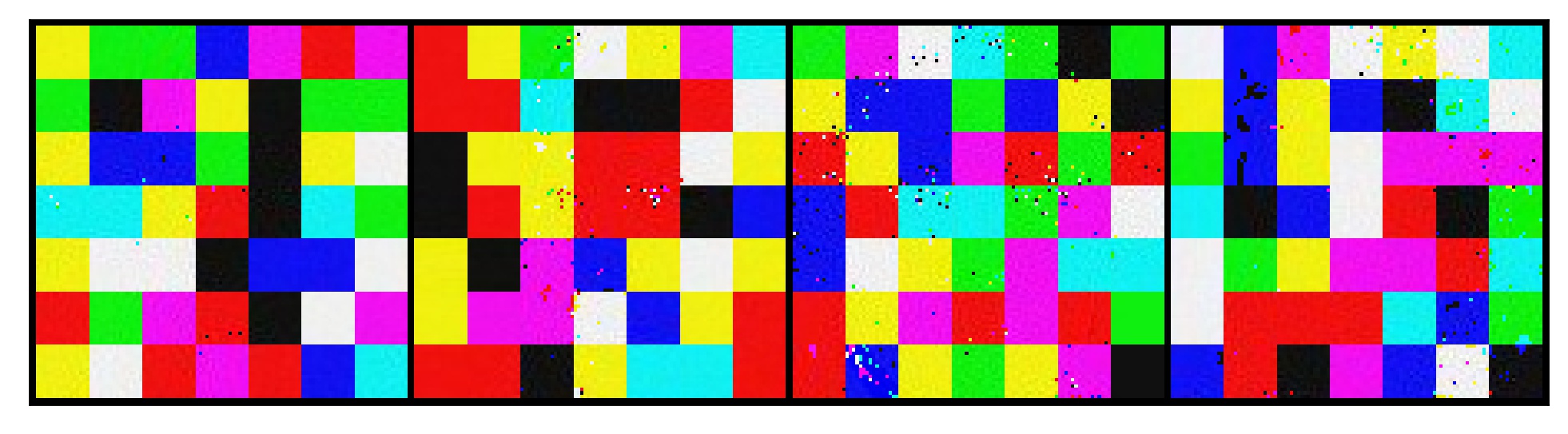}
                \caption*{Perturbation}
        	\end{subfigure}
        \end{center}
    	\begin{subfigure}{.5\textwidth}
    		\centering
    		\includegraphics[width=1.0\textwidth, valign=c]{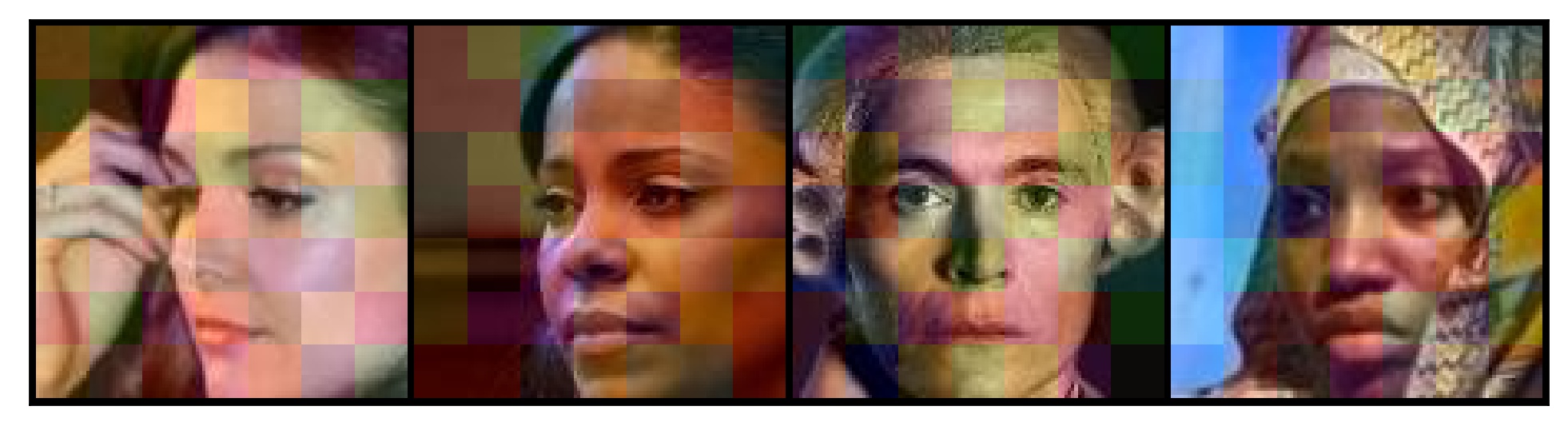}
            \caption*{Protected Image}
    	\end{subfigure}
    	\begin{subfigure}{.5\textwidth}
    		\centering
    		\includegraphics[width=1.0\textwidth, valign=c]{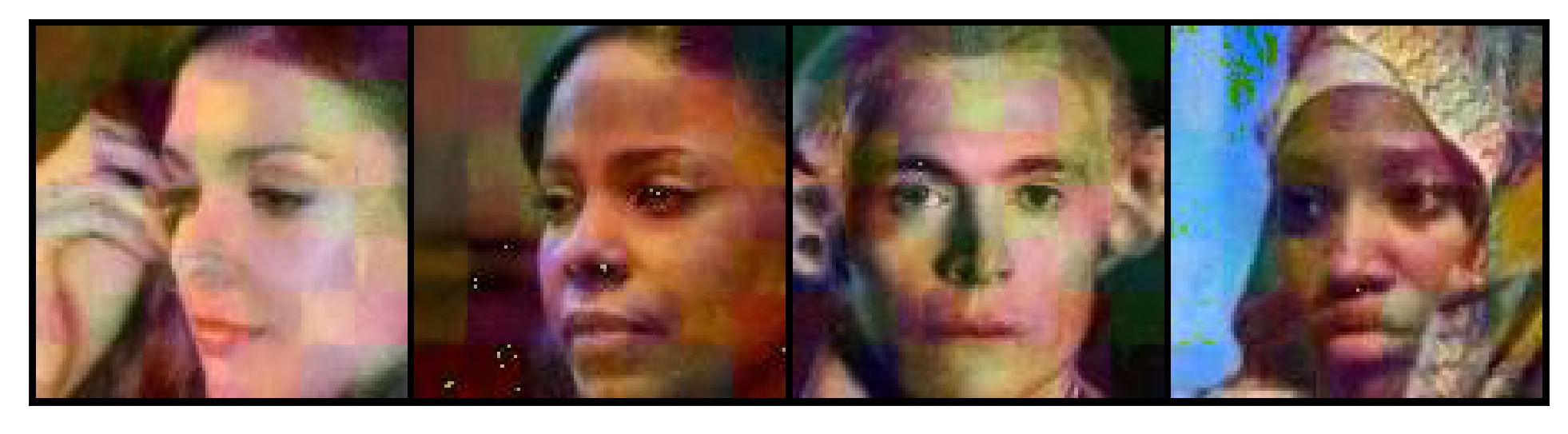}
            \caption*{Denoised Image}
    	\end{subfigure}
	\caption{SHR~\citep{yu2022shr}}
    \end{subfigure}
    \caption{Samples from the WebFace~\citep{yi2014} dataset.
             In each case, we generate the data-protecting perturbations with a maximum magnitude of $16/255$.
             For denoising, we use \textsc{Avatar} based on a diffusion model pre-trained on the CelebA~\citep{liu2015deep} dataset.
             As seen, the SHR~\citep{yu2022shr} perturbations leave a noticeable pattern over the protected data, which put their utility (e.g., posting over the social media) under question.
             Nevertheless, since they have a more global structure, they remain persistent even after denoising.
             }
	\label{fig:WebFace_Samples}
\end{figure*}

\newpage
\subsection{Extended Experimental Results over Different Architectures}
In \Cref{tab:architecture}, we presented our results on training RN-18 models over protected data. 
To show the applicability of our approach across various architectures, we also report our results for three additional architectures, namely DN-121, VGG-16, and WRN-34, in \Cref{tab:architecture_appendix}.
Similar to our RN-18 experiments, \textsc{Avatar} delivers the best performance against protected data.

\subsection{On Selecting Diffusion Step $t^{*}$}\label{sec:additional_experiments:selection}
In \Cref{fig:timestep}, we demonstrated that setting $t^{*}=100$ delivers a consistent performance across different architectures.
However, chances are that practitioners may want to replace the diffusion model used in \textsc{Avatar} with one of their own.
In such cases, the diffusion model might have different characteristics compared to the ones used in this paper.
In this part, we present two methods for setting $t^{*}$.

\subsubsection{Using the $\alpha_{t}$ Curves}
A na\"ive approach in selecting a suitable diffusion timestep $t^{*}$ is using the $\alpha_{t}$ curves between the new diffusion model and a reference model.
Specifically, since the value of $\alpha_{t}$ in~\Cref{eq:forward_process} controls the amount of disruptive noise, we can use the value of $\alpha_{t}$ to guide our hyper-parameter selection.
To this end, we can find an equivalent $t^{*}$ such that the value of $\alpha_{t^{*}}$ is set to an acceptable value.
This is because if too much disruptive noise is required to be added to the data to counteract the protecting perturbation, it means that the data has already been corrupted so much that it has lost its utility in the first place.

We demonstrate this approach for selecting the timestep $t^{*}$ for our IN-1k-32$\times$32 experiments in~\Cref{tab:dist_mismatch_new}.
As discussed in~\Cref{sec:sec:mismatch}, for this new experiment we want to use a guided diffusion model (DDPM-IP~\citep{ning2023ddpmip}) which uses a cosine schedule for sampling.
As per our prior experience, we know that an acceptable value for $t^{*}$ using a linear scheduler is $100$.
As such, we can draw the $\alpha_t$ curve for both cases, and find an equivalent $t^{*}$ for the cosine scheduler in DDPM-IP.
As shown in~\Cref{fig:cosine_vs_linear}, we can see that in this new case we should set $t^{*}=200$ to get an equivalent $\alpha_t$ as the one which we previously used for the CIFAR-10 experiments.

\subsubsection{Using Reconstruction Quality}
Another approach to set a viable value for the diffusion timestep $t^{*}$ is through controlling a desirable reconstruction quality.
Recall that the goal of availability attacks is to preserve the normal utility of the data.
As such, they usually aim to add imperceptible perturbations to the data.
This assumption can help us in selecting a good value for $t^{*}$.
In particular, having a small portion of clean data, we can run the denoising process of \textsc{Avatar} on these benign data and record a reconstruction Peak-to-Signal-Noise-Ratio~(PSNR) for different values of $t^{*}$. 
In general, as we move towards larger $t^{*}$, the PSNR drops. 
We can set an acceptable level of PSNR value, for example $22$dB, to select $t^{*}$. 
Beyond that, the PSNR drops so significantly that both the clean and protected data become unreasonably noisy, losing their utility.

To demonstrate this point through our IN-1k-32$\times$32 experiments in~\Cref{tab:dist_mismatch_new}, we have reported \textsc{Avatar}'s reconstruction PSNR for different values of $t$ in \Cref{tab:psnr_quality}.
As seen, while $t^{*}=100$ reaches a PSNR value of $22.52$dB when we use a linear scheduler for sampling, we can still get a reasonable PSNR of $23.71$dB for $t^{*}=200$ in DDPM-IP.
Therefore, we can pick $t^{*}=200$ for denoising using the IN-1k-32$\times$32 model.

\begin{table*}[p!]
	\caption{Test accuracy (\%) of various neural network architectures trained over data availability attacks on CIFAR-10, CIFAR-100, SVHN, and ImageNet-100 datasets without and with our denoising approach. The mean and standard deviation are computed over 5 seeds.}
	\label{tab:architecture_appendix}
	\begin{center}
		\begin{small}
		    \setlength\tabcolsep{0.45em}
			\def\arraystretch{1.65}
			\begin{tabular}{ccccccccccc}
				\toprule
                \multirow{2}{*}{\rotatebox[origin=c]{90}{{\textbf{Data}}}}
				&\multirow{2}{*}{\rotatebox[origin=c]{90}{{\textbf{Model}}}}
				&\multirow{2}{*}{\textbf{Method}}
                &\multirow{2}{*}{\textbf{Clean}}
				&\multicolumn{6}{c}{\textbf{Data Availability Attacks}}\\
				\cmidrule(lr){5-10}
				&&&                                            & NTGA& EMN & TAP & REMN  & SHR & AR\\
				\midrule
                \multirow{8}{*}{\rotatebox[origin=c]{90}{CIFAR-10}}
				&\multirow{2}{*}{\rotatebox[origin=c]{90}{\footnotesize RN-18}}
				& Vanilla   &\multirow{2}{*}{$94.50 \pm 0.09$} & $11.49 \pm 0.69$  & $24.85 \pm 0.71$ & $7.86 \pm 0.90$  & $20.50 \pm 1.16$ & $10.82 \pm 0.22$ & $12.09 \pm 1.12$\\
				&& \textsc{Avatar}                            && $87.95 \pm 0.28$  & $90.95 \pm 0.10$ & $90.71 \pm 0.19$ & $88.49 \pm 0.24$ & $85.69 \pm 0.27$ & $91.57 \pm 0.18$\\
    			\cmidrule(lr){2-10}
				&\multirow{2}{*}{\rotatebox[origin=c]{90}{\footnotesize VGG-16}}
				& Vanilla   &\multirow{2}{*}{$93.22 \pm 0.05$} & $11.02 \pm 0.15$  & $25.82 \pm 0.74$ & $9.42 \pm 1.03$   & $20.58 \pm 0.96$ & $10.88 \pm 0.28$ & $11.49 \pm 1.23$\\
				&& \textsc{Avatar}                            && $87.13 \pm 0.27$  & $89.71 \pm 0.16$ & $89.38 \pm 0.17$  & $87.17 \pm 0.22$ & $84.74 \pm 0.30$ & $90.21 \pm 0.15$\\
    			\cmidrule(lr){2-10}
				&\multirow{2}{*}{\rotatebox[origin=c]{90}{\footnotesize DN-121}}
				& Vanilla   &\multirow{2}{*}{$94.62 \pm 0.14$} & $11.63 \pm 1.08$  & $24.73 \pm 0.81$ & $11.09 \pm 1.69$  & $19.02 \pm 1.21$ & $14.57 \pm 1.36$ & $13.67 \pm 1.91$\\
				&& \textsc{Avatar}                            && $88.33 \pm 0.35$  & $91.33 \pm 0.21$ & $91.24 \pm 0.16$  & $88.83 \pm 0.24$ & $86.12 \pm 0.29$ & $91.77 \pm 0.15$\\
    			\cmidrule(lr){2-10}
				&\multirow{2}{*}{\rotatebox[origin=c]{90}{\footnotesize WRN-34}}
				& Vanilla   &\multirow{2}{*}{$93.87 \pm 0.07$} & $10.72 \pm 0.37$  & $22.31 \pm 1.55$ & $8.59 \pm 0.43$   & $19.71 \pm 0.95$ & $10.33 \pm 0.14$ & $12.09 \pm 0.43$\\
				&& \textsc{Avatar}                            && $87.64 \pm 0.19$  & $90.19 \pm 0.21$ & $89.64 \pm 0.24$  & $87.64 \pm 0.28$ & $85.22 \pm 0.42$ & $90.71 \pm 0.14$\\
                \midrule
                \multirow{8}{*}{\rotatebox[origin=c]{90}{SVHN}}
				&\multirow{2}{*}{\rotatebox[origin=c]{90}{\footnotesize RN-18}}
				& Vanilla   &\multirow{2}{*}{$96.29 \pm 0.12$} & $9.65 \pm 0.70$   & $9.13 \pm 2.00$  & $65.97 \pm 1.99$  & $11.55 \pm 0.19$ & $10.59 \pm 3.98$ & $6.76 \pm 0.07$\\
				&& \textsc{Avatar}                            && $89.84 \pm 0.32$  & $93.84 \pm 0.12$ & $93.35 \pm 0.10$  & $88.51 \pm 0.23$ & $83.82 \pm 0.39$ & $94.13 \pm 0.17$\\
    			\cmidrule(lr){2-10}
				&\multirow{2}{*}{\rotatebox[origin=c]{90}{\footnotesize VGG-16}}
				& Vanilla   &\multirow{2}{*}{$95.94 \pm 0.12$} & $23.24 \pm 8.42$  & $9.13 \pm 2.00$  & $63.81 \pm2.77$   & $10.87 \pm 0.43$ & $9.45 \pm 3.69$  & $10.87 \pm 5.09$\\
				&& \textsc{Avatar}                            && $89.75 \pm 0.19$  & $93.61 \pm 0.14$ & $93.03 \pm 0.26$  & $87.01 \pm 0.41$ & $82.03 \pm 0.48$ & $93.73 \pm 0.14$\\
    			\cmidrule(lr){2-10}
				&\multirow{2}{*}{\rotatebox[origin=c]{90}{\footnotesize DN-121}}
				& Vanilla   &\multirow{2}{*}{$96.47 \pm 0.12$} & $19.06 \pm 5.84$  & $30.49 \pm 5.53$ & $69.04 \pm 1.80$  & $11.48 \pm 2.09$ & $10.54 \pm 3.45$ & $10.23 \pm 3.64$\\
				&& \textsc{Avatar}                            && $90.52 \pm 0.13$  & $94.39 \pm 0.24$ & $94.05 \pm 0.18$  & $88.76 \pm 0.53$ & $84.35 \pm 0.71$ & $94.61 \pm 0.10$\\
    			\cmidrule(lr){2-10}
				&\multirow{2}{*}{\rotatebox[origin=c]{90}{\footnotesize WRN-34}}
				& Vanilla   &\multirow{2}{*}{$96.35 \pm 0.07$} & $23.24 \pm 8.42$  & $18.57 \pm 7.00$ & $70.72 \pm 1.23$  & $18.10 \pm 10.17$ & $6.84 \pm 0.15$  & $9.04 \pm 2.47$\\
				&& \textsc{Avatar}                            && $90.26 \pm 0.16$  & $94.29 \pm 0.12$ & $94.01 \pm 0.28$  & $89.34 \pm 0.35$  & $83.50 \pm 0.47$ & $94.60 \pm 0.12$\\
                \midrule
                \multirow{8}{*}{\rotatebox[origin=c]{90}{CIFAR-100}}
				&\multirow{2}{*}{\rotatebox[origin=c]{90}{\footnotesize RN-18}}
				& Vanilla   &\multirow{2}{*}{$75.01 \pm 0.41$} & $1.32 \pm 0.31$   & $2.05 \pm 0.18$  & $14.10 \pm 0.19$  & $10.88 \pm 0.33$  & $1.39 \pm 0.10$  & $2.15 \pm 0.46$\\
				&& \textsc{Avatar}                            && $63.98 \pm 0.55$  & $65.73 \pm 0.36$ & $64.99 \pm 0.10$  & $64.88 \pm 0.08$  & $58.52 \pm 0.46$ & $64.54 \pm 0.23$\\
    			\cmidrule(lr){2-10}
				&\multirow{2}{*}{\rotatebox[origin=c]{90}{\footnotesize VGG-16}}
				& Vanilla   &\multirow{2}{*}{$72.03 \pm 0.06$} & $1.16 \pm 0.03$   & $1.94 \pm 0.19$  & $15.25 \pm 0.57$  & $9.11 \pm 0.54$   & $1.57 \pm 0.33$  & $2.42 \pm 0.38$\\
				&& \textsc{Avatar}                            && $62.00 \pm 0.34$  & $63.60 \pm 0.26$ & $62.69 \pm 0.21$  & $62.38 \pm 0.30$  & $56.38 \pm 0.41$ & $62.27 \pm 0.34$\\
    			\cmidrule(lr){2-10}
				&\multirow{2}{*}{\rotatebox[origin=c]{90}{\footnotesize DN-121}}
				& Vanilla   &\multirow{2}{*}{$77.47 \pm 0.33$} & $1.87 \pm 0.34$   & $2.41 \pm 0.40$  & $15.94 \pm 0.25$  & $8.94 \pm 0.49$   & $1.81 \pm 0.23$  & $2.34 \pm 0.48$\\
				&& \textsc{Avatar}                            && $65.84 \pm 0.41$  & $67.86 \pm 0.22$ & $67.27 \pm 0.18$  & $66.85 \pm 0.20$  & $60.16 \pm 0.24$ & $66.78 \pm 0.43$\\
    			\cmidrule(lr){2-10}
				&\multirow{2}{*}{\rotatebox[origin=c]{90}{\footnotesize WRN-34}}
				& Vanilla   &\multirow{2}{*}{$73.39 \pm 0.43$} & $1.51 \pm 0.18$   & $1.94 \pm 0.24$  & $12.00 \pm 0.45$  & $9.16 \pm 0.61$   & $1.37 \pm 0.19$  & $1.90 \pm 0.24$\\
				&& \textsc{Avatar}                            && $61.62 \pm 0.36$  & $63.49 \pm 0.32$ & $62.68 \pm 0.34$  & $62.52 \pm 0.32$  & $56.64 \pm 0.60$ & $62.57 \pm 0.26$\\
                \midrule
                \multirow{8}{*}{\rotatebox[origin=c]{90}{ImageNet-100}}
				&\multirow{2}{*}{\rotatebox[origin=c]{90}{\footnotesize RN-18}}
				& Vanilla   &\multirow{2}{*}{$80.05 \pm 0.13$} & $74.74 \pm 0.52$  & $1.78 \pm 0.17$  & $9.14 \pm 0.40$   & $13.28 \pm 0.51$  & $43.48 \pm 1.56$\\
				&& \textsc{Avatar}                            && $71.08 \pm 0.48$  & $72.84 \pm 0.90$ & $76.52 \pm 0.46$  & $39.79 \pm 0.98$  & $59.85 \pm 1.01$\\
    			\cmidrule(lr){2-9}
				&\multirow{2}{*}{\rotatebox[origin=c]{90}{\footnotesize VGG-16}}
				& Vanilla   &\multirow{2}{*}{$79.16 \pm 0.30$} & $73.98 \pm 0.45$ & $1.54 \pm 0.24$   & $9.30 \pm 0.39$   & $12.04 \pm 0.15$  & $72.10 \pm 0.68$\\
				&& \textsc{Avatar}                            && $70.78 \pm 0.56$ & $73.34 \pm 0.50$  & $76.25 \pm 0.19$  & $39.74 \pm 0.44$  & $67.22 \pm 0.54$\\
    			\cmidrule(lr){2-9}
				&\multirow{2}{*}{\rotatebox[origin=c]{90}{\footnotesize DN-121}}
				& Vanilla   &\multirow{2}{*}{$79.66 \pm 0.25$} & $75.16 \pm 0.22$ & $4.58 \pm 0.31$   & $19.28 \pm 0.84$  & $15.46 \pm 0.77$  & $31.47 \pm 1.01$\\
				&& \textsc{Avatar}                            && $73.41 \pm 0.23$ & $75.96 \pm 0.48$  & $77.96 \pm 0.48$  & $48.20 \pm 0.32$  & $54.24 \pm 0.65$\\
    			\cmidrule(lr){2-9}
				&\multirow{2}{*}{\rotatebox[origin=c]{90}{\footnotesize WRN-34}}
				& Vanilla   &\multirow{2}{*}{$74.46 \pm 0.52$} & $68.38 \pm 1.17$ & $1.22 \pm 0.15$   & $8.20 \pm 0.28$   & $10.45 \pm 0.28$  & $52.95 \pm 4.13$\\
				&& \textsc{Avatar}                            && $64.20 \pm 0.76$ & $66.33 \pm 1.01$  & $70.34 \pm 0.85$  & $29.89 \pm 0.66$  & $58.24 \pm 1.47$\\
			    \bottomrule
			\end{tabular}
		\end{small}
	\end{center}
\end{table*}

\begin{minipage}{\textwidth}
\vspace{1em}
  \begin{minipage}[tb!]{0.40\textwidth}
    \centering
    \includegraphics[width=1.00\textwidth]{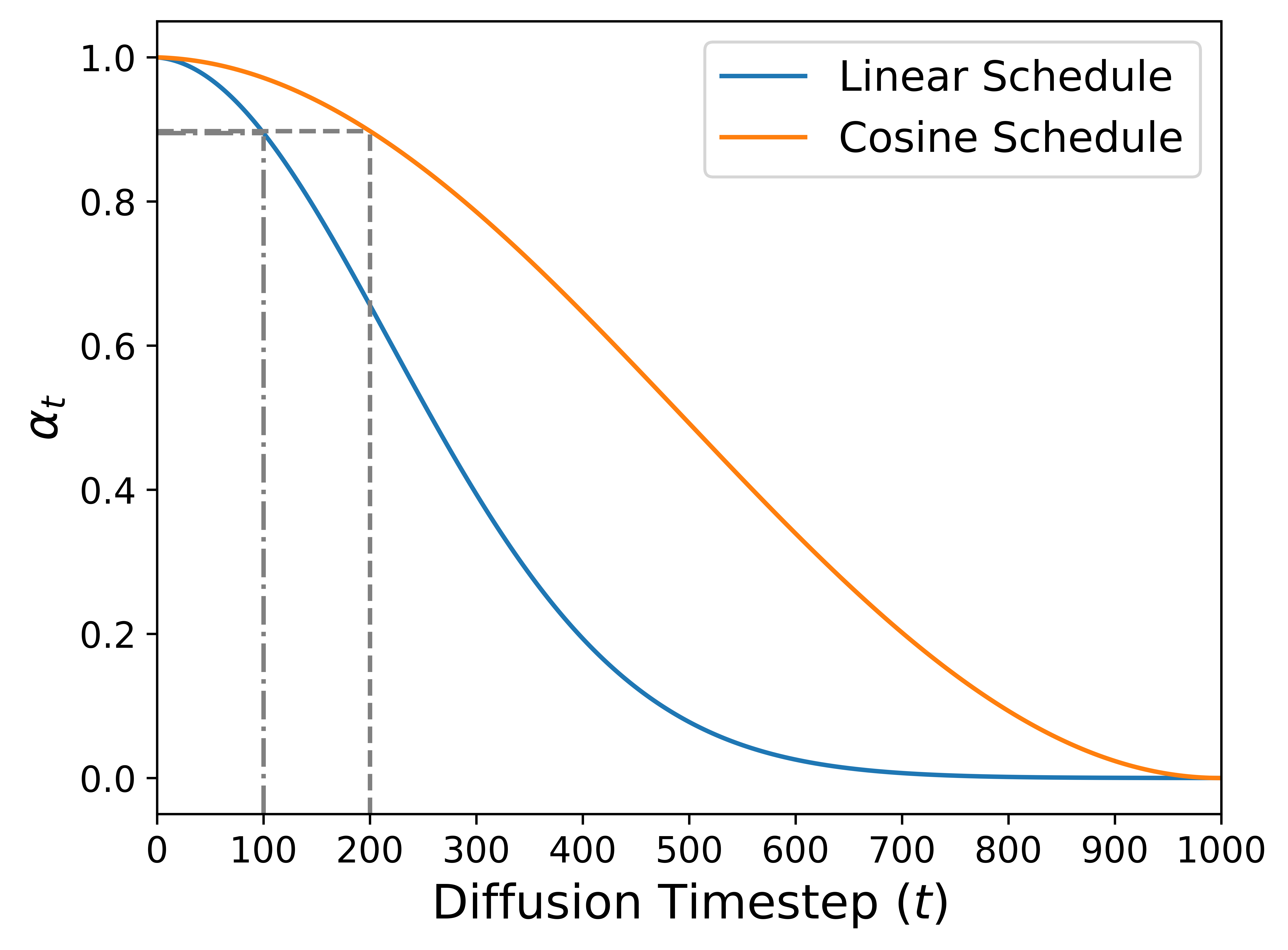}
    \captionof{figure}{The $\alpha_t$ curve for linear vs.~cosine sampling schedulers.\label{fig:cosine_vs_linear}}
  \end{minipage}
  \hspace{3em}
  \begin{minipage}[tb!]{0.49\textwidth}
    \centering
    \captionof{table}{Reconstruction PSNR for clean CIFAR-10 dataset. The mean and standard deviation are computed over 1000 random samples as the validation set. \label{tab:psnr_quality}}
    \setlength\tabcolsep{0.45em}
	\def\arraystretch{1.65}
     \begin{center}
    	\begin{footnotesize}
            \begin{tabular}{lcccc}
            \toprule
            \multirow{2}{*}{\textbf{Scheduler}} &\multicolumn{4}{c}{\textbf{Diffusion Timestep ($t$)}}  \\
            \cmidrule{2-5}
                           & $50$               & $100$             & $150$              & $200$    \\
            \midrule
            Linear         & $25.19 \pm 4.16$   & $22.52 \pm 3.32$  & $20.83 \pm 2.92$   & $19.54 \pm 2.61$\\
            Cosine         & $31.20 \pm 1.50$   & $27.49 \pm 1.51$  & $25.27 \pm 1.50$   & $23.71 \pm 1.50$\\
            \bottomrule
            \end{tabular}
	   \end{footnotesize}
    \end{center}
    \end{minipage}
\end{minipage}

\subsection{Additional Experimental Results over Different Combination of Availability Attacks}

A scenario that might happen in the real-world is that different classes use a different type of protection.
To simulate this scenario, we choose five of the best performing availability attacks, namely CON (C), NTGA (N), TAP (T), REMN (R), and SHR (S), based on our results in~\Cref{tab:architecture} to protect four classes of the CIFAR-10 dataset.
We create different combinations of these five attacks to protect the four classes, resulting in five distinct combinations which we name CNTR, NTRS, RSCN, SCNT, and TRSC.
We use \textsc{Avatar} to defuse the entire dataset, which includes both protected and unprotected classes.
To this end, we use our setting from \Cref{sec:sec:mismatch} and use DDPM-IP models pre-trained over the IN-1k-32$\times$32 dataset.
We then train RN-18 models over protected and defused data.
Our results have been reported as confusion matrices in~\Cref{fig:combination_attacks}.
As seen, our model is attack-agnostic and can revive the normal data.

\subsection{Additional Experimental Results over Different Perturbation Norms}
Another interesting use-case might happen when different classes use a different perturbation norm to protect their data.
We designed an experiment on CIFAR-10 to test this case.
For these experiments, we first choose four classes of the CIFAR-10 randomly and aim to protect them with availability attacks.
We then use four distinct levels of protection, from $\varepsilon=4$ to $\varepsilon=32$, to protect these selected classes.
\Cref{fig:combination_attacks_samples} shows a few samples for each of the availability attacks used in this scenario.
Like the previous experiment, again we use our settings from \Cref{sec:sec:mismatch} to run these experiments.
As seen in the confusion matrices of~\Cref{fig:epsilon_attacks}, \textsc{Avatar} performance decreases as we increase the perturbation norm.
This is in line with our theoretical insights: to protect the data against \textsc{Avatar}, we need larger perturbations. 
However, a larger perturbation means losing the regular utility of the data.

\begin{figure*}[htp!]
    \centering
    \includegraphics[width=0.6\textwidth]{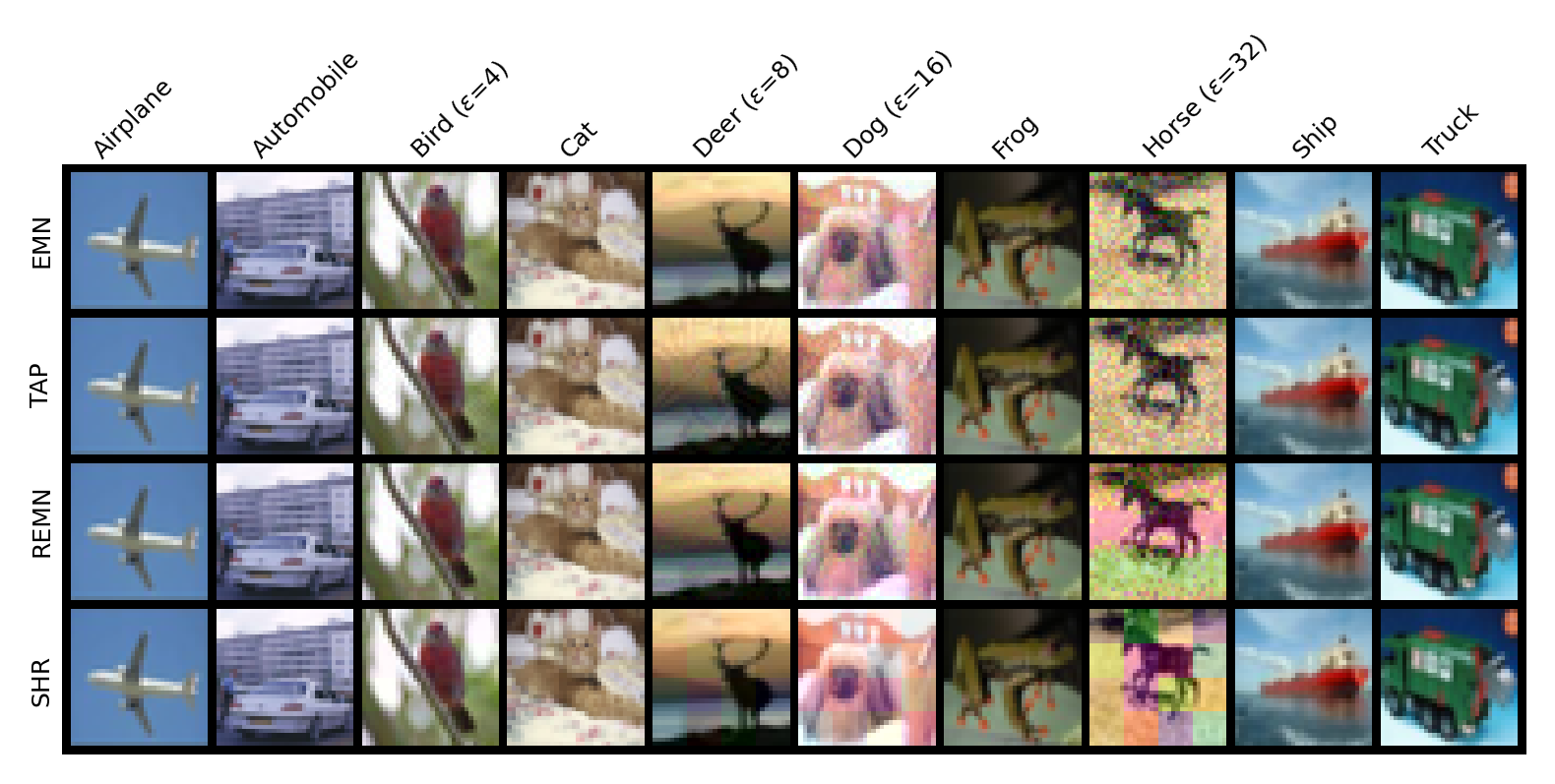}
    \caption{Samples from protected CIFAR-10 datasets with four different availability attacks. We have selected four classes to protect, where in each case we use a different perturbation norm to protect the data.
    The protecting perturbations become extremely visible as we increase their norm.}
    \label{fig:combination_attacks_samples}
    \vspace{-2em}
\end{figure*}

\begin{figure*}[p!]
    \begin{subfigure}{1.0\textwidth}
    \centering
        \begin{subfigure}{.40\textwidth}
    		\centering
    		\includegraphics[width=1.0\textwidth]{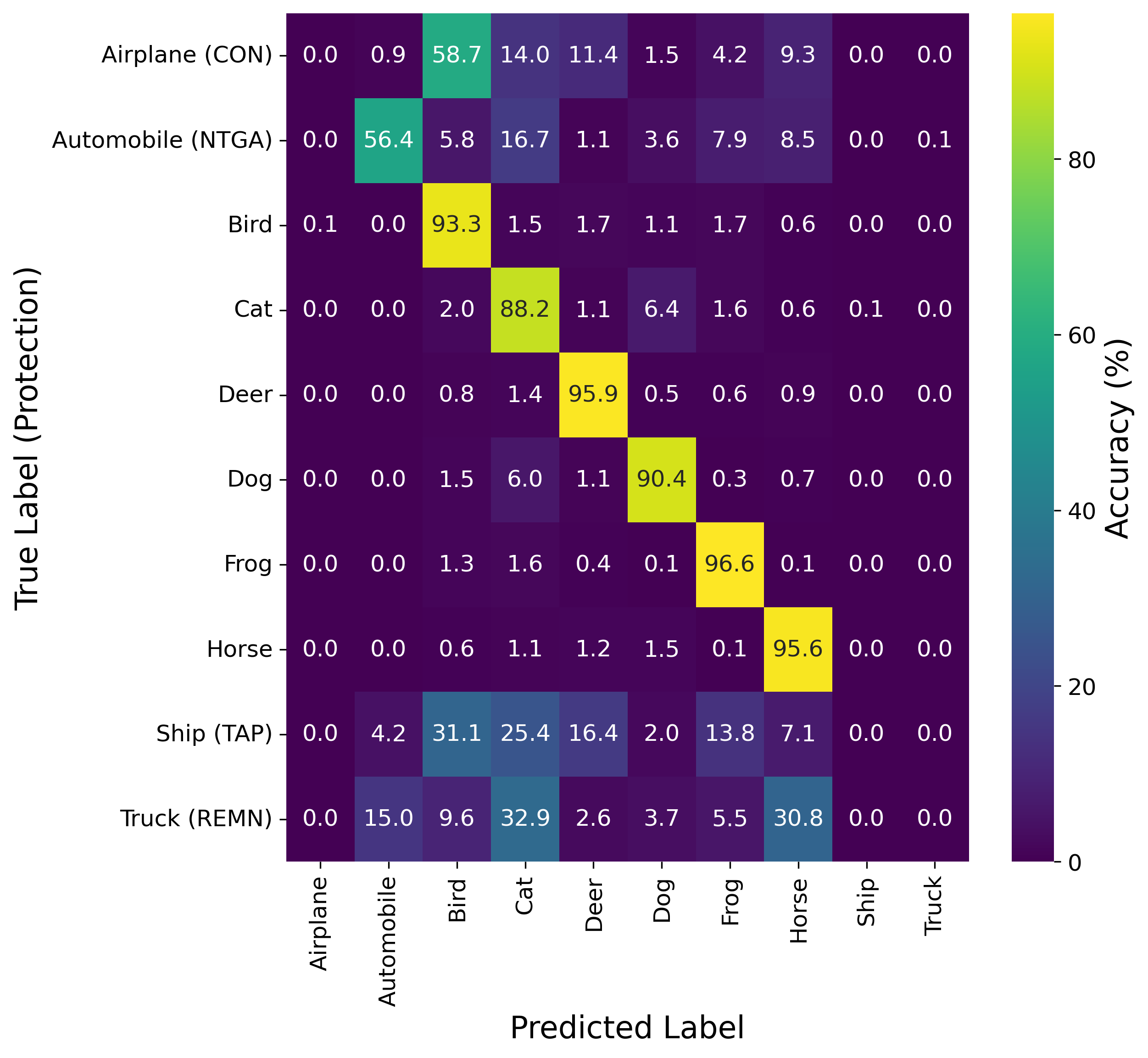}
            \caption*{Vanilla Training}
    	\end{subfigure}
     \hspace{5em}
        \begin{subfigure}{.40\textwidth}
    		\centering
    		\includegraphics[width=1.0\textwidth]{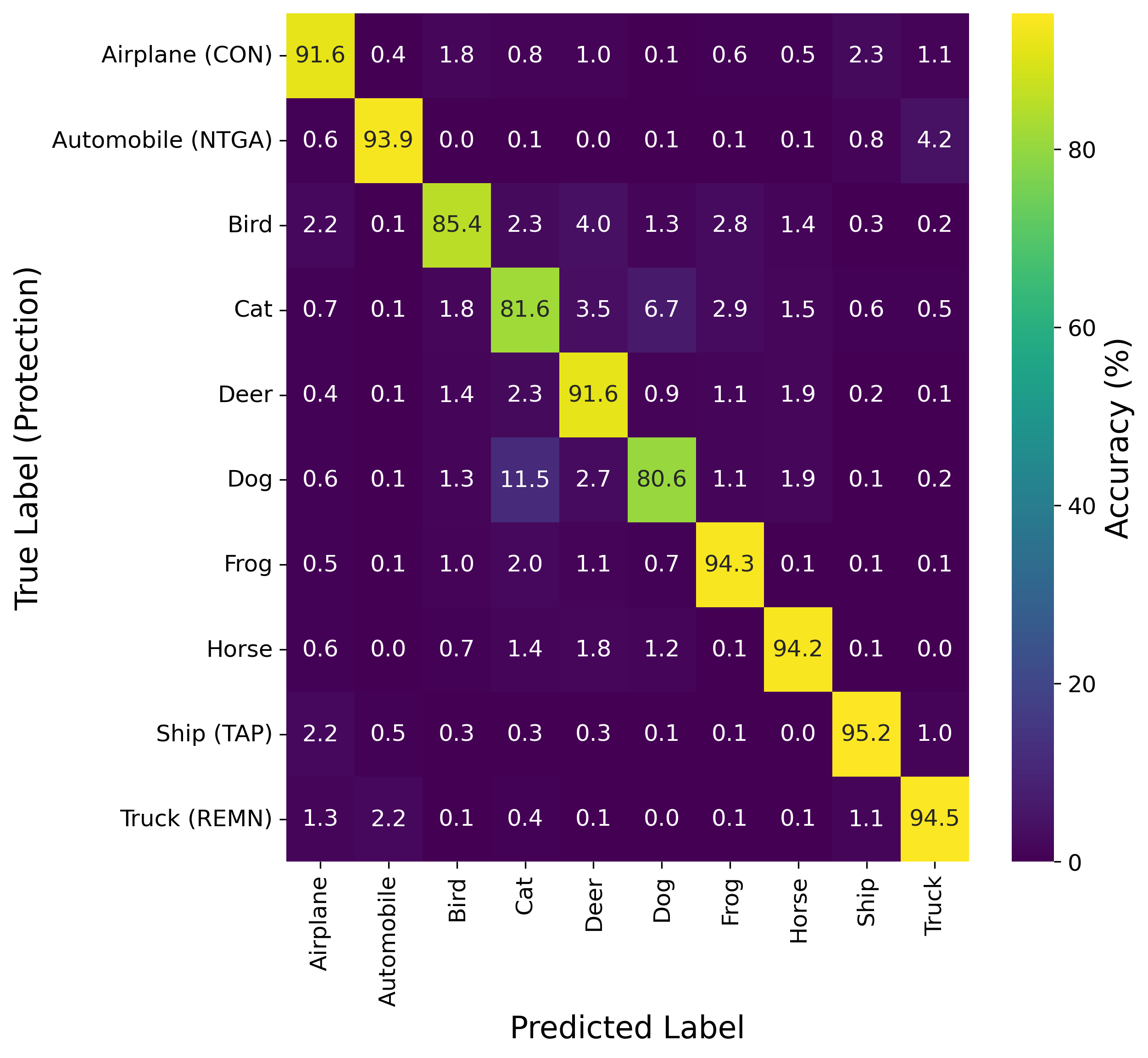}
            \caption*{\textsc{Avatar}}
    	\end{subfigure}
	\caption{CNTR}
    \vspace*{2em}
    \end{subfigure}
    \\
    \begin{subfigure}{1.0\textwidth}
    \centering
        \begin{subfigure}{.40\textwidth}
    		\centering
    		\includegraphics[width=1.0\textwidth]{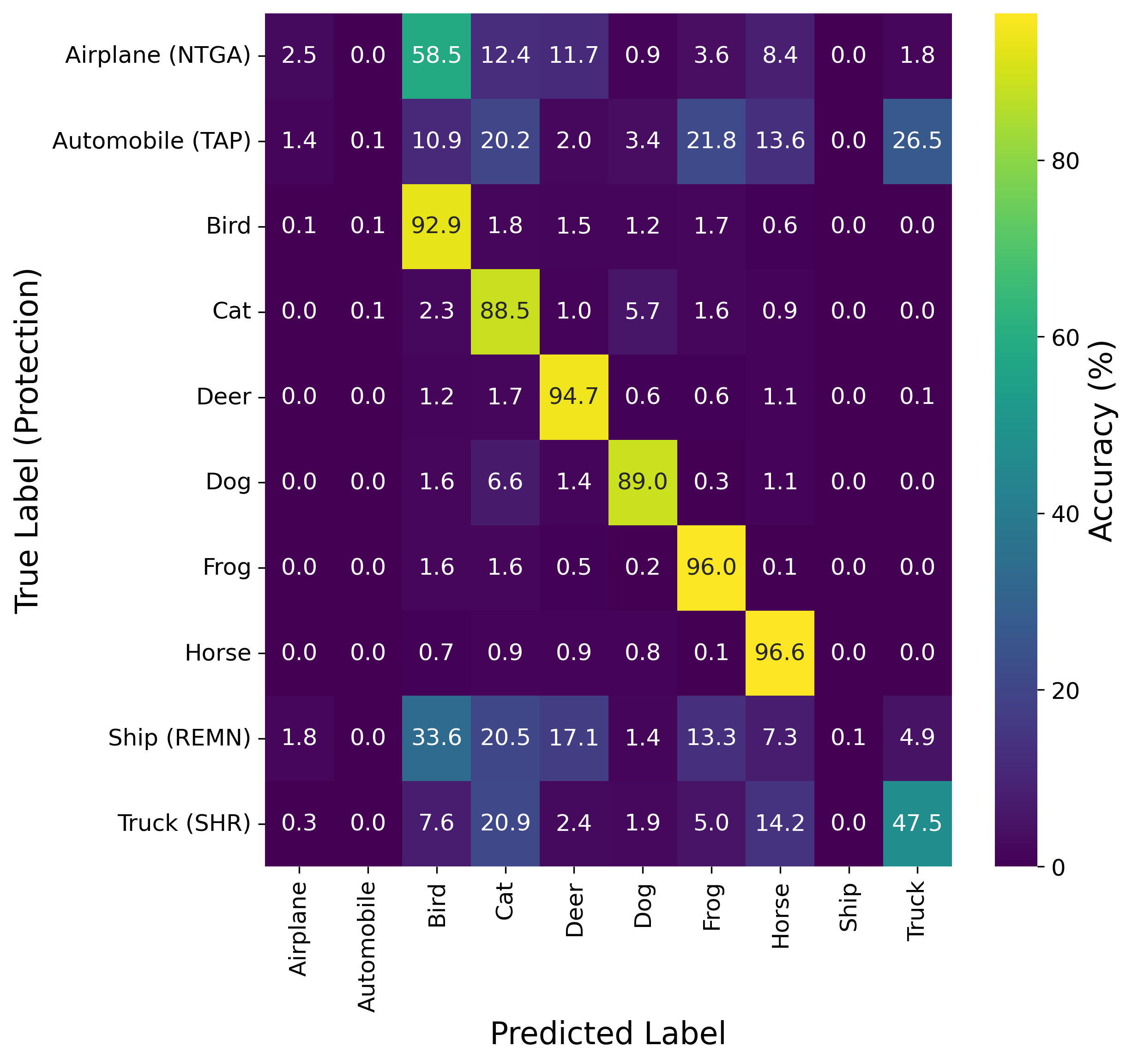}
            \caption*{Vanilla Training}
    	\end{subfigure}
     \hspace{5em}
        \begin{subfigure}{.40\textwidth}
    		\centering
    		\includegraphics[width=1.0\textwidth]{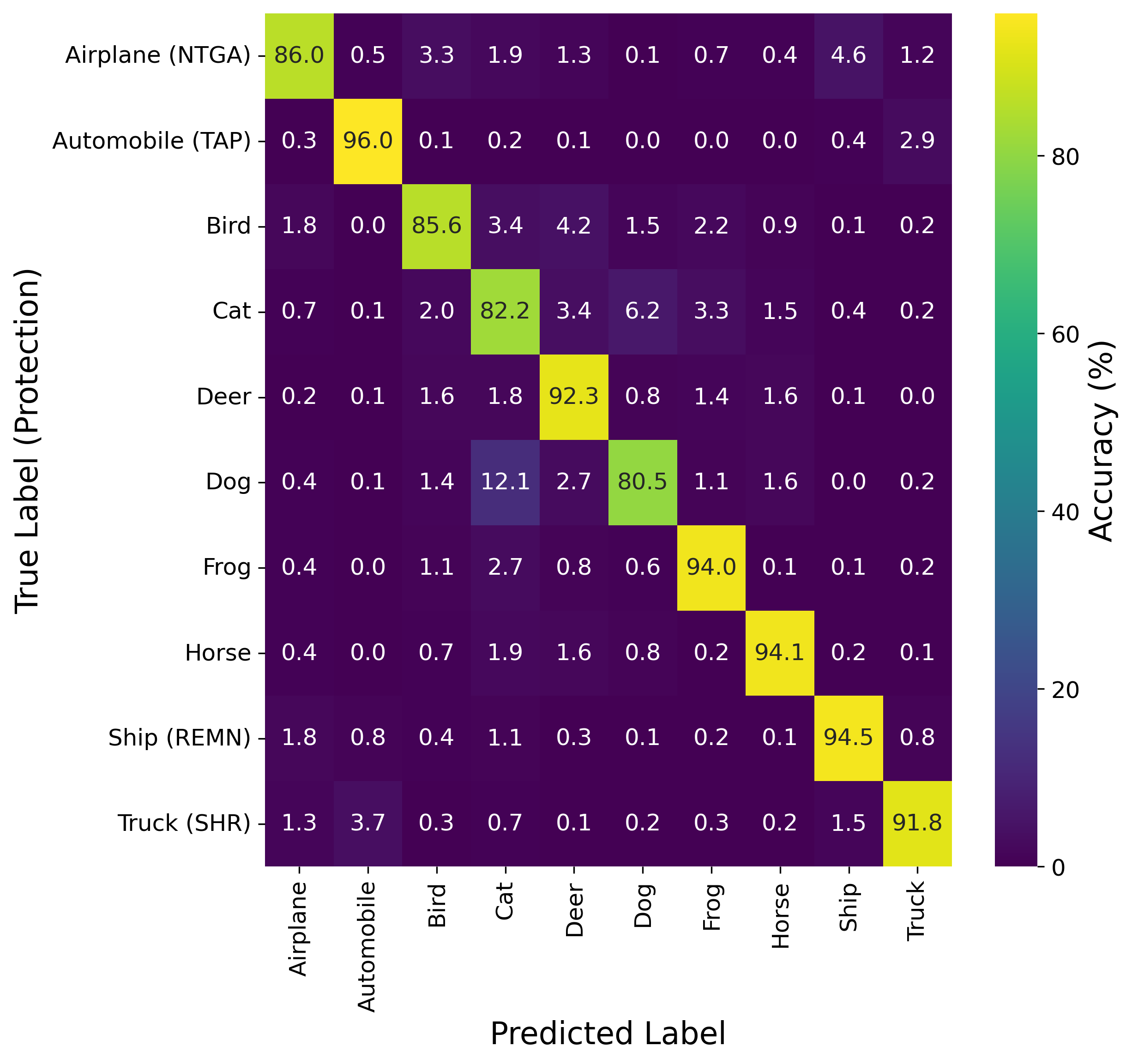}
            \caption*{\textsc{Avatar}}
    	\end{subfigure}
	\caption{NTRS}
    \vspace*{2em}
    \end{subfigure}
    \\
    \begin{subfigure}{1.0\textwidth}
    \centering
        \begin{subfigure}{.40\textwidth}
    		\centering
    		\includegraphics[width=1.0\textwidth]{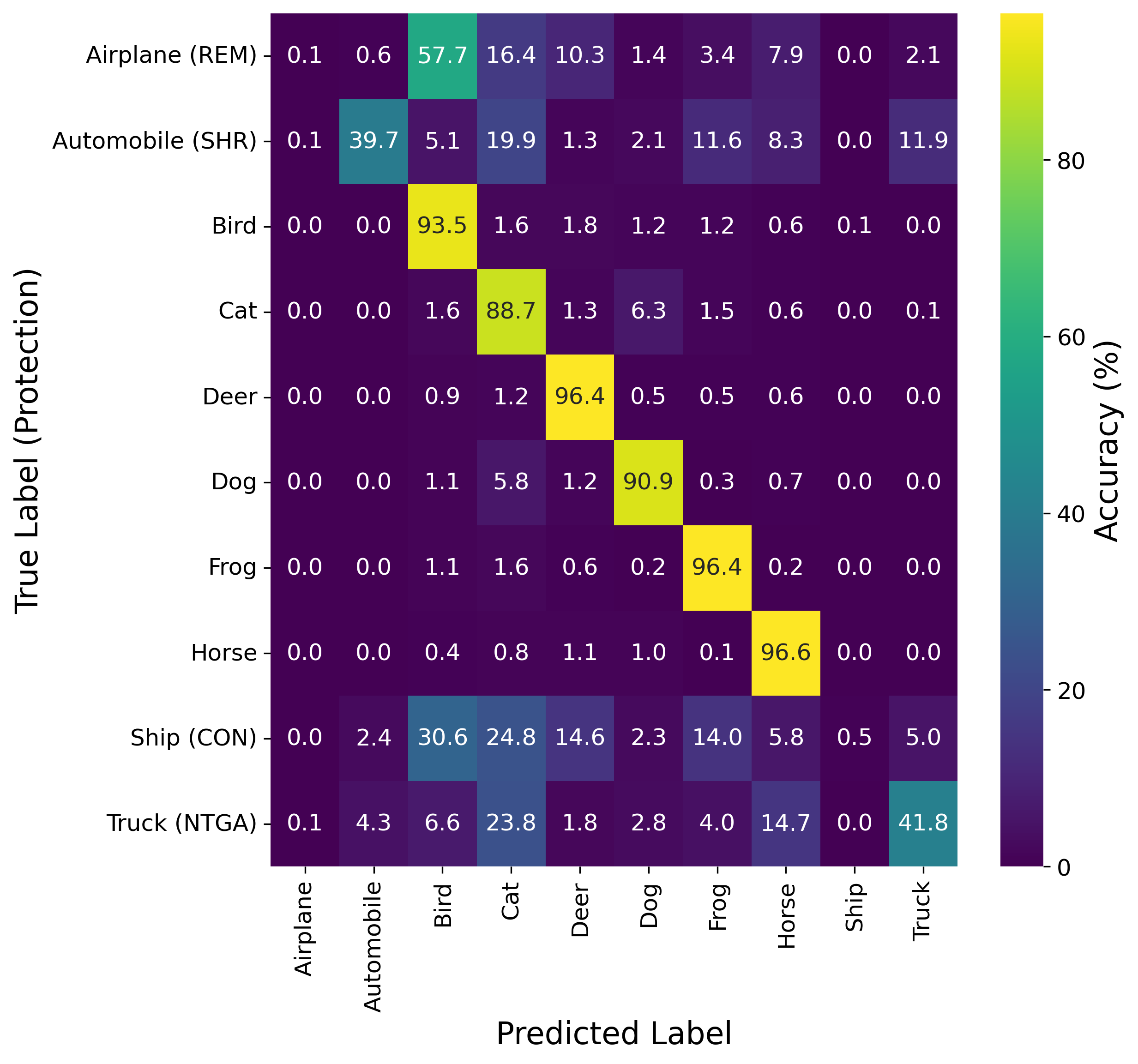}
            \caption*{Vanilla Training}
    	\end{subfigure}
     \hspace{5em}
        \begin{subfigure}{.40\textwidth}
    		\centering
    		\includegraphics[width=1.0\textwidth]{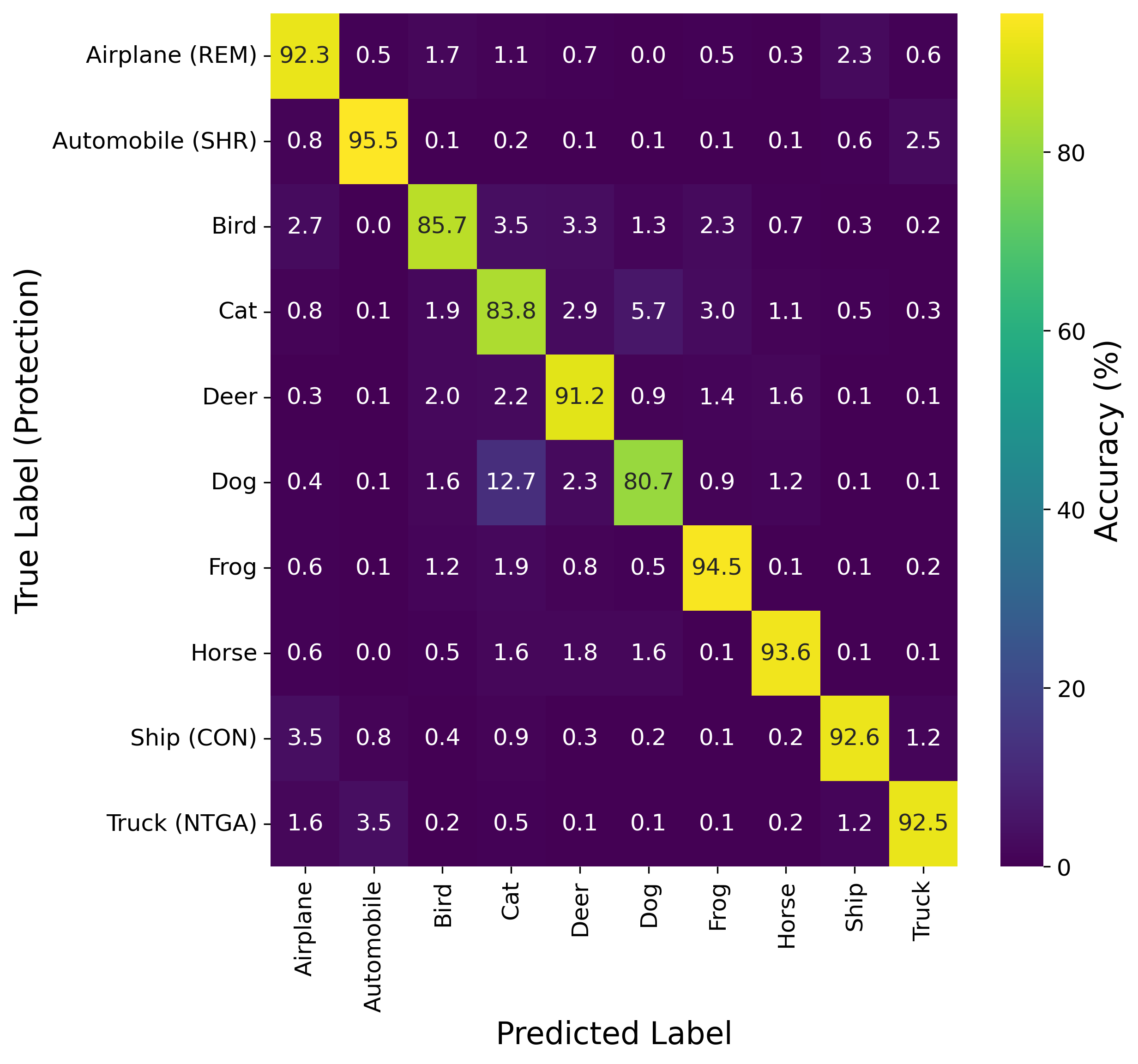}
            \caption*{\textsc{Avatar}}
    	\end{subfigure}
	\caption{RSCN}
    \end{subfigure}
\end{figure*}

\begin{figure*}[p!]
    \ContinuedFloat
    \begin{subfigure}{1.0\textwidth}
    \centering
        \begin{subfigure}{.40\textwidth}
    		\centering
    		\includegraphics[width=1.0\textwidth]{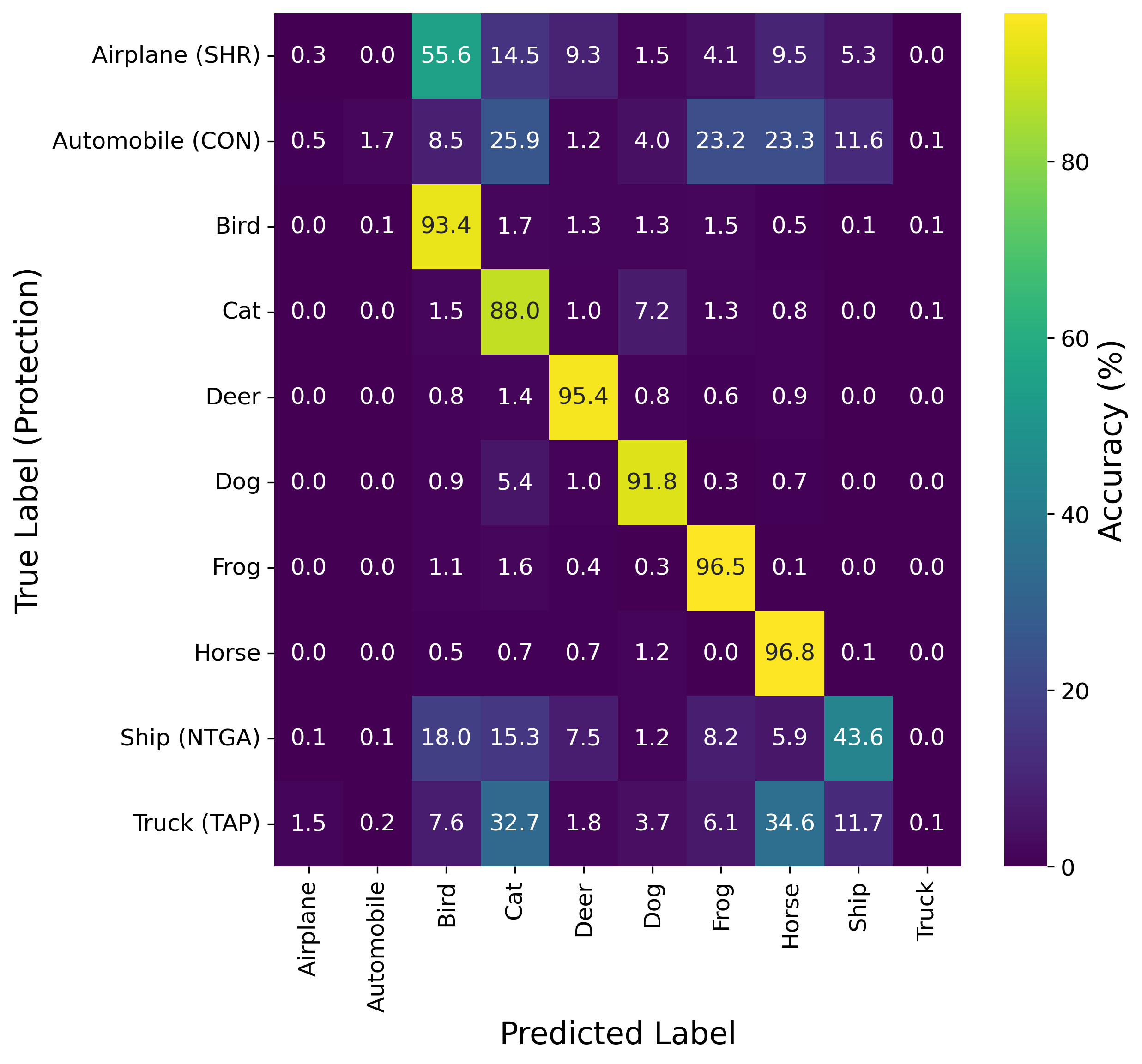}
            \caption*{Vanilla Training}
    	\end{subfigure}
     \hspace{5em}
        \begin{subfigure}{.40\textwidth}
    		\centering
    		\includegraphics[width=1.0\textwidth]{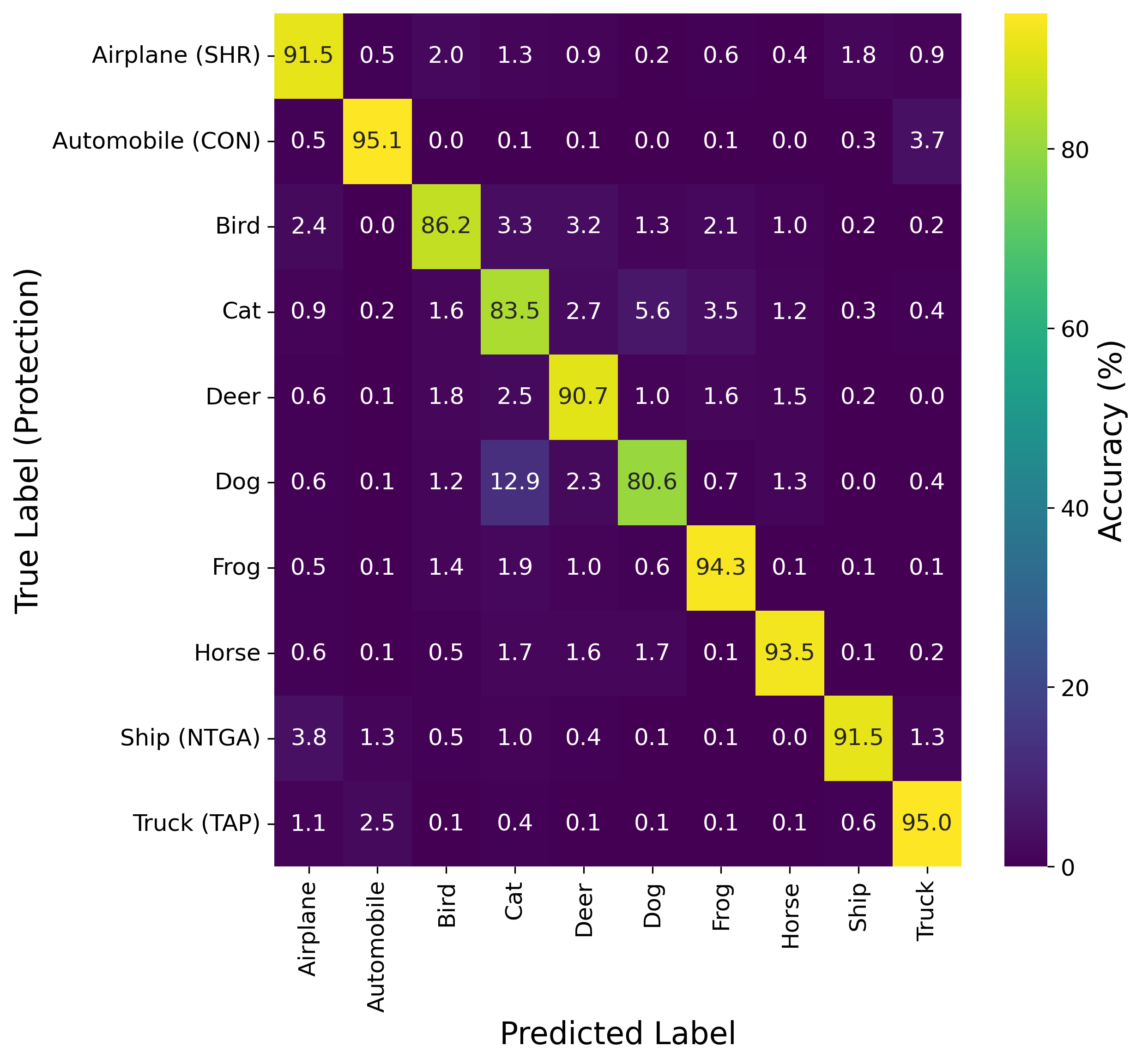}
            \caption*{\textsc{Avatar}}
    	\end{subfigure}
	\caption{SCNT}
    \vspace*{2em}
    \end{subfigure}
        \\
    \begin{subfigure}{1.0\textwidth}
    \centering
        \begin{subfigure}{.40\textwidth}
    		\centering
    		\includegraphics[width=1.0\textwidth]{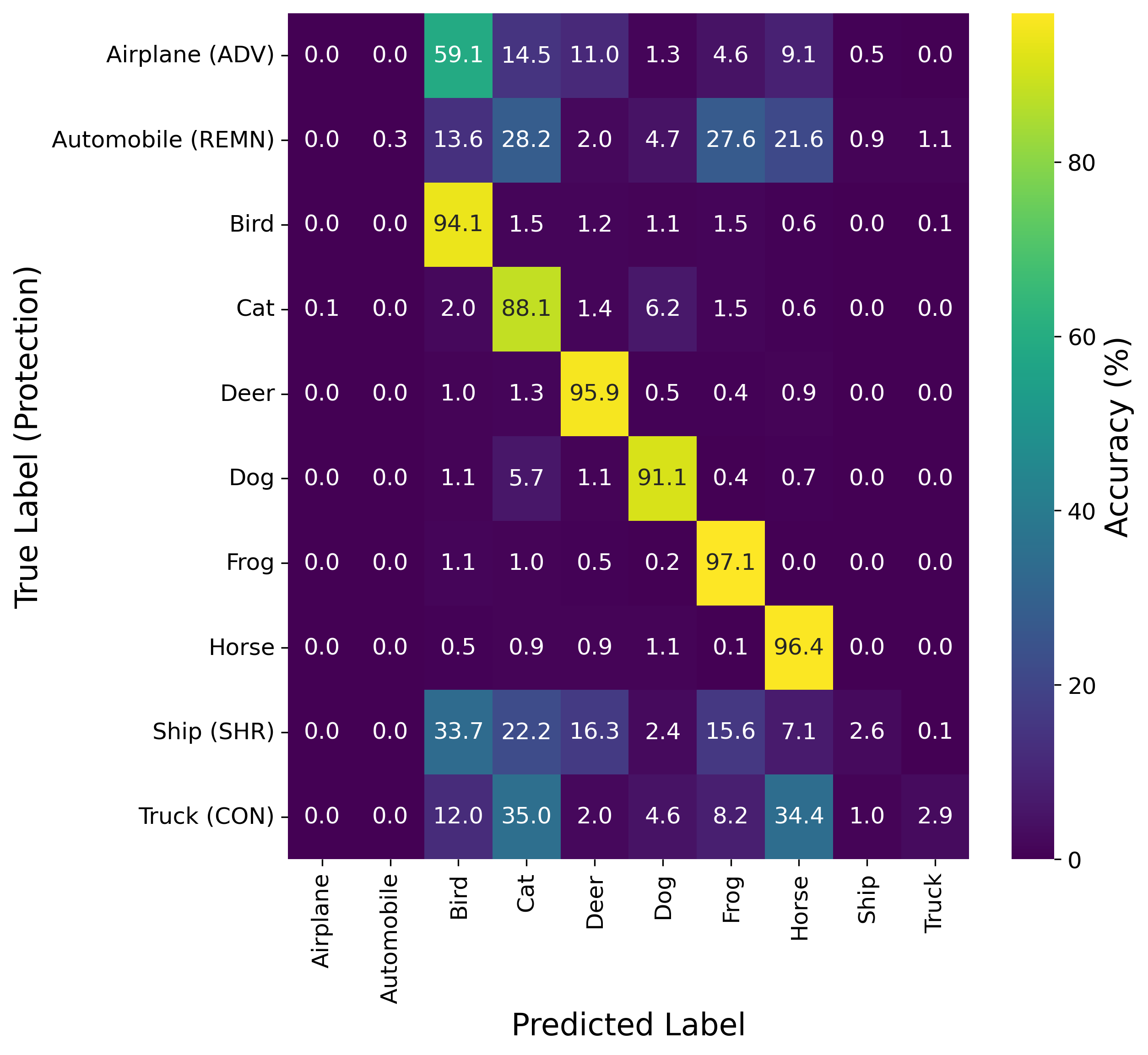}
            \caption*{Vanilla Training}
    	\end{subfigure}
     \hspace{5em}
        \begin{subfigure}{.40\textwidth}
    		\centering
    		\includegraphics[width=1.0\textwidth]{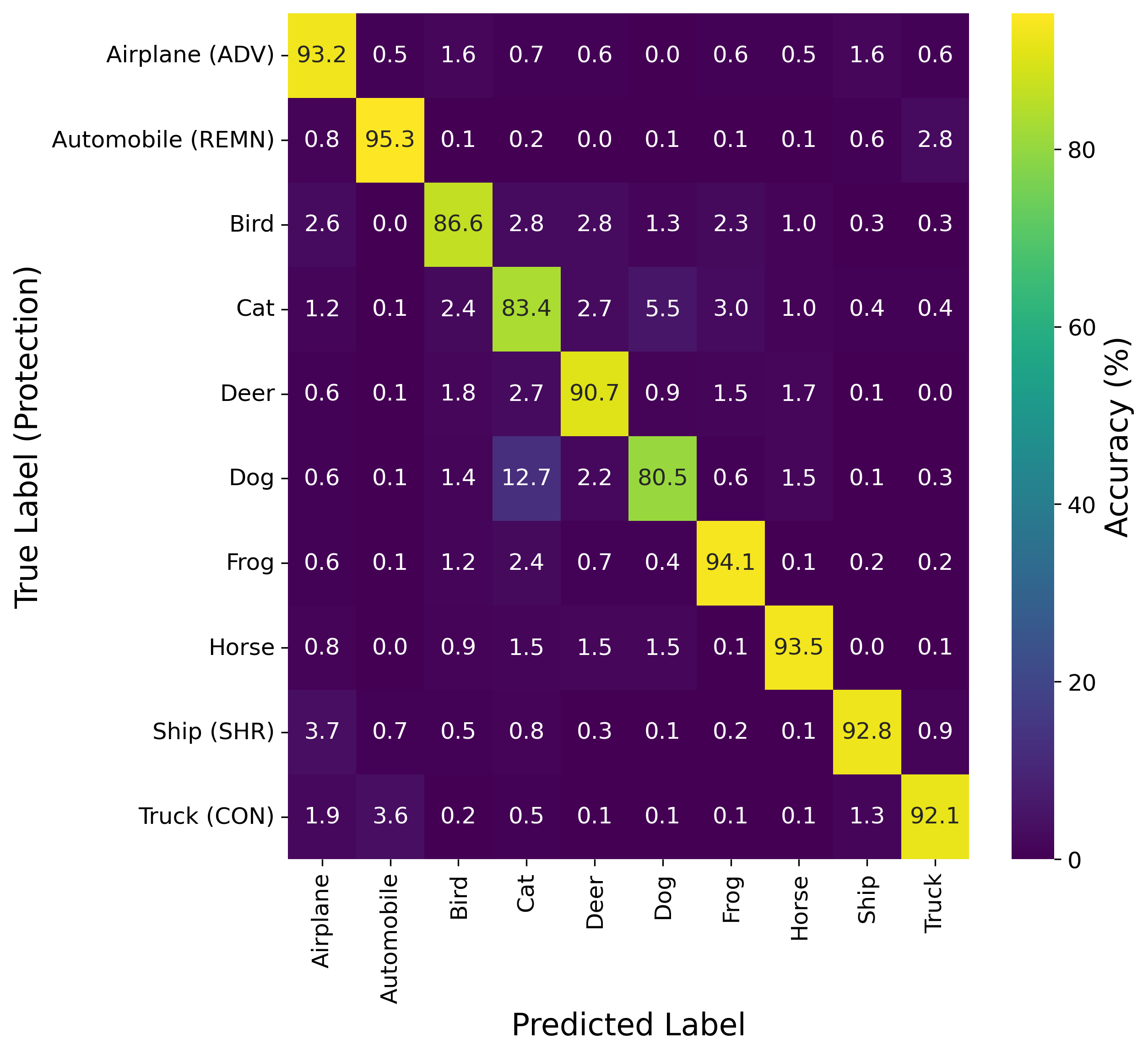}
            \caption*{\textsc{Avatar}}
    	\end{subfigure}
	\caption{TRSC}
    \end{subfigure}
    \caption{The confusion matrices of RN-18 classifiers trained over CIFAR-10 dataset. In each case, we use a different combination of availability attacks to protect some of the classes. For \textsc{Avatar}, we follow our settings for the experiments in~\Cref{tab:dist_mismatch_new} and use a DDPM-IP pre-trained over the IN-1k-32$\times$32 dataset.}
	\label{fig:combination_attacks}
\end{figure*}

\begin{figure*}[tb!]
    \begin{subfigure}{1.0\textwidth}
    \centering
        \begin{subfigure}{.40\textwidth}
    		\centering
    		\includegraphics[width=1.0\textwidth]{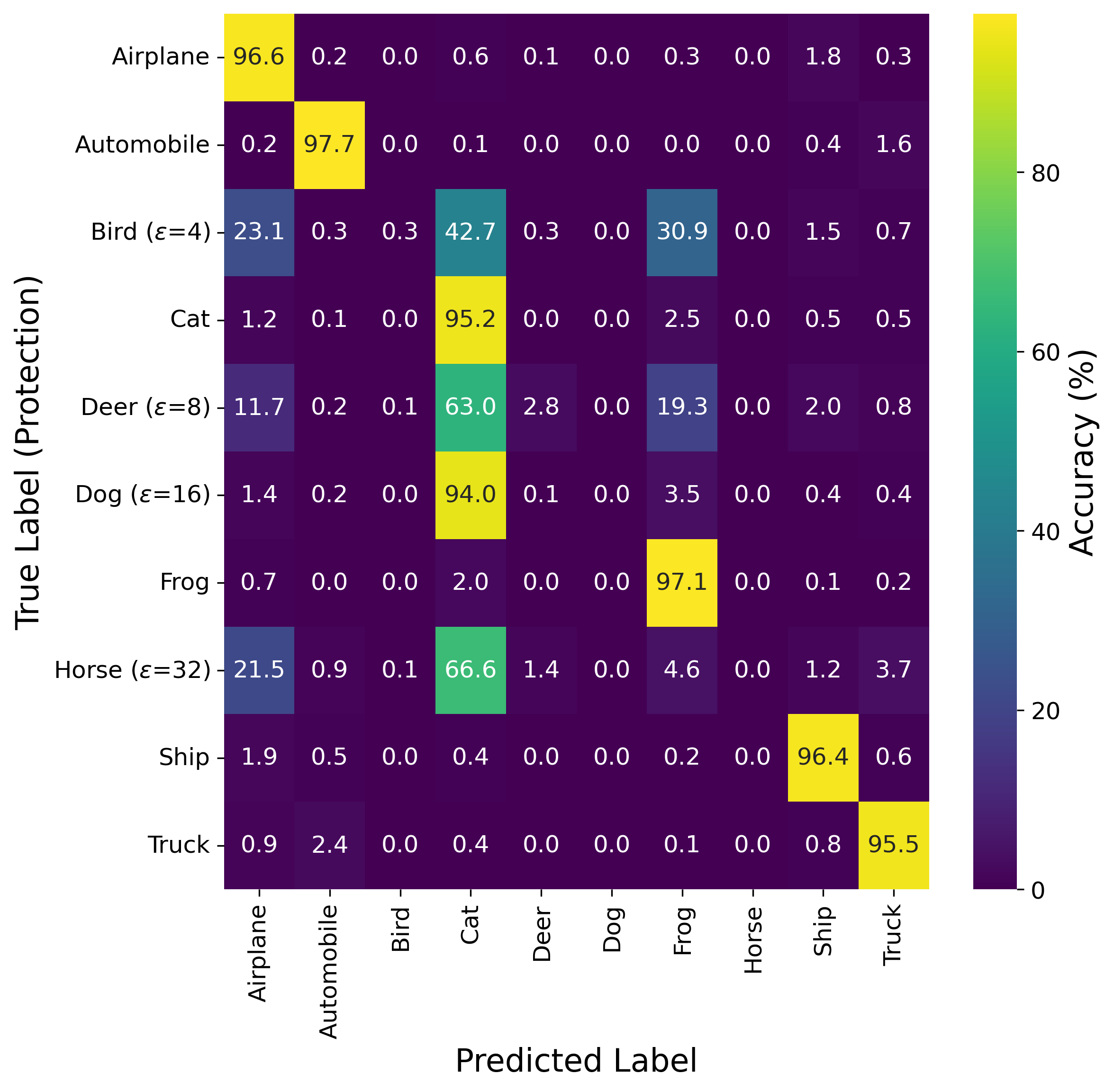}
            \caption*{Vanilla Training}
    	\end{subfigure}
     \hspace{5em}
        \begin{subfigure}{.40\textwidth}
    		\centering
    		\includegraphics[width=1.0\textwidth]{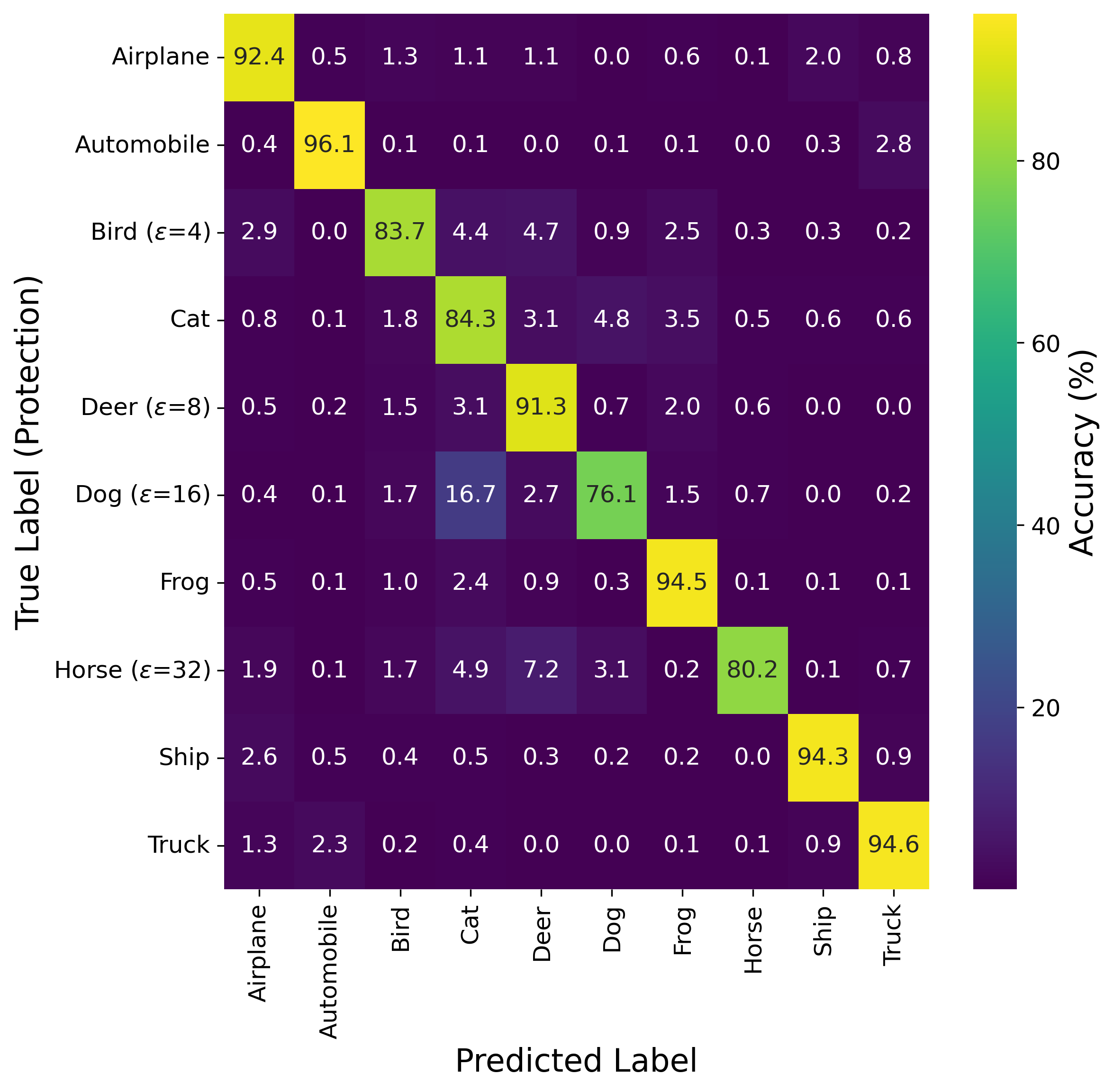}
            \caption*{\textsc{Avatar}}
    	\end{subfigure}
	\caption{EMN~\citep{huang2021emn}}
    \vspace*{2em}
    \end{subfigure}
    \\
    \begin{subfigure}{1.0\textwidth}
    \centering
        \begin{subfigure}{.40\textwidth}
    		\centering
    		\includegraphics[width=1.0\textwidth]{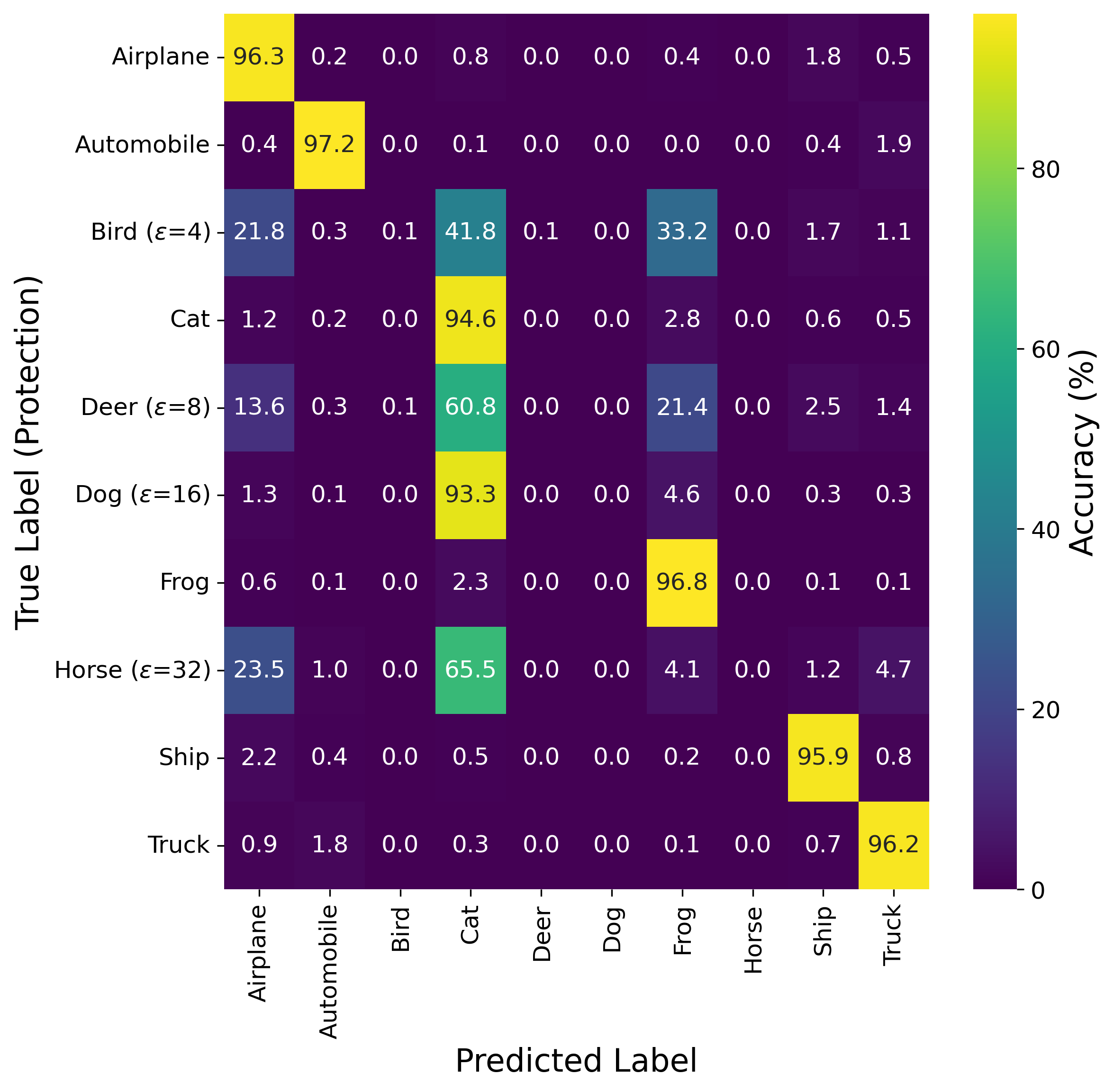}
            \caption*{Vanilla Training}
    	\end{subfigure}
     \hspace{5em}
        \begin{subfigure}{.40\textwidth}
    		\centering
    		\includegraphics[width=1.0\textwidth]{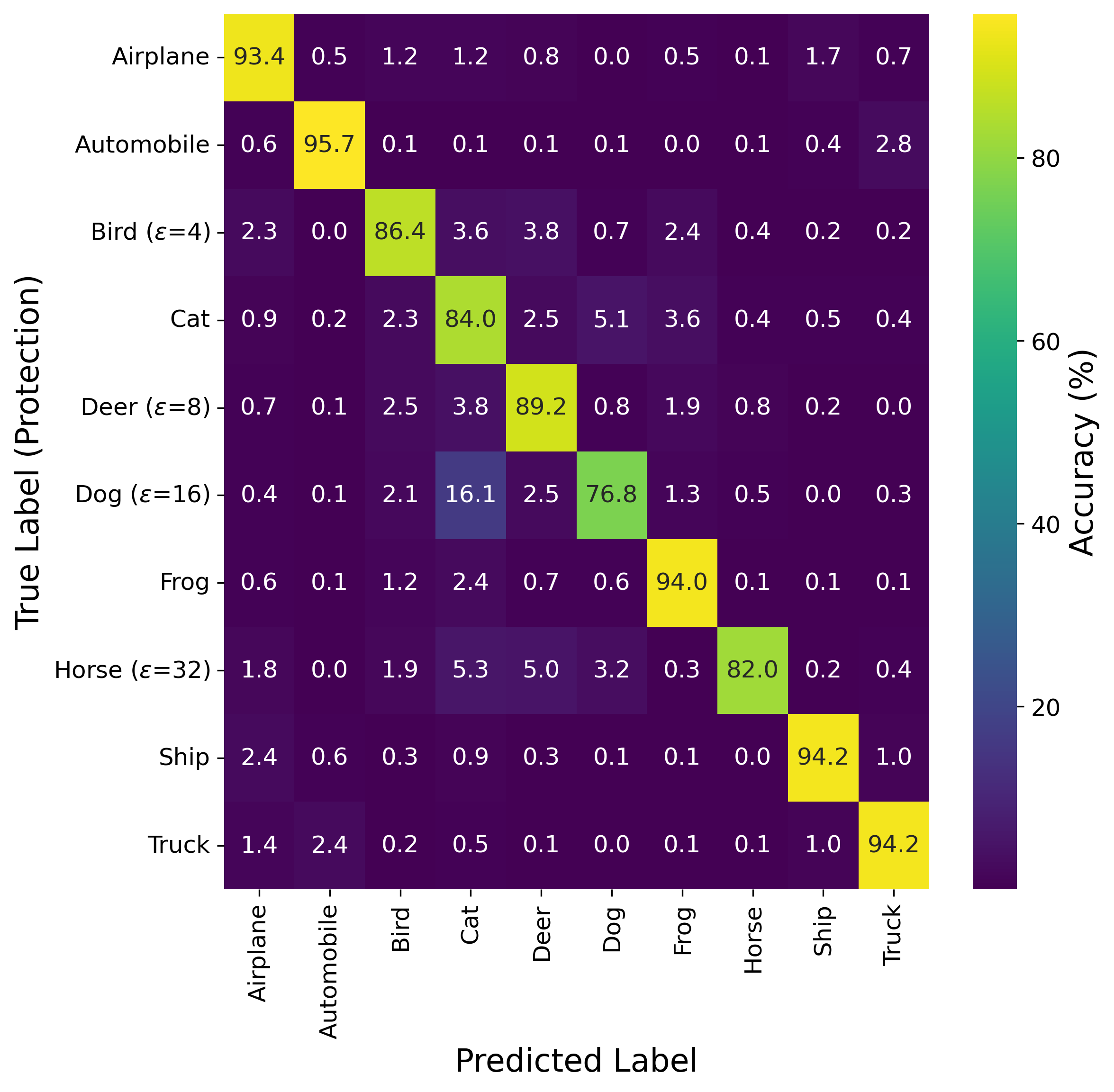}
            \caption*{\textsc{Avatar}}
    	\end{subfigure}
	\caption{TAP~\citep{fowl2021tap}}
    \end{subfigure}
\end{figure*}

\begin{figure*}[tb!]
    \ContinuedFloat
    \begin{subfigure}{1.0\textwidth}
    \centering
        \begin{subfigure}{.40\textwidth}
    		\centering
    		\includegraphics[width=1.0\textwidth]{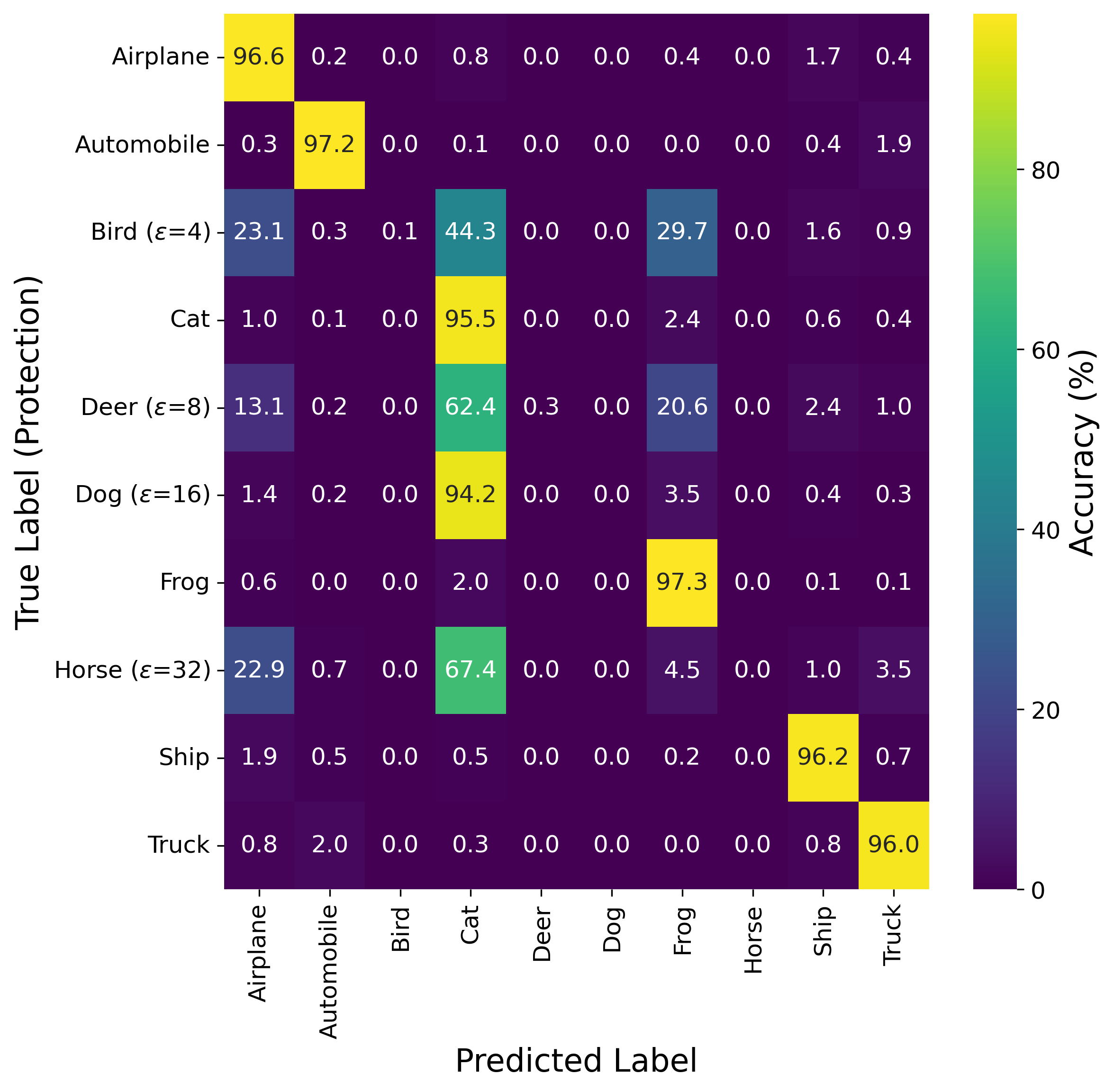}
            \caption*{Vanilla Training}
    	\end{subfigure}
     \hspace{5em}
        \begin{subfigure}{.40\textwidth}
    		\centering
    		\includegraphics[width=1.0\textwidth]{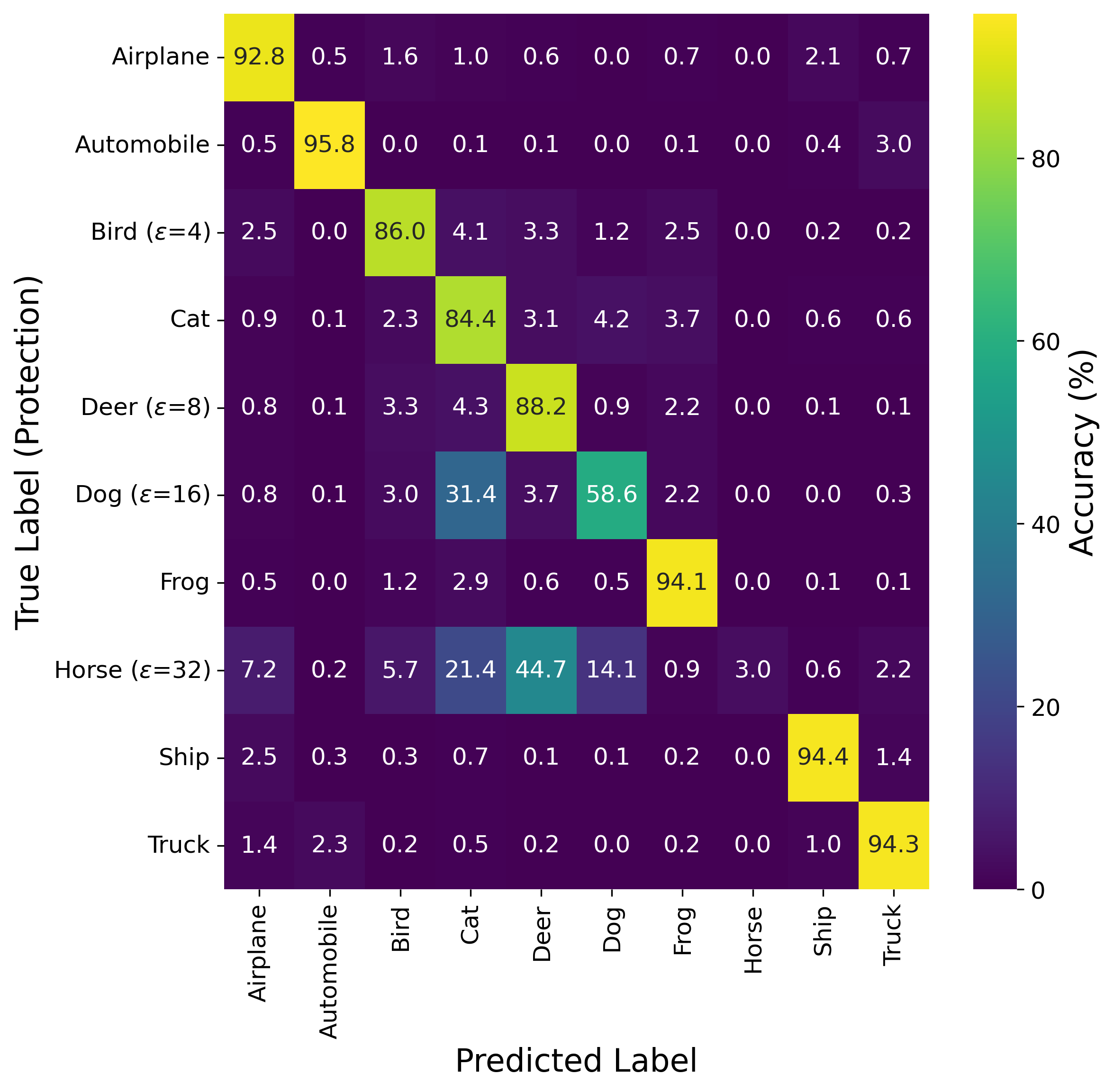}
            \caption*{\textsc{Avatar}}
    	\end{subfigure}
	\caption{REMN~\citep{fu2022remn}}
    \vspace*{2em}
    \end{subfigure}
        \\
    \begin{subfigure}{1.0\textwidth}
    \centering
        \begin{subfigure}{.40\textwidth}
    		\centering
    		\includegraphics[width=1.0\textwidth]{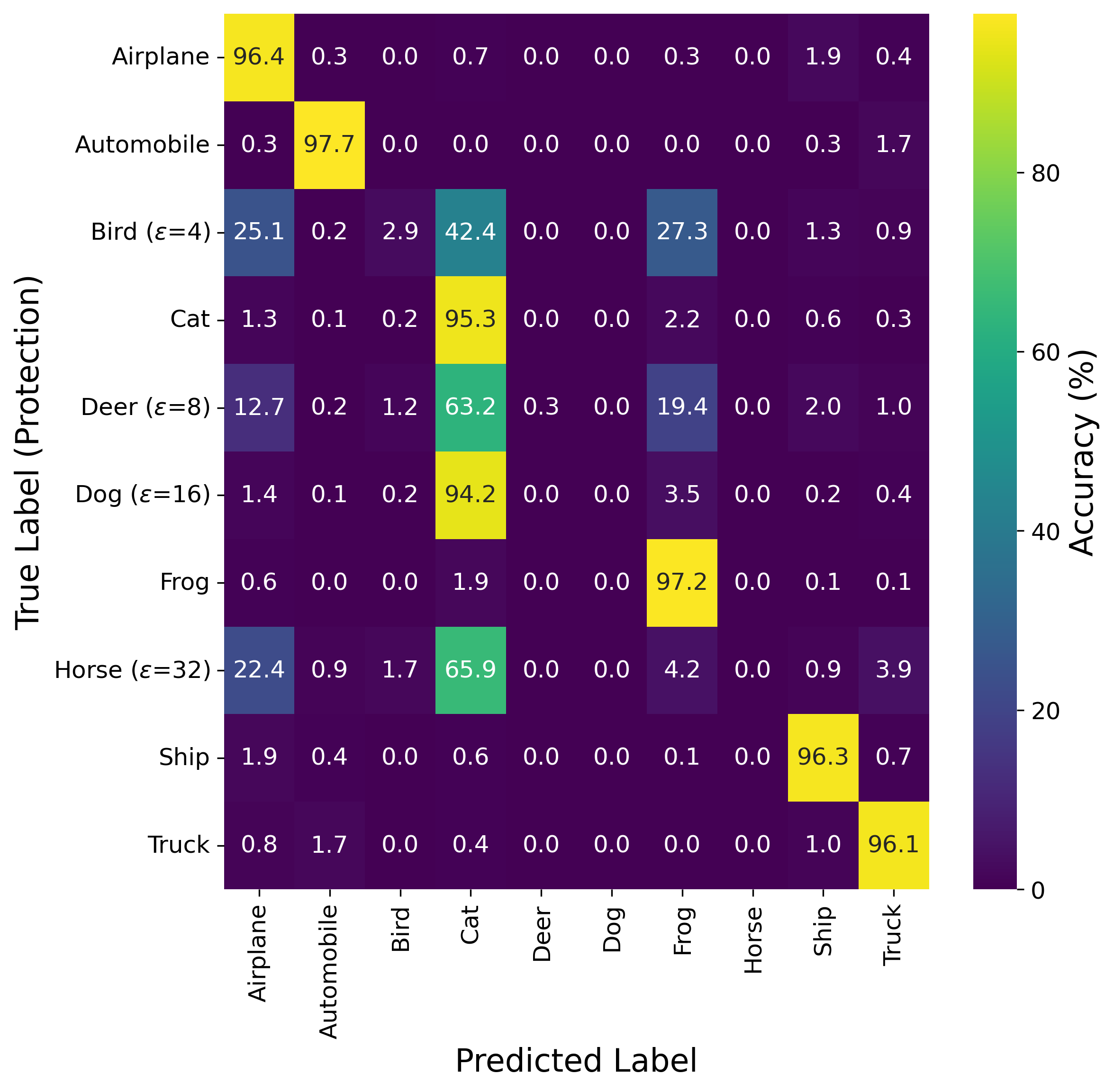}
            \caption*{Vanilla Training}
    	\end{subfigure}
     \hspace{5em}
        \begin{subfigure}{.40\textwidth}
    		\centering
    		\includegraphics[width=1.0\textwidth]{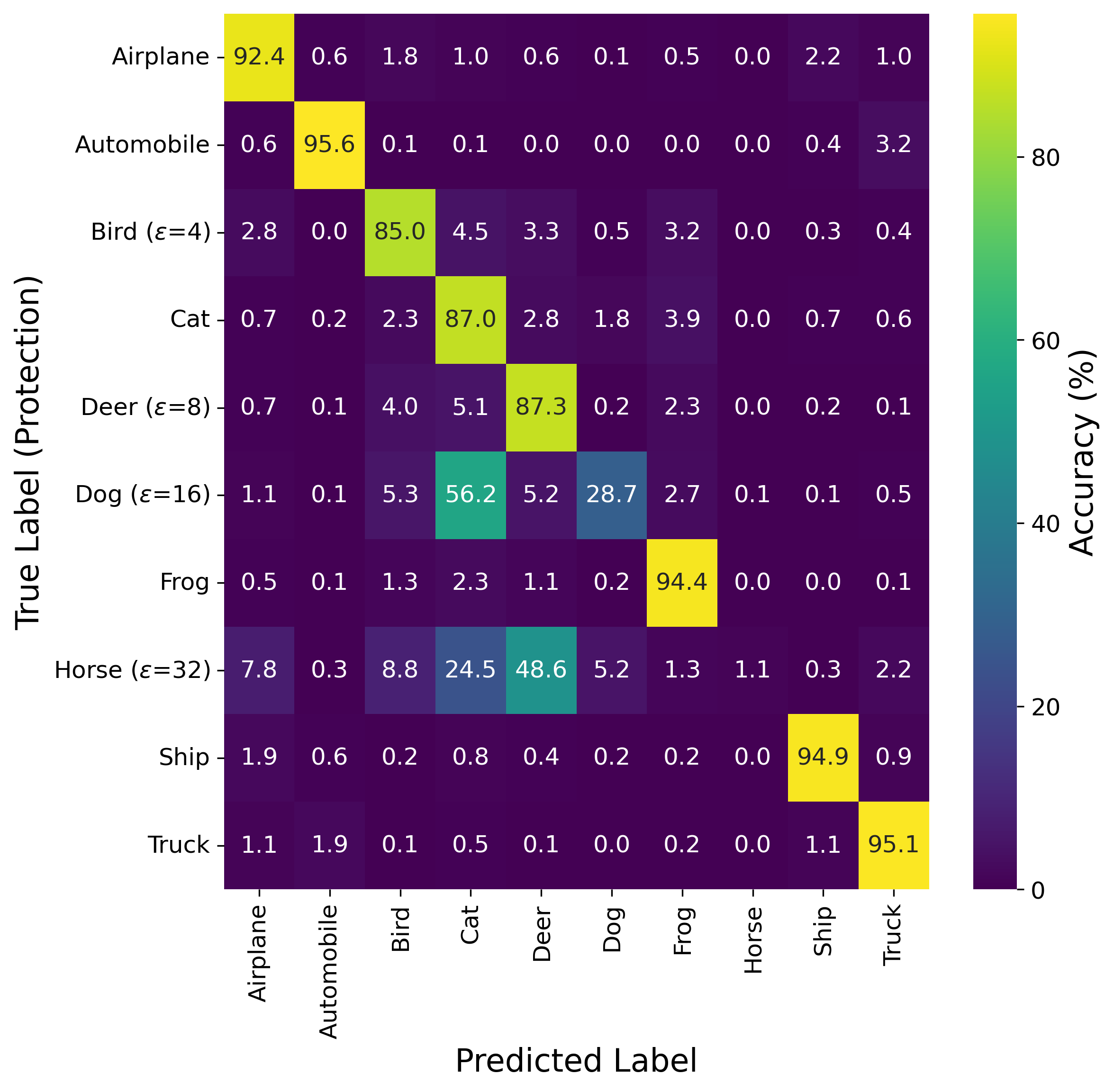}
            \caption*{\textsc{Avatar}}
    	\end{subfigure}
	\caption{SHR~\citep{yu2022shr}}
    \end{subfigure}
    \caption{The confusion matrices of RN-18 classifiers trained over CIFAR-10 dataset. For each availability attack, we use four different perturbation norms to protect a randomly selected set of classes. For \textsc{Avatar}, we follow our settings for the experiments in~\Cref{tab:dist_mismatch_new} and use a DDPM-IP pre-trained over the IN-1k-32$\times$32 dataset.}
	\label{fig:epsilon_attacks}
\end{figure*}
\end{document}